\documentclass{article}

\usepackage{PRIMEarxiv}

\usepackage{algorithm}
\usepackage{algpseudocode}
\usepackage{graphicx} 
\usepackage{amsmath}
\usepackage{amssymb}
\usepackage{amsthm}
\usepackage{dsfont}
\usepackage{color}
\usepackage{colortbl}
\usepackage{appendix}
\usepackage{url}
\usepackage{subcaption}
\usepackage{dsfont}
\usepackage{amsfonts}
\usepackage{graphicx}
\usepackage{xcolor}
\usepackage{colortbl}
\usepackage{booktabs}
\usepackage{rotating}
\usepackage{multirow}
\usepackage{bm}

\usepackage{tikz}
\usepackage{pgfplots}

\newif\ifdraft

\newtheorem{remark}{Remark}

\newtheorem{definition}{Definition}[section]
\newtheorem{lemma}{Lemma}

\newcommand{\calib}{\zeta}
\newcommand{\ocalib}{\calib^*}
\newcommand{\hcal}{\widetilde{h}}
\newcommand{\ohcal}{\hcal^{*}}
\newcommand{\ohcalplus}{\ohcal_\oplus}
\newcommand{\ohcalminus}{\ohcal_\ominus}
\newcommand{\hcrisp}{\hbar}

\newcommand{\ddef}{\stackrel{\;\text{def}}{=}\;}

\newcommand{\quant}{\rho}
\newcommand{\oquant}{\rho^*}

\newcommand{\acc}{\alpha}
\newcommand{\oacc}{\alpha^*}

\newcommand{\inspace}{\mathcal{X}}
\newcommand{\outspace}{\mathcal{Y}}

\newcommand{\trset}{D_{tr}}
\newcommand{\hset}{D_{h}}
\newcommand{\valset}{D_{val}}
\newcommand{\teset}{D_{te}}
\newcommand{\Dplus}{D_{\oplus}}
\newcommand{\Dminus}{D_{\ominus}}
\newcommand{\population}{D^\infty}
\newcommand{\sample}{\mathcal{S}}

\newcommand{\tpr}{\text{tpr}}
\newcommand{\fpr}{\text{fpr}}
\newcommand{\tprhat}{\hat{\tpr}}
\newcommand{\fprhat}{\hat{\fpr}}

\newcommand{\ycal}{\widetilde{y}}

\title{On the Interconnections of Calibration, Quantification, and Classifier Accuracy Prediction under Dataset Shift}

\author{Alejandro Moreo}

\author{
    Alejandro Moreo \\
    Istituto di Scienza e Tecnologie dell'Informazione,  
    Consiglio Nazionale delle Ricerche\\
    Via Giuseppe Moruzzi 1, 56124, Pisa, Italy\\
  \texttt{alejandro.moreo@isti.cnr.it} \\
}

\begin{document}

\maketitle

\begin{abstract}
When the distribution of the data used to train a classifier differs from that of the test data, i.e., under dataset shift, well-established routines for calibrating the decision scores of the classifier, estimating the proportion of positives in a test sample, or estimating the accuracy of the classifier, become particularly challenging.
This paper investigates the interconnections among  three fundamental problems, calibration, quantification, and classifier accuracy prediction, under dataset shift conditions. 
Specifically, we prove their equivalence through mutual reduction, i.e., we show that access to an oracle for any one of these tasks enables the resolution of the other two. Based on these proofs, we propose new methods for each problem based on direct adaptations of well-established methods borrowed from the other disciplines. Our results show such methods are often competitive, and sometimes even surpass the performance of dedicated approaches from each discipline.
The main goal of this paper is to fostering cross-fertilization among these research areas, encouraging the development of unified approaches and promoting synergies across the fields.
\end{abstract}


\section{Introduction}

Classifiers are often deployed in contexts in which the \textit{independent and identically distributed} (IID) assumption is violated, i.e., in which the data used to train the model and the future data to be classified are not drawn from the same distribution.
This situation is generally referred to as \textit{dataset shift} in the machine learning literature \cite{datasetShiftStorkey:2009}.

In this context, three problems have gained increased attention in the last years.
\textit{Classifier calibration} \cite{EncyclopediaCalibrationFlach2016,calibrationsurveyML2023} concerns the manipulation of the confidence scores produced by a classifier so that these 
effectively reflect the likelihood that a given instance is positive. 
\textit{Quantification} \cite{SurveyQuantification:2017,Book2023} is instead concerned with estimating the  prevalence of the classes of interest in an unlabelled set.
Finally, \textit{classifier accuracy prediction} aims at inferring how well a classifier will fare on unseen data \cite{elsahar-galle-2019-annotate,Guillory:2021so}.

Well-established procedures for attaining these three goals when the IID assumption holds are known and routinely used.
For instance, calibrating the classifier's outputs can be attained by learning a calibration map (a function mapping classifier confidence scores into values reflecting the likelihood of the positive class) on held-out validation data  \cite{Platt:2000fk,empiricalbinning:zadrozny:2001,isotonic:barlow:1972}.\footnote{Notwithstanding this, calibration under IID conditions is still an active area of research; see, e.g., \cite{calibrating_sufficiently_Tasche2021}.} Concerning quantification, the class proportions in a test set can be estimated by simply classifying and counting how many positives fall under which class. Finally, the performance that a classifier will exhibit on unseen data can be estimated on held-out validation data \cite{Hastie2009}. 

However, when the IID assumption is violated, these standard routines are prone to failure.
For example, calibration is a property typically defined with respect to a specific distribution, meaning that a classifier is unlikely to remain well-calibrated across different distributions \cite{CalibrationDatasetShiftNeurIPS2019}. Similarly, the commonly used ``classify and count'' approach for class prevalence estimation is known to yield biased predictions under certain dataset shift conditions \cite[\S 1.2]{Book2023}. 
Likewise, the accuracy of a classifier estimated on held-out validation data or via cross-validation is biased when the IID assumption does not hold \cite{datasetShiftStorkey:2009}.

The three problems are deeply interconnected and may arise in related real-world applications. Consider, for example, a classifier trained on X-ray lung images to assist clinical decisions regarding medical interventions. In order for the classifier to aid in making informed decisions, the classifier must not only provide classifier decisions, but also a measure of uncertainty attached to it, i.e., its outputs must be well calibrated. Now, suppose a pandemic occurs, altering the natural prevalence of a pneumonic virus in the population. In response, several actions must be undertaken to develop a trustworthy response to the virus.
Estimating the prevalence of the population affected by the virus (for which any of the quantification methods discussed in \cite{SurveyQuantification:2017} can be employed) is of the utmost importance for epidemiologists \cite{patrone2024minimizing}; these prevalence estimates are useful to recalibrate the classifier outputs to reflect the new priors \cite{godau2025navigating} (using techniques such as those in \cite{Elkan:2001fk,guilbert2024calibration}). Alternatively, the practitioner might need to estimate the accuracy of different candidate classifiers under the new conditions in order to, e.g., replace the existing classifier with a different one that is expected to perform more accurately in the new distribution, or maybe to reweigh the relative importance of different classifiers in an ensemble (which necessitate methods for classifier accuracy prediction such as, e.g., \cite{Guillory:2021so,Garg:2022qv}).
A better understanding of the interdependencies between the three problems may lead to more robust and integrated solutions, especially in dynamic or high-stakes environments such as healthcare. 

However, the applicability of one technique or another is ultimately subject to verifying certain assumptions about the relationship between the underlying distributions.
In the context of dataset shift \cite{datasetShiftStorkey:2009}, the two most important such types of shift include \emph{covariate shift} (CS) and \emph{prior probability shift} (aka \emph{label shift} --LS).
CS presupposes a change in the marginal distribution of the observed features, while the posterior probabilities are assumed stationary. Conversely, LS has to do with a change in the class prevalence subject to stationary class-conditional distributions of the observed features. More formal definitions are provided later on.



In this paper, we argue that the interplay between the three problems is deeper than it seems at first glance.
We begin by offering formal proofs of their equivalence via reduction, showing that having access to a perfect model (an oracle) for any of these problems will automatically enable the resolution of the other two tasks.
We then build on those formalizations to propose practical adaptations of methods from each discipline to address the other two. 
Somehow surprisingly, our results show that some of these adaptations tend to perform unexpectedly well in other problems.
The main contribution of this paper is to raise awareness of the close relationship between these problems, with the hope of fostering cross-fertilization among different fields and promoting unified solutions.

The rest of this paper is structured as follows. 
In Section~\ref{sec:related} we review some occasional interactions among calibration, quantification, and classifier accuracy prediction in the literature, as well as dedicated related work for each discipline.
Section~\ref{sec:preliminaries} is devoted to fix the notation and present technical definitions which are functional to the proofs of problems equivalence via reduction presented in Section~\ref{sec:reductions}.
Section~\ref{sec:methods} explores new methods that we propose for each of the three task, by adapting principles and techniques borrowed from the other tasks.
We report the experiments we have carried out in Section~\ref{sec:experiments}, in which we confront proper methods from each discipline with the newly proposed adaptations.
Section~\ref{sec:discussions} offers some final discussions and wraps up, while also suggesting potential ideas for future work.


\section{Related Work}
\label{sec:related}

Perhaps the most obvious link between the three problems pivots on classifier calibration. Despite the fact that the frequent co-occurrence of the term ``calibration'' with the term ``quantification'' in the literature often actually refers to ``uncertainty quantification'', and not to ``quantification'' in the sense of class prevalence estimation, it is a well-known fact in the quantification community that a properly calibrated classifier for the test distribution would automatically enable an accurate estimator of class prevalence \cite{Card:2018pb,tasche20144}. Conversely, the link between calibration and classifier accuracy prediction is embedded in the definition of the calibration property itself, according to which the predicted confidence of a perfectly calibrated classifier is a good approximation of
its actual \emph{probability of correctness} \cite{calibrationsurveyML2023,head2tail2023}. 

Beyond this, some sporadic interconnections among the problems have been echoed
recently in \cite{calibrateEstrapolateWu2024}, in which quantification and calibration take part in the so-called ``Calibrate-Extrapolate'' framework for applications in the social sciences; and in \cite{Volpi:2024ye}, in which quantification is used to support the task of classifier accuracy prediction.

The remainder of this section reviews the related literature on calibration (Section~\ref{sec:related:calibration}), quantification (Section~\ref{sec:related:quantification}), and classifier accuracy prediction (Section~\ref{sec:related:cap}).

\subsection{Calibration}
\label{sec:related:calibration}

The concept of model calibration first appeared in the field of meteorology \cite{murphy1972scalar}, and is nowadays used to refer to techniques for transforming the confidence scores generated by a classifier into estimates of posterior probabilities, i.e., into values reflecting the probability of a datapoint to belong to each class of interest \cite{EncyclopediaCalibrationFlach2016,calibrationsurveyML2023}. 
A perfectly calibrated classifier is thus a model that perfectly accounts for its own uncertainty. For example, if a perfectly calibrated binary classifier returns 0.95 as the posterior probability for the positive class, we know that the classifier is confident that the instance is positive, but there is still 5\% chance the decision is incorrect. Similarly, if the classifier returns 0.49, the classifier is also telling us that it is highly uncertain about the decision, suggesting that one should be very cautious if any critical decision is to be taken based on such prediction.


Calibration techniques may be divided in two categories: regularization-based and post-hoc calibration techniques.
Regularization-based techniques add a calibration-oriented factor to the loss the model optimizes at training time.
These techniques have proven useful for calibrating the output of modern neural networks, which are otherwise known to be poorly calibrated and typically overconfident \cite{Guo:2017hh}. Popular such techniques include temperature scaling \cite{hinton2015distilling}, vector scaling \cite{Guo:2017hh}, and Dirichlet calibration \cite{dirichletcalibration:kull2019}. 

On the other side, post-hoc calibration techniques use held-out validation data 
to learn a calibration map that transforms the confidence scores of a previously trained model. A calibration map can be learned using non-parametric techniques such as empirical binning \cite{empiricalbinning:zadrozny:2001} or isotonic calibration 
\cite{isotonic:barlow:1972}, or using parametric ones such as Platt scaling \cite{Platt:2000fk}, which relies on the logistic function, or more recent methods relying on more flexible parametric functions such as the beta calibration \cite{kull2017beyond}. In this paper, we will hereafter focus on post-hoc calibration techniques.

In high-stakes domains with non-stationary conditions, having well-calibrated outputs is of crucial importance, since these can be easily adapted to changes in the class prior or changes in the cost distribution \cite{calibrationsurveyML2023}. 
Unfortunately, the post-hoc calibration techniques discussed above have been found to fall short of producing reliable estimates of classifier uncertainty in the presence of dataset shift \cite{CalibrationDatasetShiftNeurIPS2019,karandikar2021soft}.
As a response, several post-hoc methods have recently been proposed in the literature to cope with the problem of dataset shift \cite{cpcs2020icml,transcal2020,head2tail2023,popordanoska2024lascal}. All these methods operate under the same conditions, i.e., that all available labelled examples come from the training (or ``source'') distribution, while all examples from the test (or ``target'') distribution are devoid of labels.

\cite{cpcs2020icml} proposes Calibrated Prediction with Covariate Shift (CPCS), a calibration method that minimizes an upper bound on the expected calibration error (ECE), which is derived following assumptions from CS. The method adopts principles from domain adaptation, and relies on importance weighting to correct for the shift between the training and test distributions, respectively. The importance weights take the form of a ratio $\frac{Q(x)}{P(x)}$, where $P$ and $Q$ are the density functions of the training and test distributions. As such weights are unknown, the authors propose to estimate them using a ``source vs. target'' discriminator, i.e., by taking the posterior probability of a logistic regression model trained to identify whether an instance comes from the test distribution (acting as the positive class) or instead from the training one (acting as the negative class), in a surrogate binary classification problem.

Later on, a method called Transferable Calibration (TransCal) \cite{transcal2020} was proposed in response to a phenomenon observed under CS assumptions, in which the improvement in classifier accuracy through domain adaptation techniques comes at the expense of ECE in the target distribution. TransCal is a hyperparameter-free method that minimizes a new calibration measure called Importance Weighted Expected Calibration Error (IWECE) which manages to estimate the calibration error in the target domain under CS assumptions with lower bias and variance than CPCS.

The Head2Tails calibration model \cite{head2tail2023} was later proposed as a means to counter the distribution shift that occurs when the classifier is trained in a close-to-balanced setting, but the data to be classified in deployment conditions instead follow a long-tailed class distribution, where few classes (the ``heads'') gather most of the density whereas many classes (the ``tails'') are severely imbalanced. 
Classifiers trained this way tend to be overconfident for the ``head'' classes.
The method estimates the importance weights of datapoints in the ``tail'' classes to better calibrate their outputs by transferring knowledge from the ``heads''.
Somehow surprisingly, though, the authors describe such a setting as an instance of CS. While it is clear that a shift in the class proportions results in a shift of the marginal distributions of the covariates, we believe such a setting is more representative of LS instead, since only the priors of the classes have changed. Later on, two closely related calibration methods (called ``polynomial positive'' and ``exponential'') are presented in \cite{guilbert2024calibration} in the context of imbalanced binary classification problems. The main goal is to improve the calibrated outcome for the minority class (which represents the interesting phenomenon, such as a rare illness) on the grounds that no action is typically undertaken for the majority class examples (e.g., on healthy individuals).

Recently, the method LasCal \cite{popordanoska2024lascal} has been proposed to better calibrate classifiers under LS. The method adopts a temperature scaling strategy, and optimizes for the temperature parameter by minimizing a pointwise consistent ECE that roots on the LS assumptions.

\subsection{Quantification}
\label{sec:related:quantification}

Quantification was proposed by \cite{Forman:2005fk} as the machine learning task of estimating the class prevalence values in unlabelled sets; see \cite{SurveyQuantification:2017,Book2023} for overviews.
For this reason, most of the quantification methods proposed so far 
have mainly focused on scenarios characteristic of LS, i.e., in cases in which the priors are expected to vary with respect to the training conditions; see \cite{Schumacher2021,Lequa2022Overview,Lequa2024Overview} for evaluation campaigns. 
Prototypical applications of quantification span areas such as epidemiology \cite{King:2008fk}, social sciences \cite{Hopkins:2010fk}, ecological modeling \cite{Beijbom:2015yg}, market research \cite{Esuli:2010kx}, and sentiment analysis \cite{Moreo:2022bf}.
Conversely, little attention in the quantification literature has been paid to CS, aside from some theoretical analysis \cite{Tasche:2022hh,Tasche:2023um} and few experimental evaluations \cite{Gonzalez:2024cs,Lequa2024Overview}.

The IID approach to quantification comes down to training a classifier using supervised examples from the training distribution, and use it to issue label predictions for the test examples; the class prevalence values are simply obtained by counting the fraction of instances that fall under each class. This method, called ``classify and count'' (CC) is a biased estimator under LS conditions \cite[\S 1.2]{Book2023}, and represents a weak baseline any proper quantification method is expected to beat.
However, when the fractions are instead counted as the expected value of a well-calibrated probabilistic classifier, i.e., when using the ``probabilistic classify and count'' method (PCC) \cite{Bella:2010kx}, the resulting estimates are known to be accurate under CS assumptions \cite{Card:2018pb,Tasche:2022hh}.

Arguably, the method ``adjusted classify and count'' (ACC) \cite{Forman:2008kx} is the simplest way to correct for the bias of CC under LS. ACC applies a linear correction to the CC estimate (denoted by $\hat{p}^{\text{CC}}$), by taking into account an estimate of the true positive rate (tpr) 
and false positive rate (fpr) 
of the classifier.
Under LS conditions, the tpr and fpr are assumed stationary, so they can be estimated in hold-out validation data from the training distribution (see also Remark~\ref{remark:pps} in Section~\ref{sec:definitions}).
%
%
The method was later rediscovered and popularized in the machine learning community under the name of Black-Box Shift Estimator (BBSE) by \cite{Lipton:2018fj}. The probabilistic counterpart, dubbed ``probabilistic adjusted classify and count'' (PACC) \cite{Bella:2010kx}, consists of replacing the crisp count obtained via CC
with the expected soft count  
obtained via PCC,
and also using the estimates of the posterior probabilities returned by the probabilistic classifier (instead of the crisp ones) to obtain estimates for the tpr and fpr.

One of the most important families of approaches in the quantification literature is the so-called distribution matching (DM) \cite{Forman:2005fk}.
DM approaches seek for the mixture parameter (the sought class prevalence values) that yields a mixture of the class-conditional distributions of the training datapoints which is closest to the distribution of the test datapoints, in terms of any given divergence measure. Different methods have been proposed in the literature relying on different mechanisms for representing the distributions, including histograms of posteriors \cite{Gonzalez-Castro:2013fk,Perez-Mon:2025jt} or Fourier transformations of the features \cite{Dussap:2023gr}, 
and different divergence measures, including the Hellinger Distance \cite{Gonzalez-Castro:2013fk} or the Topsøe divergence \cite{Maletzke:2019qd}, 
among others. 
One DM-based method which has demonstrated superior performance in recent evaluations is KDEy~\cite{KDEyMoreo:2025}, in which distributions are represented by means of kernel density estimates of the posteriors, and in which the Kullback-Leibler divergence is adopted as the divergence measure to minimize. 

Some efforts have been recently paid trying to unify the most important quantification methods under a common DM-based framework \cite{Firat:2016uq,Garg:2020jt,Bunse:2022ky,Castano:2023mb}.
While the above-discussed methods can be framed as specific instances of these frameworks, one notable exception is the Expectation-Maximization method proposed by \cite{Saerens:2002uq}, typically referred to as EMQ (or SLD after the name of its proponents \cite{Esuli:2021le}) in the quantification literature. This method precedes the field of quantification itself, and was indeed proposed for a different scope: updating the posterior probabilities of the test datapoints, as estimated by a probabilistic classifier that has been calibrated for the training distribution, to a target distribution in which the priors have changed. The method relies on a mutually recursive update of the priors and the posteriors, following the EM algorithm, until convergence. 
While the method was originally proposed as a means for recalibrating the outputs of a probabilistic classifier with respect to a new prior under LS, it has become very popular in quantification endeavours, in which we rather take, as the output of the algorithm, the last updated priors found by the method.
This method is known to behave extremely well, and several enhancements have been proposed in the machine learning literature, including further recalibration rounds \cite{Alexandari:2020dn} and regularized variants thereof \cite{Azizzadenesheli:2019qf}.

\subsection{Classifier Accuracy Prediction}
\label{sec:related:cap}

The standard method for estimating the classifier accuracy on future data under the IID assumption is cross-validation \cite{Hastie2009}. 
Recently, some methods have been proposed to operate under dataset shift conditions, such as the \emph{Reverse Classification Accuracy} (RCA) \cite{elsahar-galle-2019-annotate}, which bases its prediction on the validation accuracy of another classifier trained on test data labelled by the target classifier; \textit{Mandoline}~\cite{Chen:2021qe}, relying on importance weighting and density estimation; or \cite{Chen:2021jt}, which relies instead on the agreement ratio of an ensemble of auxiliary models which is updated iteratively; \textit{Generalisation Disagreement Equality}
(GDE)~\cite{jiang2022assessing}, which analyzes the degree of disagreement on IID data of identical neural architectures as a proxy of their accuracy on shifted data. 

In this paper, we will take a closer look at three recent methods that have showcased superior performance in recent years: ATC \cite{Garg:2022qv}, DoC \cite{Guillory:2021so} and LEAP \cite{Volpi:2024ye}. 

ATC, standing for \textit{Average Thresholded Confidence} \cite{Garg:2022qv}, searches for the threshold value $t$ on the confidence value of a classifier that yields, on validation data, a proportion of datapoints with confidence higher than it equals  the proportion of datapoints correctly classified (i.e., equals the classifier accuracy score on the validation data). The authors poposed to compute the confidence value in terms of maximum confidence or negative entropy.


The method DoC, for \textit{Difference of Confidence}
\cite{Guillory:2021so}, learns a regression model to predict the model performance on the test set. The regressor is trained on a measure of distance between the in-distribution data and the out-of-distribution data. Such measure reflects the accuracy gap, and is computed in terms of the average max confidence over a series of validation samples $V_i$ which are generated with the aim of reflecting the type of dataset shift to be faced. 
The underlying principle is that, if the model is well-calibrated and there is no shift, the expected value of the model’s confidence must coincide with its accuracy. The regressor is thus used to learn a correction from the average maximum confidence value to the actual accuracy under a shifted distribution.

Finally, LEAP (Linear Equation -based Accuracy Prediction) \cite{Volpi:2024ye} models the problem as a system of linear equations with $n^2$ unknowns, with $n$ the number of classes. The equations are derived following the LS assumptions and relies, for their specification, on a quantification method (KDEy, see Section~\ref{sec:related:quantification}) as an intermediate step. A key difference of this method, is that it predicts the $n^2$ entries of a contingency table, and is thus not tied to one specific evaluation measure. While LEAP has proven to fare well under LS, it has never been tested under CS.

\section{Preliminaries}
\label{sec:preliminaries}

This section introduces notation and definitions that will be instrumental for the following sections.

\subsection{Notation}
\label{sec:notation}

Let $\inspace\subset\mathbb{R}^d$ be the input space of the covariates for some $d\in\mathbb{N}$, and $\outspace$ be the output space of class labels. We will restrict our attention to the binary case, in which case we define $\outspace=\{0,1\}$ with 0 denoting the negative class, and 1 denoting the positive class. By $\mathcal{S}$ we denote the space of samples of elements from $\inspace$.

We define a binary classifier as $h: \inspace \rightarrow \mathbb{R}$, i.e., as a function $h\in\mathcal{H}$, where $\mathcal{H}$ denotes the class of hypotheses, returning raw (uncalibrated) confidence scores. From $h$ we can obtain a crisp classifier $\hcrisp: \inspace \rightarrow \outspace$ via thresholding by simply defining $\hcrisp(x)=1$ if $h(x)>t$, and $\hcrisp(x)=0$ otherwise,
for some threshold value $t\in\mathbb{R}$. We use $\widetilde{h}:\mathcal{X}\rightarrow[0,1]$ to denote a probabilistic classifier, i.e., a classifier whose outputs are interpretable as posterior probabilities.

We assume the classifier $h$ has been generated using some training data $\trset=\{(x_i, y_i)\}_{i=1}^m$, with $x_i\in\inspace, y_i\in\outspace$. We also assume the existence of a test set $\teset=\{x_i\}_{i=1}^{n}$ for which we do not observe the true labels. We write $\Phi(x)$ to denote the true label of a test datapoint $x$. We assume a dataset shift scenario in which the training and test data have been drawn from two distributions $P$ and $Q$, respectively, with $P\neq Q$. 
We use the random variables $X$, $Y$ taking values, respectively, on $\mathcal{X}$, $\mathcal{Y}$. We will also use $\hat{Y}=\hcrisp(X)$ and $\widetilde{Y}=\hcal(X)$ to denote the random variables of the classifier's crisp and soft predictions, ranging on $\{0,1\}$ and $[0,1]$, respectively.




\subsection{Definitions}
\label{sec:definitions}

\begin{definition} Covariate shift  (CS) \cite{datasetShiftStorkey:2009} is the type of dataset shift in which the marginal distributions of the covariates are assumed to change, while the labelling function is assumed stationary, i.e.:
\begin{align*}
    P(X) &\neq Q(X) \\
    P(Y|X) &= Q(Y|X)
\end{align*}
This type of shift is characteristic of $X\rightarrow Y$ problems \cite{Fawcett:2005fk}, i.e., of \emph{causal learning} where we are interested in predicting effects from causes. In problems affected by CS, it is often convenient to factorize the join distribution as $P(X,Y)=P(Y|X)P(X)$.
\end{definition}

\begin{definition} Prior probability shift \cite{datasetShiftStorkey:2009}, variously called label shift \cite{Lipton:2018fj} (LS) or class-prior change \cite{du-Plessis:2012nr}, is the type of dataset shift in which the class proportions are assumed to change, while the class-conditional densities of the covariates are assumed stationary, i.e.:
\begin{align*}
    P(Y) &\neq Q(Y) \\
    P(X|Y) &= Q(X|Y)
\end{align*}
This type of shift is characteristic of $Y\rightarrow X$ problems \cite{Fawcett:2005fk}, i.e., of \emph{anti-causal learning} where we are interested in predicting causes from effects. In problems affected by LS, it is often convenient to factorize the join distribution as $P(X,Y)=P(X|Y)P(Y)$.
\end{definition}

\begin{remark}
\label{remark:pps}
Given any deterministic measurable function $f:\mathcal{X}\rightarrow\mathbb{R}$, the LS assumption $P(X|Y)=Q(X|Y)$ implies $P(f(X)|Y)=Q(f(X)|Y)$ \cite[Lemma 1]{Lipton:2018fj}. This has important implications when we take $f=\hcrisp$ (our crisp classifier), as this implies that the class-conditional distribution of the classifier predictions is stationary, i.e.,  $P(\hat{Y}|Y)=Q(\hat{Y}|Y)$. Similarly, if we take $f=\hcal$ (our probabilistic classifier) it holds that $P(\tilde{Y}|Y)=Q(\tilde{Y}|Y)$
\end{remark}

\begin{definition}
\label{def:calibration}
A \textit{calibrator} $\calib: \mathcal{H}\times \sample\rightarrow\mathcal{H}$ is a model that takes as input an uncalibrated classifier $h\in\mathcal{H}$ and a reference unlabelled set $D\in\sample$, and generates a classifier $\hcal$ which is calibrated with respect to $D$. A perfect calibrator $\ocalib$ is such that the classifier $\ohcal = \ocalib(h,D)$ it returns is perfectly calibrated with respect to $D$, which in turn means for all outcomes $\ycal\in[0,1]$ of $\ohcal$ the following condition holds true:
\begin{equation}
\label{eq:hcal}
    \ycal=\frac{|\{x\in D : \Phi(x)=1, \ohcal(x)=\ycal \}|}{|\{x\in D : \ohcal(x)=\ycal \}|}
\end{equation}
\noindent that is, the value $\ycal$ corresponds to the true proportion of positives among all elements for which the classifier assigns the same value.

Asymptotically, i.e., that when the sample $D$ approximates the entire population $\population$, the calibrated classifier $\ohcal=\ocalib(h,\population)$ is such that  for all $\ycal=\ohcal(x)$: 
\begin{equation}
\label{eq:ycalib:popul}
    \ycal = Q(Y=1|\widetilde{Y}=\ycal)
\end{equation}




\end{definition}

\begin{remark} Note that the notion of calibration is not necessarily tied to classifier accuracy \cite{calibrationsurveyML2023}; in other words, a perfectly calibrated classifier is not equivalent to a perfect classifier. For example, if a classifier is perfectly calibrated and returns the posterior probability $\ohcal(x)=0.60$ for some $x$, this means that predicting $x$ belongs to the positive class (e.g.,  by thresholding at $t=0.5$) has 40\% chance of being mistaken.
\end{remark}

\begin{remark} Note also that our definition of a calibrator function does not align with the more traditional one, in which the calibrator function is given access to the class labels of the datapoints in the reference set. Conversely, our notion of calibration aligns well with more recent work on calibration in the context of dataset shift (see, e.g., \cite{transcal2020,popordanoska2024lascal}).
\end{remark}

\begin{definition}
\label{def:quantification}
A binary \textit{quantifier} is a model 
$\quant: \mathcal{S} \rightarrow [0,1]$ 
that predicts the fraction of positive instances in a sample. A perfect quantifier $\oquant$ is such that:
\begin{equation}
\label{eq:oquant}
    \oquant(D)=\frac{|\{x\in D : \Phi(x)=1\}|}{|D|}
\end{equation}
With respect to the population, the perfect quantifier returns the prior of the positive class in the underlying distribution, i.e.:
\begin{equation}
\oquant(\population)=Q(Y=1)    
\end{equation}

\end{definition}


\begin{definition}
\label{def:cap}
A predictor of classifier accuracy is a model 
$\acc:\mathcal{H}\times \mathcal{S}\rightarrow [0,1]$ 
that takes as input a classifier $h\in\mathcal{H}$ and a reference set $D\in\mathcal{S}$ and predicts the accuracy $h$ will have on $D$. A perfect predictor of classifier accuracy satisfies:
\begin{equation}
\label{eq:oacc}
    \oacc(h,D)=\frac{|\{x\in D : \Phi(x)=\hcrisp(x)\}|}{|D|}
\end{equation}
At the population level, the predictor $\oacc$ returns the probability that the predictions of $h$ are correct in the underlying distribution, i.e.:
\begin{equation}
    \oacc(h,\population)=Q(Y=\hat{Y})
\end{equation}
\end{definition}

\begin{remark}
While the term \textit{accuracy} is generally used to denote any evaluation measure of classifier performance, our definition is bounded to the case in which such measure is taken to be \textit{vanilla accuracy}, i.e., the fraction of correctly classified instances.
\end{remark}


\section{Problem Equivalences via Reduction}
\label{sec:reductions}


In this section, we focus our attention to the binary case, and show that the three problems are equivalent via reduction. That is, we show that a perfect model for one of the problems would enable a perfect solution for the other two. 
In what follows, we cover the cases in which we assume access to a perfect calibrator $\ocalib$ (Section~\ref{sec:accesstocal}), access to a perfect quantifier $\oquant$ (Section~\ref{sec:accesstoquant}), and access to a perfect estimator of classifier accuracy $\oacc$ (Section~\ref{sec:accesstocap}).

\subsection{Access to a Perfect Calibrator $\ocalib$}
\label{sec:accesstocal}

Let us assume we have access to a perfect calibrator and try to attain perfect models for quantification (Lemma~\ref{lemma:cal2quant}) and classifier accuracy prediction (Lemma~\ref{lemma:cal2cap}).

\begin{lemma}
\label{lemma:cal2quant}
$\ocalib \implies \oquant$, i.e., from a perfect calibrator we can attain a perfect quantifier.
\end{lemma}

\begin{proof}
Let $\ohcal=\ocalib(h,D)$ be a perfectly calibrated classifier with respect to $D$. Consider the following estimator:
\begin{align}
\label{eq:cal2quantmean}
    E(\ohcal, D) \ddef \frac{1}{|D|}\sum_{x_i\in D} \ohcal(x_i)=\frac{1}{|D|}\sum_{x_i\in D} \ycal_i
\end{align}
This summation can be equivalently rewritten as a sum of unique posterior values $\ycal$ in $U=\{\ohcal(x_i):x_i\in D\}$, each multiplied by the total number of times it appears in the sequence, i.e., multiplied by $|\{x_i\in D: \ohcal(x_i)=\ycal\}|$. Note that some classifiers might assign the same posterior value to several datapoints. For example, some implementations of the probabilistic $k$-NN classifier will at most assign $k+1$ different such values; decision trees may assign the same posterior value to every datapoint assigned to the same leaf node of a decision tree. Other classifiers, however, might potentially assign different probabilities to every datapoint in $D$ \cite{EncyclopediaCalibrationFlach2016}. While this has implications in practical cases, it does not compromise the proof whatsoever. It thus follows that:
%
\begin{align*}
    E(\ohcal,D)=\frac{1}{|D|}\sum_{\ycal_j\in U} \left(\ycal_j \cdot |\{x_i\in D: \ohcal(x_i)=\ycal_j\}|\right)
\end{align*}
Substituting $\ycal_j$ by the definition of a perfectly calibrated classifier (Equation~\ref{eq:hcal}) it follows that
\begin{align*}
    E(\ohcal,D) =\; & \frac{1}{|D|}\left(\sum_{\ycal_j\in U} \frac{|\{x_i\in D : \Phi(x_i)=1, \ohcal(x_i)=\ycal_j \}|}{|\{x_i\in D : \ohcal(x_i)=\ycal_j \}|} \cdot |\{x_i\in D: \ohcal(x_i)=\ycal_j\}|\right) \\
    =\; & \frac{1}{|D|}\sum_{\ycal_j\in U} |\{x_i\in D : \Phi(x_i)=1, \ohcal(x_i)=\ycal_j \}| \\
    =\; & \frac{1}{|D|} |\{x_i\in D : \Phi(x_i)=1\}| \\
    =\; & \oquant(D)
\end{align*}

\noindent which completes the proof. Note this estimator is also consistent at the population level; i.e., that $E(\ohcal, \population)=\oquant(\population)$, since
\begin{align}
\begin{split}
\label{eq:ocal2oquant:popul1}
    E(\ohcal, \population)=\mathbb{E}_{X\sim Q_X}[\ohcal(X)]=\mathbb{E}_{X\sim Q_X}[\widetilde{Y}]=
    \int_{\mathcal{X}} \widetilde{y} \cdot Q(x)  \; dx 
\end{split}
\end{align}
where we replace $\ycal$ with its definition in the asymptomatic case (Equation~\ref{eq:ycalib:popul}) and change the integration variable $x$ in favour of $\ycal$, that subsumes all datapoints $x$ for which $\ohcal(x)=\ycal$ (that is, where $Q(\widetilde{Y}=\widetilde{y})=\int_{x\;:\;\ohcal(x)=\widetilde{y}} Q(X=x) \; dx$) thus obtaining

\begin{align}
\begin{split}
\label{eq:ocal2oquant:popul2}
    E(\ohcal, \population)
    &=\int_{0}^1 Q(Y=1|\widetilde{Y}=\widetilde{y}) \cdot Q(\widetilde{Y}=\widetilde{y}) \; d\ycal 
    =Q(Y=1)
\end{split}
\end{align}
\end{proof}


\begin{lemma}
\label{lemma:cal2cap}
$\ocalib \implies \oacc$, i.e., from a perfect calibrator we can attain a perfect estimator of classifier accuracy.
\end{lemma}

\begin{proof}
    Consider the following partition over $D$:
    \begin{align}
    \begin{split}
    \label{eq:partitionh}
        \Dplus&=\{x\in D: \hcrisp(x)=1\} \\
        \Dminus&=\{x\in D: \hcrisp(x)=0\}
    \end{split}
    \end{align}
    Let us generate two perfectly calibrated classifiers, $\ohcalplus=\ocalib(h,\Dplus)$ and $\ohcalminus=\ocalib(h,\Dminus)$ for each of these parts, respectively. Now, consider the following estimator of classifier accuracy:
    %
    \begin{align}
        \label{eq:calib2cap}
        E(h,D,\ocalib)\ddef\frac{\sum_{x_i\in\Dplus}\ohcalplus(x_i)+\sum_{x_j\in\Dminus}(1-\ohcalminus(x_j))}{|D|}
    \end{align}

    We can equivalently express the summations as the sum of unique values in $U_\oplus=\{\ohcalplus(x):x\in\Dplus\}$ and $U_\ominus=\{\ohcalminus(x):x\in\Dminus\}$, each multiplied by the number of times each value appears in the sequence:
    \begin{align}
    \label{eq:cal2cap:tmp}
        E(h,D,\ocalib)&=\frac{\sum_{\ycal_\oplus\in U_\oplus}\ycal_\oplus\cdot|\{x_i\in\Dplus:\ohcalplus(x_i)=\ycal_\oplus\}|}{|D|} \nonumber \\        &\quad + \frac{\sum_{\ycal_\ominus\in U_\ominus}(1-\ycal_\ominus)\cdot|\{x_j\in\Dminus:\ohcalminus(x_j)=\ycal_\ominus\}|}{|D|}
    \end{align}
    By the definition of perfect calibration it follows that 
    \begin{align}
    (1-\ycal)&=1-\frac{|\{x\in D : \Phi(x)=1, \ohcal(x)=\ycal \}|}{|\{x\in D : \ohcal(x)=\ycal \}|} \nonumber \\ 
    &=\frac{|\{x\in D : \Phi(x)=0, \ohcal(x)=\ycal \}|}{|\{x\in D : \ohcal(x)=\ycal \}|} \label{eq:complement}
    \end{align}
    \noindent i.e., that the computing the complement $(1-\ycal)$ of a well-calibrated value $\ycal$ has the effect of ``switching'' from $\Phi(x)=1$ to $\Phi(x)=0$ in the numerator of Equation~\ref{eq:hcal}. Note also that, by multiplying both sides of Equation~\ref{eq:hcal} by the denominator of the right-most part, we obtain:
    \begin{align*}
    \ycal\cdot|\{x\in D : \ohcal(x)=\ycal \}|=|\{x\in D : \Phi(x)=1, \ohcal(x)=\ycal \}|
    \end{align*}
    \noindent if we do the same multiplication in Equation~\ref{eq:complement}, we can turn back to Equation~\ref{eq:cal2cap:tmp} and continue as follows:
    \begin{align*}
        E(h,D,\ocalib)&=\frac{\sum_{\ycal_\oplus\in U_\oplus}|\{x_i\in\Dplus:\Phi(x_i)=1, \ohcalplus(x_i)=\ycal_\oplus\}|}{|D|} \nonumber \\        
        &\quad + \frac{\sum_{\ycal_\ominus\in U_\ominus}|\{x_j\in\Dminus:\Phi(x_j)=0, \ohcalminus(x_j)=\ycal_\ominus\}|}{|D|} \\
        &= \frac{|\{x_i\in\Dplus:\Phi(x_i)=1\}| + |\{x_j\in\Dminus:\Phi(x_j)=0\}|}{|D|} \\
        &= \frac{|\{x\in D : \Phi(x)=\hcrisp(x)\}|}{|D|} \\
        &= \oacc(h,D)
    \end{align*}

This estimator is consistent at the population level as well. Note that, in the asymptotic case, Equation~\ref{eq:calib2cap} can be rewritten as two expected values:
\begin{align*}
    E(h,\population,\ocalib)= 
    \mathbb{E}_{X\sim Q_{X|\hat{Y}=1}}[\ohcalplus(X)]Q(\hat{Y}=1)+\left(1-\mathbb{E}_{X\sim Q_{X|\hat{Y}=0}}[\ohcalminus(X)]\right)Q(\hat{Y}=0)
\end{align*}
By virtue of Equations~\ref{eq:ocal2oquant:popul1} and \ref{eq:ocal2oquant:popul2}, we can replace the expected values with the corresponding (conditional) distributions of the priors as follows
\begin{align*}
    E(h,D,\ocalib)
    =& \; Q(Y=1|\hat{Y}=1)Q(\hat{Y}=1)+(1-Q(Y=1|\hat{Y}=0))Q(\hat{Y}=0)\\
    =& \; Q(Y=1|\hat{Y}=1)Q(\hat{Y}=1)+Q(Y=0|\hat{Y}=0)Q(\hat{Y}=0)\\
    =& \; Q(Y=1, \hat{Y}=1)+Q(Y=0, \hat{Y}=0) \\
    =& \; Q(Y=\hat{Y})
\end{align*}
    
\end{proof}


\subsection{Access to a Perfect Quantifier $\oquant$}
\label{sec:accesstoquant}

In this case, we will assume the availability of a perfect quantifier and we will try to attain perfect models for calibration (Lemma~\ref{lemma:quant2cal}) and classifier accuracy prediction (Lemma~\ref{lemma:quant2cap}).

\begin{lemma}
\label{lemma:quant2cal}
$\oquant \implies \ocalib$, i.e., from a perfect quantifier we can attain a perfect calibrator.
\end{lemma}

\begin{proof}
Consider the following estimator $E(h,\oquant,D)\equiv\hcal$ that we implement using the quantifier oracle as:
\begin{align}
    \hcal(x)&\ddef \oquant(\{x'\in D : h(x)=h(x')\})    
\end{align}
To simplify the notation, we define the equivalence relation $x\sim_h x'$ which induces the equivalence class 
\begin{equation*}
    [x]_h=\{x'\in D : h(x)=h(x')\}
\end{equation*}
\noindent that is, the set of all datapoints to which the uncalibrated classifier $h$ assigns the same value it assigns to $x$. This allows us to express our estimator more compactly as:
\begin{equation}
\label{eq:ehcalquant}
    \hcal(x):=\oquant([x]_h)
\end{equation}
Similarly, we can define the equivalence relation $x\sim_{\hcal} x'$ that induces the equivalence class 
\begin{equation}
\label{eq:eqclasshcal}
[x]_{\hcal}=\{x'\in D : \hcal(x)=\hcal(x')\}    
\end{equation}
Using this notation, it is worth noting that if $\hcal$ is perfectly calibrated, then, by definition, it can be rewritten as $\hcal(x)=\oquant([x]_{\hcal})$ (see Equation~\ref{eq:hcal}). The proof thus consists of verifying whether $\oquant([x]_{\hcal})=\oquant([x]_h)$ holds true.

Let us begin by noting that $x\sim_h x' \implies x\sim_{\hcal}x'$ (i.e., that $x'\in[x]_h \implies x'\in[x]_{\hcal}$), while the converse implication is not necessarily true. This implicitly means that the partition of $D$ induced by $\sim_{\hcal}$ is \textit{coarser-grained} than the partition of $D$ induced by $\sim_h$ or, in other words, that an equivalence class of $\hcal$ might contain more than one equivalence classes of $h$. Let us define the indices that enumerate the parts of $[x]_{\hcal}$ as:
\begin{equation}
    I_x=\{i\in\mathbb{N}:[x']_h^{(i)}\subseteq[x]_{\hcal}\}
\end{equation}
Such that $[x]_{\hcal}=\bigcup_{i\in I_x}[x']_h^{(i)}$ with $i\neq j \implies[x]_h^{(i)}\cap[x']_h^{(j)}=\emptyset$.
%
%
In the following derivation, we leverage the fact that the prevalence of a partitioned set can be computed as the number of positives in each partition (which we can easily obtain by multiplying the prevalence of that part by its cardinality) divided by the total number of elements in the set.\footnote{That is, given a set $\mathcal{A}$ partitioned as $A_1,\ldots,A_k$, such that $\mathcal{A}=\bigcup_{i=1}^k A_i$ with $A_i\cap A_j=\emptyset$ for $i\neq j$, the prevalence of $\mathcal{A}$ can be obtained as $\oquant(\mathcal{A})=\frac{1}{|\mathcal{A}|}\sum_{i=1}^k\left(\oquant(A_i)\cdot|A_i|\right)$.} In this way
\begin{align}
    \oquant([x]_{\hcal}) &= \oquant(\cup_{i\in I_x}[x']_{h}^{(i)}) \\
        &= \frac{\sum_{i\in I_x}\oquant([x']_{h}^{(i)})\cdot|[x']_{h}^{(i)}|}{|[x]_{\hcal}|}
\end{align}
By the definition of $[x]_{\hcal}$ (see Equations~\ref{eq:ehcalquant}, \ref{eq:eqclasshcal}), all its elements are such that $x',x''\in[x]_{\hcal} \implies \oquant([x']_h)=\oquant([x'']_h)$, that is, the class prevalence of $p_i=\oquant([x']_{h}^{(i)})$ is the same for all equivalence classes that compose $[x]_{\hcal}$, i.e., $p_i=p_j, \forall i,j\in I_x$. Let us simply call $p$ this (constant) prevalence value that we factor out as follows:
\begin{equation}
    \oquant([x]_{\hcal}) = \frac{p\cdot\sum_{i\in I_x}|[x']_{h}^{(i)}|}{|[x]_{\hcal}|} = p
\end{equation}
The proof for the finite sample is completed by realising that $p$ is also the prevalence value of $[x]_h$, which is one of the elements in the partition, i.e., $p=\oquant([x]_h)$, and therefore if follows that $p=\oquant([x]_{\hcal})=\oquant([x]_h)$ meaning that the classifier $\hcal$ was indeed well-calibrated.

At the population level the estimator of Equation~\ref{eq:ehcalquant} corresponds to:
\begin{equation}
\label{eq:hcal:popul}
    \hcal(x)\ddef Q(Y=1|h(X)=h(x))
\end{equation}
If $\hcal$ were well-calibrated, then the following would also hold true
\begin{align}
\label{eq:hwellcal:popul}
    \hcal(x)
    &=Q(Y=1|\hcal(X)=\hcal(x))    
\end{align}
\noindent so we need to prove whether $Q(Y=1|h(X)=h(x))=Q(Y=1|\hcal(X)=\hcal(x))$ is true. To this aim, let us begin by applying the definition of conditional probability to Equation~\ref{eq:hwellcal:popul}:
\begin{align*}
    \hcal(x)
    &=\frac{Q(Y=1,\hcal(X)=\hcal(x))}{Q(\hcal(X)=\hcal(x))}
\end{align*}
Since, as we saw before, the random variable $\hcal(X)=\widetilde{Y}$ induces a coarser-grained partition of $\mathcal{X}$ than $h(X)$, we can ``unpack'' the probability density of the former as an aggregation of all the corresponding cases of the latter. Let $\widetilde{y}=\hcal(x)$ so we can rewrite:
\begin{align*}
Q(\hcal(X)=\ycal)&=\sum_{w\in\mathcal{W}_{\ycal}}Q(h(X)=w)    \\
Q(Y=1,\hcal(X)=\ycal)&=\sum_{w\in\mathcal{W}_{\ycal}}Q(Y=1,h(X)=w)    
\end{align*}
with $\mathcal{W}_{\ycal}=\{h(x'):\oquant([x']_h)=\ycal)\}$. 
Now, we can proceed from where we left off as follows:
\begin{align*}
    \hcal(x)    
    &=\frac{\sum_{w\in \mathcal{W}_{\ycal}} Q(Y=1,h(X)=w)}{\sum_{w\in \mathcal{W}_{\ycal}} Q(h(X)=w)}\\
    &=\frac{\sum_{w\in \mathcal{W}_{\ycal}} Q(Y=1|h(X)=w)Q(h(X)=w)}{\sum_{w\in \mathcal{W}_{\ycal}} Q(h(X)=w)}
\end{align*}
and since we know, by the construction of $\mathcal{W}_{\ycal}$, that the value of $Q(Y=1|h(X)=w)$ is the same constant value for all the choices of $w\in \mathcal{W}_{\ycal}$, and since we know $h(x)\in\mathcal{W}_{\ycal}$, we can factor it out from the summation as follows:
\begin{align*}
    \hcal(x)        
    &=Q(Y=1|h(X)=h(x))\cdot\frac{\sum_{w\in \mathcal{W}_{\ycal}} Q(h(X)=w)}{\sum_{w\in \mathcal{W}_{\ycal}} Q(h(X)=w)}\\
    &=Q(Y=1|h(X)=h(x))
\end{align*}
Which coincides with Equation~\ref{eq:hcal:popul} thus proving that $\hcal$ is indeed well-calibrated also at the population level.

\end{proof}


\begin{lemma}
\label{lemma:quant2cap}
$\oquant \implies \oacc$, i.e., from a perfect quantifier we can attain a perfect estimator of classifier accuracy.
\end{lemma}

\begin{proof}
The proof is immediate by reusing Lemma~\ref{lemma:quant2cal} and Lemma~\ref{lemma:cal2cap}:
    \begin{align*}
        \oquant \implies \ocalib &\quad \mathrm{(Lemma~\ref{lemma:quant2cal})}   \\
        \ocalib \implies \oacc &\quad \mathrm{(Lemma~\ref{lemma:cal2cap})}
    \end{align*}
\end{proof}


\subsection{Access to a Perfect Estimator of Classifier Accuracy $\oacc$}
\label{sec:accesstocap}

We now assume the availability of a perfect estimator of classifier accuracy and try to attain perfect models for 
quantification  (Lemma~\ref{lemma:acc2quant}) and calibration (Lemma~\ref{lemma:acc2calib}).

\begin{lemma}
\label{lemma:acc2quant}
$\oacc \implies \oquant$, i.e., from a perfect estimator of classifier accuracy we can attain a perfect quantifier.
\end{lemma}

\begin{proof}
    Let us reuse the partitions $\Dplus$ and $\Dminus$ produced by the classifier predictions (see Equations~\ref{eq:partitionh}) and define
    %
    %
    the following estimator of class prevalence:
    \begin{align}
    \label{eq:oacc2oquant}
        E(h,D,\oacc) &\ddef \oacc(h,\Dplus) \frac{|\Dplus|}{|D|} + (1-\oacc(h,\Dminus)) \frac{|\Dminus|}{|D|}
    \end{align}
    From the definition of a perfect estimator of classifier accuracy (Equation~\ref{eq:oacc}), it follows that
    \begin{align*}
        E(h,D,\oacc) &=\; \frac{|\{x\in \Dplus : \hcrisp(x)=\Phi(x)\}|}{|\Dplus|}\frac{|\Dplus|}{|D|} + \left(1-\frac{|\{x\in \Dminus : \hcrisp(x)=\Phi(x)\}|}{|\Dminus|}\right) \frac{|\Dminus|}{|D|} \\        
        &=\; \frac{|\{x\in \Dplus : \hcrisp(x)=\Phi(x)\}|}{|D|} + \frac{|\Dminus|-|\{x\in \Dminus : \hcrisp(x)=\Phi(x)\}|}{|D|} \\
        &=\; \frac{|\{x\in \Dplus : \hcrisp(x)=\Phi(x)\}|}{|D|} + \frac{|\{x\in \Dminus : \hcrisp(x)\neq\Phi(x)\}|}{|D|}         
    \end{align*}
    and since we know $\hcrisp(x)=1$ for every $x\in\Dplus$, and $\hcrisp(x)=0$ for every $x\in\Dminus$, we can proceed as follows
    \begin{align*}  
        E(h,D,\oacc) 
        &= \frac{|\{x\in \Dplus : \hcrisp(x)=1 ,  \Phi(x)=1 \}| + |\{x\in \Dminus : \hcrisp(x)=0 ,  \Phi(x)=1\}|}{|D|} \\
        &= \frac{|\{x\in D : \Phi(x)=1\}|}{|D|}\\
        &= \oquant(D)
    \end{align*}

That this estimator is also consistent at the population level follows from the fact that 
\begin{align*}
    \oacc(h,\population_\oplus)&=Q(Y=\hat{Y}|\hat{Y}=1) \\
    1-\oacc(h,\population_\ominus)&=Q(Y\neq\hat{Y}|\hat{Y}=0)
\end{align*}
which allows us to rewrite Equation~\ref{eq:oacc2oquant} as 
\begin{align*}                   
    E(h,\population,\oacc)&=Q(Y=\hat{Y}|\hat{Y}=1)\;Q(\hat{Y}=1)+Q(Y\neq\hat{Y}|\hat{Y}=0)\;Q(\hat{Y}=0)\\
    &=Q(Y=\hat{Y},\hat{Y}=1)+Q(Y\neq\hat{Y},\hat{Y}=0)\\
    &=Q(Y=1,\hat{Y}=1)+Q(Y=1,\hat{Y}=0)\\
    &=Q(Y=1)
\end{align*}
    
\end{proof}


\begin{lemma}
\label{lemma:acc2calib}
$\oacc \implies \ocalib$, i.e., from a perfect estimator of classifier accuracy we can attain a perfect calibrator.
\end{lemma}

\begin{proof}
    The proof follows by transitivity using Lemma~\ref{lemma:acc2quant} and Lemma~\ref{lemma:quant2cal}:
    \begin{align*}
        \oacc \implies \oquant &\quad \mathrm{(Lemma~\ref{lemma:acc2quant})}   \\
        \oquant \implies \ocalib &\quad \mathrm{(Lemma~\ref{lemma:quant2cal})}
    \end{align*}
\end{proof}

\section{Methods}
\label{sec:methods}

In this section, we reuse the intuitions behind the proofs of Section~\ref{sec:reductions} to propose direct adaptations of methods originally proposed for calibration, quantification, or classifier accuracy prediction, that allow them be applied to the other problems. Additionally, we present two new original methods for calibration which exploit intuitions borrowed from the quantification literature in Section~\ref{sec:methods:new}.

The problem setting is as follows. Let us assume to have access to a labelled $\trset$, drawn from the training distribution $P$, and an unlabelled set $\teset$, drawn from the test distribution $Q$ (see Section~\ref{sec:notation}).
We assume $P$ and $Q$ are related to each other by  CS or LS.
$\trset$ is partitioned, in a stratified way, into $\hset$, that we use to generate a raw (uncalibrated) probabilistic classifier $h$, and $\valset$, that we use to generate the corresponding estimator depending on the target task, as follows:
%
\begin{itemize}
    \item In calibration experiments, $h$ is the classifier we need to calibrate. The calibrator $\calib$ is given access to $\valset$ and the (unlabelled) set $\teset$, and is required to generate $\hcal$, a variant of $h$ calibrated for $\teset$.
    \item All quantification methods we will consider are of type \emph{aggregative}, i.e., are methods that, in order to predict the class priors, learn a permutation-invariant function aggregating the posterior probabilities returned by a classifier. In this case, $h$ acts as the surrogate classifier that a quantifier $\rho$ uses to represent the datapoints in $\valset$ on which the aggregation function is modelled. The same classifier is also used to represent the datapoints in $\teset$ with which the aggregation function predicts the class proportions.
    \item When classifier accuracy prediction is the goal, $h$ represents the classifier whose accuracy we want to estimate on $\teset$, and $\valset$ is used to train a predictive model of classifier performance.
\end{itemize}



\subsection{Direct Adaptations}

By ``direct adaptation'', we mean the process of taking one of the estimators presented in the proofs of Section~\ref{sec:reductions} and replacing the oracle it depends on with an actual surrogate method. This also implies making decisions about how the surrogate model is trained, since the oracles of the proofs were assumed to be ready-to-use functions that necessitate no training. These proofs also relied on continuous random variables which sometimes need to be quantized in order to become operable.


We denote by $M_{a2z}$ the direct adaptation of method $M$, originally proposed for problem $a$, to be used as a surrogate method for addressing problem $z$, with $a,z\in\{\calib,\quant,\acc\}$. For example, by $\text{PACC}_{\quant 2 \calib}$ we denote an adaptation of the quantification method PACC to generate a calibration method.


\begin{itemize}
    \item \textbf{The $\bm{\calib 2 \quant}$ method} is the simplest one. It uses a calibrator method $\calib$ to generate, based on $\valset$ (labelled) and $\teset$ (unlabelled), a version $\hcal$ of $h$ calibrated for $\teset$. The quantification task is performed by computing the average of the calibrated posteriors, as a direct application of Lemma~\ref{lemma:cal2quant} (Equation~\ref{eq:cal2quantmean}).

    \item \textbf{The $\bm{\calib 2 \acc}$ method} first uses $h$ to split $\valset$ and $\teset$ into $\{\valset^{\oplus},\valset^{\ominus}\}$ and $\{\teset^{\oplus},\teset^{\ominus}\}$, respectively (see Equation~\ref{eq:partitionh}). Two calibrated versions $\hcal_{\oplus}$ and $\hcal_{\ominus}$ are generated using the corresponding partitions, and the final estimate of classifier accuracy is returned by applying Equation~\ref{eq:calib2cap} of Lemma~\ref{lemma:cal2cap} on $\teset$.

    \item \textbf{The $\bm{\quant 2 \calib}$ method} first trains a quantifier $\quant$ using the posterior probabilities of the instances in $\valset$. The method requires quantizing the posteriors of the test datapoints estimated by $h$. Given the hyperparameter $b$ indicating the number of bins, we generate an isometric binning of the posteriors probabilities of the test instances, and then estimate, for the test instances falling within each bin $i$, the prevalence $p_i$ of the positive class using the quantification method $\rho$. Such estimates might not be monotonically increasing (e.g., $\hat{p}_i$ can happen to be higher than $\hat{p}_j$, with $i<j$); in such cases we enforce monotonicity by simply applying the rule $p'_i=\max\{\hat{p}_{i-1}, \hat{p}_i\}$. To reduce abrupt changes between nearby bins, we also apply average smoothing with a window size of 1; that is, we zero-pad the sequence adding $p'_0=0$ and $p'_{b+1}=1$ and then apply $p''_i=\frac{1}{3}(p'_{i-1}+p'_i+p'_{i+1})$ for $1\leq i\leq b$.
    Finally, and in order to avoid that at most $b$ different calibrated posteriors are returned, we treat every value $p''_i$ as the calibrated output for the bin center $c_i$, and the final calibration map is defined as a linear interporlation in the (input, output) sequence $[(0,0), (c_1, p''_1), \ldots , (c_b, p''_b),(1,1)]$.

    \item \textbf{The $\bm{\quant 2 \acc}$ method} consists of first splitting $\valset$ into  $\{\valset^{\oplus},\valset^{\ominus}\}$ using $h$, and then learning two dedicated quantifiers $\quant^{\oplus}$ and $\quant^{\ominus}$ on the corresponding partitions. At inference time, the test set $\teset$ is split into $\{\teset^{\oplus},\teset^{\ominus}\}$ using $h$, and we use our quantifiers to estimate the proportion of positive instances in the predicted positives ($\hat{p}^{\oplus}=\quant^{\oplus}(\teset^{\oplus})$) and the proportion of positive instances in the predicted negatives ($\hat{p}^{\ominus}=\quant^{\ominus}(\teset^{\ominus})$); the final estimate of classifier accuracy is computed as $$\hat{\text{acc}}=\frac{\hat{p}^{\oplus}\cdot|\teset^{\oplus}|+(1-\hat{p}^{\ominus})\cdot|\teset^{\ominus}|}{|\teset|}$$

    \item \textbf{The $\bm{\acc 2 \quant}$ method} partitions $\valset$ into  $\{\valset^{\oplus},\valset^{\ominus}\}$ using $h$, and then learns two dedicated estimators of classifier accuracy, $\acc^{\oplus}$ and $\acc^{\ominus}$, on the corresponding partitions. At inference time, the test set $\teset$ is split into $\{\teset^{\oplus},\teset^{\ominus}\}$ using $h$, and two estimates of classifier accuracy are computed on the corresponding sets as $\hat{\text{acc}}^{\oplus}=\acc^{\oplus}(h,\teset^{\oplus})$ and $\hat{\text{acc}}^{\ominus}=\acc^{\ominus}(h,\teset^{\ominus})$. The final prevalence  is computing by plug-in these estimates into Equation~\ref{eq:oacc2oquant} of Lemma~\ref{lemma:acc2quant}.

    \item \textbf{The $\bm{\acc 2 \calib}$ method} is similar to the $\quant2\calib$ approach; it trains a classifier accuracy prediction method $\acc$ on $\valset$ and quantizes the posteriors of the instances in $\teset$ produced by $h$ using isometric binning for $b$ bins. A sequence $(a_1, \ldots, a_b)$ of accuracy estimates for the $b$ bins is generated using $\acc$. In this case, $b$ needs to be even, so that $a_i$  accounts for the fraction of positive instances in bins above 0.5 (i.e., when $\frac{b}{2}< i\leq b$), while the fraction of positive instances for bins below 0.5 (i.e., when $1\leq i\leq\frac{b}{2}$) is instead estimated as $(1-a_i)$. This way, a sequence of (input, output) values $[(0,0), (c_1, 1-a_1), \ldots,(c_{\frac{b}{2}},1-a_{\frac{b}{2}}), (c_{\frac{b}{2}+1},a_{\frac{b}{2}+1}), \ldots, (c_b,a_b),(1,1) ]$ for the calibration map is generated. We impose monotonicity and smooth the sequence, which is finally used to define a calibration map via interpolation as before.

\end{itemize}

\subsection{Original New Methods for Calibration}
\label{sec:methods:new}

In this section, we propose two new calibration methods that gain inspiration from quantification methods.

\subsubsection{PacCal: a new calibration method based on PACC}

The first of the new calibration methods we propose gains inspiration from the ``probabilistic adjusted classify and count'' (PACC) quantification method \cite{Bella:2010kx} and is designed to calibrate a classifier under LS. Let us denote by $\hcal$ (instead of by $h$) our classifier, to emphasize the fact that it needs to be a probabilistic classifier (we do not make any assumption about whether it has been calibrated on the training distribution or not, though).
Let us briefly recap the rationale behind PACC, which rests on the following observation:
\begin{align*}
    Q(\tilde{Y}=1) &= Q(\tilde{Y}=1|Y=1) \cdot Q(Y=1)+ Q(\tilde{Y}=1|Y=0) \cdot Q(Y=0) \\
                &= P(\tilde{Y}=1|Y=1) \cdot Q(Y=1)+ P(\tilde{Y}=1|Y=0) \cdot (1-Q(Y=1) )
\end{align*}
\noindent where the replacement of the test class-conditional distributions of classifier predictions ($Q$) with the training ones ($P$) follows from Remark~\ref{remark:pps} (Section~\ref{sec:definitions}). It then follows that:
\begin{align*}
    Q(Y=1) &= \frac{Q(\tilde{Y}=1)-P(\tilde{Y}=1|Y=0)}{P(\tilde{Y}=1|Y=1)-P(\tilde{Y}=1|Y=0)} 
\end{align*}
The class prevalence estimate of PACC is attained by estimating $Q(\tilde{Y}=1)$ as the average of the posterior probabilities returned by $\hcal$ on the test instances (i.e., the expected value on the empirical test sample), and replacing the class-conditional distributions with estimates of the true positive rate ($\tpr$) and false positive rate ($\fpr$) of the classifier obtained on held-out validation data ($\valset$), i.e.:
\begin{align}
\label{eq:pacc}
    \hat{p}^{\text{PACC}} &= \frac{\mathbb{E}_{x\sim \teset}[\hcal(x)]-\fprhat}{\tprhat-\fprhat} 
\end{align}
Now, note that the adjustment implemented in Equation~\ref{eq:pacc} is linear; we can therefore equivalently rewrite it as an affine transformation with scaling factor $\beta=\frac{1}{\tprhat-\fprhat}$ and intercept $\gamma=-\frac{\fprhat}{\tprhat-\fprhat}$ as follows:
\begin{align}
    \hat{p}^{\text{PACC}} = \mathbb{E}_{x\sim \teset}[h(x)]\cdot \beta+\gamma = \mathbb{E}_{x\sim \teset}[h(x)\cdot \beta+\gamma]
\end{align}
This suggests that we can update the posterior probability of each of the test datapoints to account for the new prior using a linear transformation. This procedure thus yields a calibrated classifier:
\begin{equation}
    \hcal'(x)\ddef \hcal(x)\cdot\beta+\gamma
\end{equation}
However, the linear transformation is not guaranteed to lie in the $[0,1]$ interval. This is inherited from the well-known instability\footnote{Some methods from the quantification literature have been proposed to better handle these instabilities, including the T50, MAX, X, Median Sweep (MS), and its variant MS2; see \cite{Forman:2008kx}.} of PACC when the denominator of Equation~\ref{eq:pacc} is small. When this happens, we simply apply a sigmoid function ($\sigma$) after the linear transformation. 
We thus dub PacCal$^\sigma$ (short for Probabilistic Adjusted Classify and Count for Calibration). The extent to which this method provides well-calibrated posteriors, in the sense of a calibration-oriented loss, is later discussed in the experimental section.

\subsubsection{DMcal: a new calibration method based on HDy}
\label{sec:methods:new:dmcal}

The second original calibration method we propose is also based on a quantification method designed to work under LS conditions: the so-called Hellinger-Distance-y (HDy) method of~\cite{Gonzalez-Castro:2013fk}. 

HDy is a distribution-matching quantification method which relies on histograms of the posterior probabilities returned by our classifier for representing the class-conditional densities of the positive (histogram $H^{\oplus}_{val}$) and negative (histogram $H^{\ominus}_{val}$) examples in the validation set $\valset$. At inference time, the histogram $H_{te}$ of the posterior probabilities of the test datapoints is generated, and the mixture parameter $p$ that yields the closest match between the ``validation mixture'' ($H^{\oplus}_{val}$ and $H^{\ominus}_{val}$) and the test distribution ($H_{te}$), in terms of the Hellinger Distance (HD), is returned as the class prior estimate, i.e.:
\begin{equation}
    \hat{p}^{\text{HDy}}=\arg\min_{p\in[0,1]} \text{HD}\left(\left(p H^{\oplus}_{val} + (1-p)H^{\ominus}_{val}\right), H_{te}\right)
\end{equation}
Given two histograms of densities $H^A=(H^A_1,\ldots,H^A_b)$ and $H^B=(H^B_1,\ldots,H^B_b)$ of $b$ bins, the Hellinger Distance is defined as:
\begin{equation*}
    \operatorname{HD}(H^A,H^B) = \sqrt{1-\sum_{i=1}^b \sqrt{H^A_i \cdot H^B_i}} 
\end{equation*}

\begin{figure}[b!]
    \centering
    \includegraphics[width=0.32\linewidth]{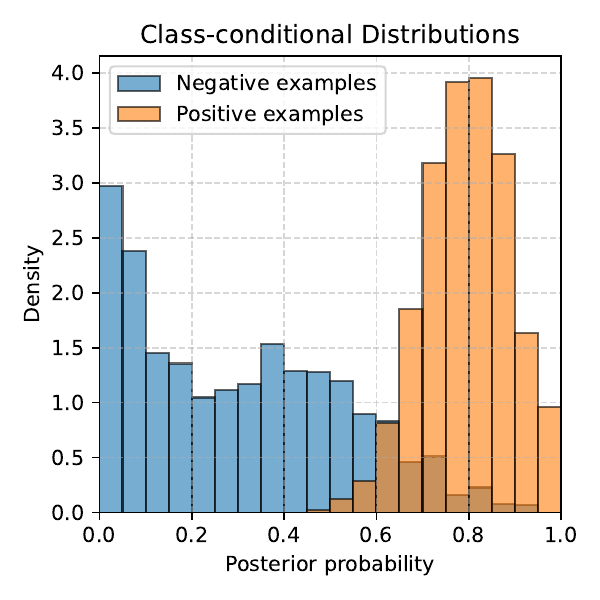}
    \includegraphics[width=0.32\linewidth]
    {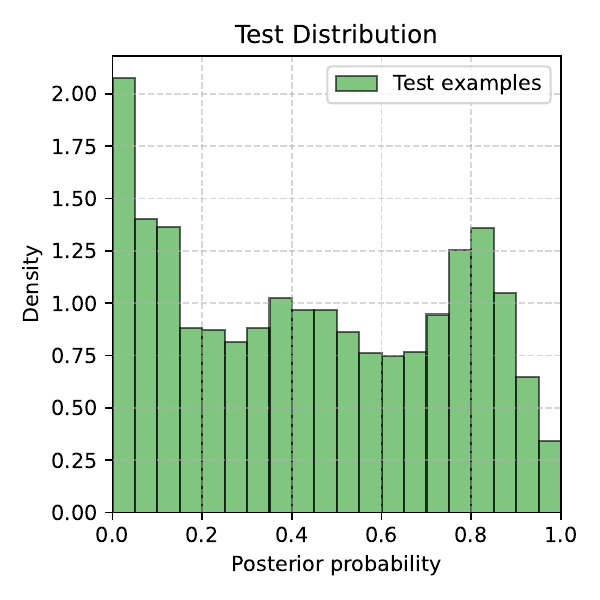}
    \includegraphics[width=0.32\linewidth]
    {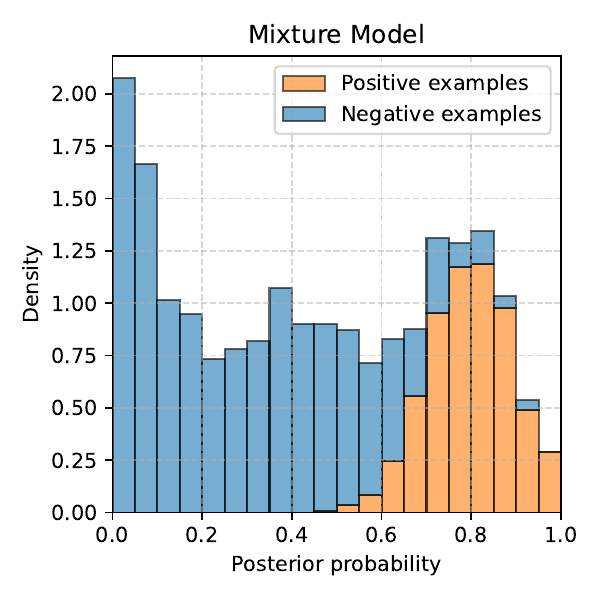}
    \includegraphics[width=0.32\linewidth]{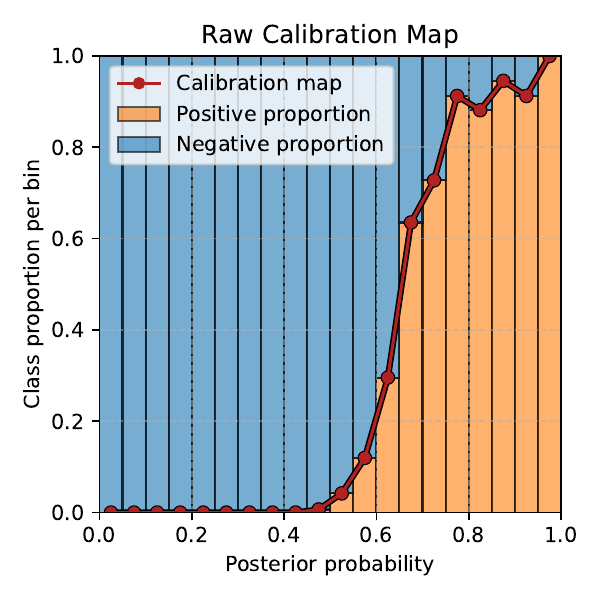}
    \includegraphics[width=0.32\linewidth]{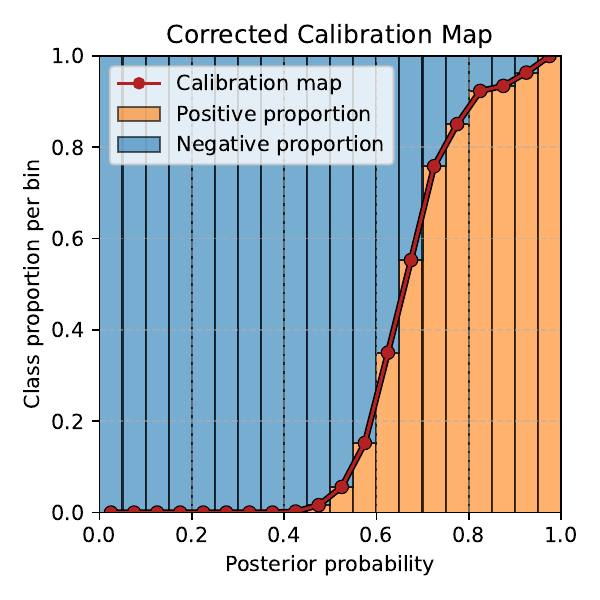}    
    \caption{Graphical explanation of DMcal. First row: The left-most panel shows the class-conditional distributions, represented as histograms of posterior probabilities ($H_{val}^{\ominus}$ in blue, and $H_{val}^{\oplus}$ in orange) modelled on validation data. The central panel displays the test histogram ($H_{te}$). The right-most panel displays the mixture model of $H_{val}^{\ominus}$ and $H_{val}^{\oplus}$ obtained with the mixture parameter $p=0.3$ which yields the closest mixture to the test histogram in terms of HD. Second row: the left-most panel shows a ``raw'' calibration map computed on the proportion of positive and negative contributions of each bin of the mixture model; the right-most panel shows the corrected calibration map after imposing monotonicity and smoothing the raw calibration map.}   
    \label{fig:DMcal_working}
\end{figure}

Once the test prevalence of the positive instances has been estimated, we can interpret the proportion of positives within each bin of the mixture as a calibrated value for that bin. More formally, the calibrated output $v_i$ for the $i$-th bin is given by:
\begin{equation}
\label{eq:calvalue}
    v_i=\frac{\hat{p}^{\text{HDy}} \cdot H^{\oplus}_{val,i}}{\hat{p}^{\text{HDy}} \cdot H^{\oplus}_{val,i} + (1-\hat{p}^{\text{HDy}}) \cdot H^{\ominus}_{val,i}}
\end{equation}
To enable continuous calibration outputs and avoid returning at most $b$ distinct output values, the calibration map is defined as a piecewise linear interpolation over the sequence of points $[(0,0),(c_1,v_1),\ldots,(c_b,v_b),(1,1)]$, where $c_i$ denotes the center of the $i$-th bin and $v_i$ the associated calibrated value obtained with Equation~\ref{eq:calvalue}.
Finally, to avoid artifacts due to random sampling (i.e., jagged outputs), the sequence of calibrated values is post-processed to ensure monotonicity and then smoothed before defining the final calibration map.
We call this method DMcal, for Distribution-Matching based Calibration; the whole process behind DMcal is illustrated in Figure~\ref{fig:DMcal_working}.

Note that this method is non-parametric, i.e., it makes no assumption about the underlying distributions (other than assuming that the underlying distributions are related to each other through LS).
DMcal is close in spirit to empirical binning \cite{empiricalbinning:zadrozny:2001}, with the key difference that the proportion of positives in each bin is taken from a mixture model re-parameterized with an estimated prior, and that the calibration map is post-processed to allow for continuous, monotonically increasing, and smooth outputs. 

\section{Experiments}
\label{sec:experiments}

This section presents our experimental evaluation of the newly proposed methods from Section~\ref{sec:methods}. These methods, inspired by approaches developed for other tasks, are assessed in terms of their ability to compete with techniques natively designed for the target problem. The section is structured into three parts: calibration (Section~\ref{sec:exp:calibration}), quantification (Section~\ref{sec:exp:quantification}), and classifier accuracy prediction (Section~\ref{sec:exp:cap}).

Before discussing the results we have obtained, we first describe the reference methods we choose (Section~\ref{sec:exp:reference}), and how we simulate CS (Section~\ref{sec:exp:covshift}) and LS (Section~\ref{sec:exp:pps}) in our experiments.

\subsection{Reference Methods}
\label{sec:exp:reference}

The methodology we follow in our experiments consists of confronting three representative reference methods from each problem area against adaptations of reference methods originally proposed for the other two problems. In addition to these 9 reference methods (3 problems $\times$ 3 methods), we also consider, in some cases, additional reference methods that enrich the comparison under specific types of dataset shift, as well as a baseline method representative of a standard IID-based solution for each problem.
The reference methods selected from each problem are listed below.

\begin{itemize}

    \item As our reference methods from calibration, we select CPCS \cite{cpcs2020icml} and TransCal \cite{transcal2020}, which are specialized for CS, and LasCal \cite{popordanoska2024lascal}, which is specialized for LS. As an additional reference method for the calibration experiments, we also consider the Head2Tail \cite{head2tail2023} calibrator which was originally described as a method for CS, but that, for reasons discussed in Section~\ref{sec:related:calibration}, we believe is more suited for LS. For all these methods, we relied on the implementations made available by the authors of LasCal.\footnote{ \url{https://github.com/tpopordanoska/label-shift-calibration}. This repository builds on top of previous implementations including the \texttt{abstention} package for calibration  \url{https://github.com/kundajelab/abstention}, Jiahao Chen's calibration package for Head2Tail \url{https://github.com/JiahaoChen1/Calibration}, or the TransCal package
    \url{https://github.com/thuml/TransCal}, among others.}
    As our IID baseline for calibration experiments, we consider Platt scaling \cite{Platt:2000fk} and relied on the implementation of \texttt{scikit-learn}. 
    
    \item The three reference methods we choose for quantification are PACC \cite{Bella:2010kx}, EMQ \cite{Saerens:2002uq}, and KDEy \cite{KDEyMoreo:2025}, three well-established methods for LS (Section~\ref{sec:related:quantification}). In some cases, we also consider PCC \cite{Bella:2010kx} as an additional reference method for CS, and calibrated variants of EMQ \cite{Alexandari:2020dn}.
    The IID solution for quantification experiments is the naive ``classify and count'' (CC) approach \cite{Forman:2005fk}. In all cases, we rely on the implementations provided by the \texttt{QuaPy} library\footnote{\url{https://github.com/HLT-ISTI/QuaPy}} for quantification \cite{QuaPy:2021bs}, and leave all hyperparameters set at their default values.
    
    \item For classifier accuracy prediction, we consider ATC \cite{Garg:2022qv}, DoC \cite{Guillory:2021so}, and LEAP \cite{Volpi:2024ye} as our reference methods (Section~\ref{sec:related:cap}). The former two methods can deal with both covariate and LS, while the latter is instead bounded to LS. DoC requires a protocol for generating validation samples out of $\valset$. To this aim, we generate 100 samples $\valset^1,\ldots,\valset^{100}$ of $|\valset^{i}|=250$ instances each using the APP protocol in experiments addressing LS (this protocol is later described in Section~\ref{sec:exp:pps:protocol}). Since the entire validation set comes from the training distribution we are not allowed to mimic the CS sampling generation protocol (which necessitates from labelled examples of the test distribution --Section~\ref{sec:exp:covshift:protocol}); in this case, we simply draw 100 random samples, of 250 instances each, out of $\valset$.
    As our IID baseline, we consider a method that estimates the classifier's accuracy on the held-out validation data $\valset$. We dub this method ``Naive'' because $\valset$ comes from the training distribution, and so the method is  agnostic to the presence of dataset shift. We have adapted the implementations of these methods available in the repository of LEAP.\footnote{\url{https://github.com/lorenzovolpi/LEAP} This repository also reuses the original code for the method ATC available at \url{https://github.com/saurabhgarg1996/ATC_code/}}
    
\end{itemize}

The new methods PacCal$^\sigma$ and DMCal rely on the implementations of PACC and HDy, respectively, which are available in \texttt{QuaPy}. For DMCal, we set the number of bins to 8, which is the default value for the underlying quantifier HDy in \texttt{QuaPy}.

The code that reproduces all the experiments, and implements the adaptation methods as well as the new methods PacCal$^\sigma$ and DMCal, is available online at \url{https://github.com/AlexMoreo/UnifyingProblems}.

\subsection{Experiments Simulating Covariate Shift}
\label{sec:exp:covshift}

\subsubsection{Sample Generation Protocol}
\label{sec:exp:covshift:protocol}
The protocol we use to simulate CS on real data consists of selecting two classification datasets, $A$ and $B$, each annotated with binary labels that are comparable across datasets. Each dataset is split into a training set ($\trset^A$, $\trset^B$) and a test set ($\teset^A$, $\teset^B$). The training set of $A$ is further divided into a classifier training set ($\hset^A$) and a validation set ($\valset^A$).
A classifier $h$ is then generated on $\hset^A$ from dataset $A$, which plays the role of the source domain.
We then construct a sequence of 100 test sets $\teset^1, \ldots, \teset^{100}$ of $|\teset^i|=250$ instances each, as a progressively interpolated mixture of the two test sets $\teset^A$ and $\teset^B$. 
Specifically, for each $i\in\{1,\ldots,100\}$, the test set $\teset^i$ consists of $n_A^{(i)}=\lceil 250 \cdot \left(1-\frac{i-1}{99}\right)\rceil$ instances randomly drawn from $\teset^A$, and $n_B^{(i)}=250-n_A^{(i)}$ instances randomly drawn from $\teset^B$. 
Note that every test set $\teset^i$ is constructed from scratch, and not as a modification of the previous one $\teset^{i-1}$, i.e., the 
the instances drawn from $A$ and $B$ in any two test sets are independent of each other and result from separate random draws.
Since $A$ and $B$ originate from different, yet related, domains, this protocol effectively simulates varying levels of CS. The first set $\teset^1$, composed entirely of instances instances from the source domain $A$, conforms to the IID assumption (i.e., there is no shift), whereas the last set $\teset^{100}$, composed entirely of instances from the target domain $B$, represents the highest level of CS.
For each test set $\teset^i$, the model under evaluation (i.e., for calibration, quantification, or classifier accuracy prediction) is given access to the classifier ($h$) and the validation set ($\valset^A$) from the source domain $A$.

We keep sample consistency across all experiments, in order to guarantee all methods are tested on the exact same training, validation, and test sets.

\subsubsection{Datasets}

As our datasets for CS experiments we choose three datasets of reviews for sentiment classification: IMDb\footnote{Data available at \url{https://huggingface.co/datasets/stanfordnlp/imdb}} (\texttt{imdb}) \cite{imdb}, Rotten Tomatoes\footnote{Data available at \url{https://huggingface.co/datasets/cornell-movie-review-data/rotten_tomatoes}} (\texttt{rt}) \cite{rt}, and Yelp Reviews\footnote{Data available at \url{https://huggingface.co/datasets/Yelp/yelp_review_full}} (\texttt{yelp}) \cite{yelp}. 
IMDb and Rotten Tomatoes already come with binary labels. For the Yelp Reviews we convert the 5-star ratings into binary labels by removing all instances labelled with 3 stars, and then treating 1 and 2 stars as negative reviews and 4 and 5 stars as positive reviews.
The three datasets are related, as they all contain opinion data, yet they differ in various aspects. Although IMDb and Rotten Tomatoes both consists of movie reviews, the reviews in IMDb are generated by registered users and are typically longer than those from Rotten Tomatoes, which are instead short reviews generated by professional critics. The Yelp Reviews is more different, as its reviews are not about movies, but rather about different businesses. 
All datasets are perfectly balanced; some statistics of the datasets we use (after filtering) are summarized in Table~\ref{tab:data:cs}.

\begin{table}[tbh]
    \caption{Statistics of the sentiment datasets used in our CS experiments} 
    \vspace{0.5cm}
    \centering
    \begin{tabular}{cccc}
       \toprule
       Dataset  & \#Training & \#Validation & \#Test \\ \midrule
       IMDb (\texttt{imdb}) & 25K & --- & 25K \\
       Rotten Tomatoes (\texttt{rt})  & 8.53K & 1.07K & 1.07K \\
       Yelp Reviews (\texttt{yelp}) & 520K & --- & 40K \\ \bottomrule
    \end{tabular}    
    \label{tab:data:cs}
\end{table}

When we apply the sample generation protocol discussed above, the three datasets take turns in a ``round robin'' fashion, acting as either the source or the target domain. For example, ``\texttt{yelp}$\rightarrow$\texttt{rt}'' denotes an experiment where the Yelp Reviews is used as the source domain and test samples are generated by mixing documents from the target domain Rotten Tomatoes. Since we are interested in simulating CS, we exclude combinations in which the same dataset plays both the role of source and target domain; we thus end up considering a total of six different source-target combinations.

\subsubsection{Classifiers}

As our classifiers, we consider three well-known pre-trained language models available at Huggingface\footnote{\url{https://huggingface.co/}}: 
\texttt{google-bert/bert-base-uncased} (hereafter BERT) \cite{devlin2019bert},
\texttt{distilbert/distilbert-base-uncased} (hereafter DistilBERT) \cite{sanh2019distilbert}, and
\texttt{FacebookAI/roberta-base} (hereafter RoBERTa) \cite{liu2019roberta}.

In all cases, we fine-tune the classification head of each pre-trained model on the available training data for 5 epochs, using a learning rate of 5E-4 and a batch size of 64 documents. 
In order to prevent overfitting, we apply early stop after 5 consecutive evaluation steps (we trigger an evaluation round every 500 training steps) showing no improvement in terms of classification $F_1$ in validation data. When the validation set is not available (e.g., in IMDb and Yelp Reviews), we extract a held-out sample of 5,000 documents from the training set with stratification.

\subsection{Experiments Simulating Label Shift}
\label{sec:exp:pps}

\subsubsection{Sample Generation Protocol}
\label{sec:exp:pps:protocol}
To simulate LS, we adopt the so-called artificial-prevalence protocol (APP), a standard protocol commonly used in quantification evaluation (see \cite{Lequa2022Overview} for further details). Given a test set $\teset$, we generate a series of test samples $\teset^1,\ldots,\teset^{100}$ each consisting of $|\teset^i|=250$ instances. For each sample $\teset^i$, the APP protocol first draws a target prevalence value $p^{(i)}$, uniformly at random from the interval $[0,1]$, which represents the desired prevalence of the positive class. Then, it draws $n_{\oplus}^{(i)}=\lceil p^{(i)}\cdot250\rceil$ \emph{positive} and $n_{\ominus}^{(i)}=250-n_{\oplus}^{(i)}$ \emph{negative} instances, uniformly at random from $\teset$. Regardless of the original training prevalence, this protocol ensures the generation of test samples with varying levels of LS, ranging from low shift (when $p^{(i)}$ is close to the training prevalence), to high shift (this happens when the training prevalence is close to 0 or 1 and the test prevalence $p^{(i)}$ is chosen to be close to 1 or 0, respectively).

We keep sample consistency across all experiments, in order to guarantee all methods are tested on the exact same training, validation, and test sets.



\subsubsection{Datasets}
For the LS experiments, we consider the 10 largest binary classification datasets from the UCI Machine Learning repository\footnote{Data available at \url{https://archive.ics.uci.edu}} that have been used in past quantification research (e.g., \cite{Perez-Gallego:2017wt,QuaPy:2021bs}). 
The final selection consists of datasets 
\texttt{cmc.3}, 
\texttt{yeast}, 
\texttt{semeion}, 
\texttt{wine-q-red}, 
\texttt{ctg.1}, 
\texttt{ctg.2}, 
\texttt{ctg.3}, 
\texttt{spambase}, 
\texttt{wine-q-white}, and
\texttt{pageblocks.5}.
The dataset size ranges from a minimum of 1,473 (\texttt{cmc.3}) to a maximum of 5,473 (\texttt{pageblocks.5}). The datasets consist of tabular data, and the number of features ranges from a minimum of 8 (\texttt{yeast}) to a maximum of 256 (\texttt{semeion}).
The selection includes imbalanced datasets (e.g., \texttt{pageblocks.5} contains only 2.1\% of positive instances, and \texttt{ctg.1} contains no less than 77,8\% of positives) and almost perfectly balanced ones (e.g., \texttt{wine-q-red}).
We fetch these datasets using the \texttt{QuaPy} package, which also applies a standardization to the feature columns.
Further details about the datasets can be consulted online.\footnote{\url{https://hlt-isti.github.io/QuaPy/manuals/datasets.html#binary-datasets}}

These datasets do not come with predefined data partitions. For each dataset, we thus generate a test partition $\teset$ consisting of 30\% of the whole data with stratification; the rest of the dataset is randomly split into a training set $\trset$ and a validation set $\valset$ of equal size. All splits are generated with stratification.

\subsubsection{Classifiers}

As our classifiers for the tabular data we consider the following classifiers: 
\begin{itemize}
    \item Logistic Regression: is an example of a probabilistic classifier generating reasonably well-calibrated posterior probabilities (not in vain, the logistic function is used for calibration purposes \cite{Platt:2000fk}), and has almost become a standard choice for the surrogate classifier in quantification experiments (see, e.g., \cite{Schumacher2021, QuaPy:2021bs}).
    \item Naïve Bayes: despite being a probabilistic classifier, is a typical example of an ill-calibrated classifier \cite{domingos1996beyond,Zadrozny:2001yg}. This classifier is interesting because it is widely used, partly due to its good efficiency.
    \item $k$ Nearest Neighbor: is a representative example of instance-based learning. We set uniform weights, so that this method becomes representative of a probabilistic classifier with a fixed, limited number of possible  outcomes. We set $k=10$ so that no more than $k+1=11$ different posterior probabilities (i.e., $\{0,\frac{1}{10},\ldots,\frac{9}{10},1\}$, the possible fractions of positive instances in the neighbourhood) can be returned. 
    \item Multi-Layer Perceptron: is representative of neural approaches, which are known to be very effective in terms of classification accuracy, but that tend to provide overconfident judgments \cite{Guo:2017hh}.
\end{itemize}

In all cases, we rely on the implementations of \texttt{scikit-learn} \cite{Pedregosa:2011yo} with any other hyperparameter left to its default value.










\subsection{Results and Discussion}

In this section, we discuss the results we have obtained for calibration (Section~\ref{sec:exp:calibration}), quantification (Section~\ref{sec:exp:quantification}), and classifier accuracy prediction (Section~\ref{sec:exp:cap}) experiments. 

Each experiment is subdivided into two experiments, one addressing CS, and another addressing LS. We display the numerical results in tabular form, using the following notational conventions: We use a color-coding to facilitate the interpretation of results, with intense green indicating the best result for each row, and intense red indicating the worst one; the rest of the values are linearly interpolated between these two. We highlight in bold the best result for each dataset as well as all the results which are not found to be statistically significantly different from it according to a Wilcoxon signed-rank test at 95\% confidence level. 

We also report, in the bottom of the tables, the \emph{win rates} (Wins) of each method with respect to the reference ones; we use $\Pr(M\succ iR)$ to indicate the percentage of all test samples (i.e., across all datasets and classifiers) in which the corresponding method has beaten at least $i$ reference methods simultaneously, for $i\in\{1,\ldots,r\}$ with $r$ the total number of reference methods. For example, $\Pr(M\succ 2R)=50\%$ indicates the method $M$ has beaten two (or more) reference methods (not necessarily the same two methods every time) in half of the cases.
When such value is accompanied with a $\dag$ symbol, it means that the win rate is statistically significant, according to a binomial test at 5\% confidence level. This test considers the null hypothesis $H_0$ that the results of method $M$ and those from the reference methods are indistinguishable, in which case the win rate should account for the fraction of times $M$ happens to rank better than $i$ reference methods in a random permutation. The null hypothesis is rejected when the total number of times ($s$) in which $M$ has beaten at least $i$ reference methods across a total of $N$ experiments surpasses the expected probability that this would happen simply due to chance (i.e., $1-\frac{i}{r+1})$ with $r$ the number of reference methods), at 95\% confidence level. For example, when $r=4$ reference methods, we would expect that any method $M$ beats exactly 3 reference methods when it ranks either first or second, and this happens in a random permutation $2/5$ of the times. More formally, this test checks whether $1-\text{BinomialCDF}(s-1,N,1-\frac{i}{r+1})<0.05$. 

We also highlight in bold the $r$ best average ranks in the entire table. The rationale behind this is that, should our reference methods outperform all newly proposed candidates (as one might expect), then all bolded values would be attributed to the reference methods, and none to the new candidates. Cases in which this does not happen are therefore worth analyzing.

We also analyse, in graphical form, the error as a function of the \emph{intensity of shift}. In CS experiments, in which the test samples are generated as a mixture of source and target datasets (see Section~\ref{sec:exp:covshift:protocol}), we define this intensity as the fraction of examples drawn from the target dataset, so that the intensity ranges from 0 (meaning IID conditions) to 1 (all the examples in the test set come from a different domain).
In LS experiments, we measure the intensity of shift as the absolute difference between the positive prevalence in the test sample and in the training set. The intensity thus ranges from 0 (IID conditions) to values close to 1 (cases in which the proportion of positives in the training set is close to 0 but in the test is close to 1, or viceversa); note that this condition is only attainable in extreme situations, e.g., for a perfectly balance training set, the maximum achievable shift intensity is 0.5. This has the effect of generating less experiments in the high-level shift regime. 
This is in contrast to the level of shift generated for CS experiments, which is uniform in the interval [0,1].
In order to render this fact evident, we display, as a background bar-chart, the density of the number of experiments that concur at each level of shift. 

Finally, for each experiment, we submit all the methods to a multiple comparison test. For this purpose, we rely on Critical Difference diagrams (CD-diagrams)\footnote{We use the software \url{https://mirkobunse.github.io/CriticalDifferenceDiagrams.jl/stable/}} \cite{demsar2006statistical}. Following \cite{benavoli2016should}, we adopt the Wilcoxon signed-rank test for the post-hoc assessment of pairwise differences. We also apply a Holm correction and set the significance level to 0.05.
Since the protocols generate test sets that are consistently aligned across all experiments, and since each one is affected by a different intensity of (prior or covariate) shift, we decided to submit all scores to the statistical test, instead of averaging them by dataset, in order to preserve the variability introduced by the sampling generation protocols and allow for a more robust comparison in this respect.


\subsubsection{Calibration Experiments}
\label{sec:exp:calibration}

The evaluation metric we use for assessing the calibration performance is the L2 Expected Calibration Error (ECE) \cite{popordanoska2024lascal}, given by:

\begin{align}
    \text{ECE}&=\sum_{i=1}^b\frac{|B_i|}{n} \left( \text{frac}_{\text{pos}}(B_i)-\text{conf}(B_i) \right)^2 \\
    \text{frac}_{\text{pos}}(B_i)&=\frac{1}{|B_i|}\sum_{j\in B_i} \mathds{1}[\;y_j=1\;]\\
    \text{conf}(B_i)&=\frac{1}{|B_i|}\sum_{j\in B_i}\hcal(x_j)
\end{align}

\noindent where $b$ is the number of bins, $B_i$ is a set containing the indexes of the instances assigned to the $i$th bin, and $\mathds{1}$ is the indicator function. In our experiments, we set $b=15$ following \cite{popordanoska2024lascal}, but adopt isometric binning (instead of adaptative binning), in order to guarantee that all methods are evaluated with the exact same bin divisions.\footnote{We have taken care to ensure that none of the binning-based methods use $b=15$, in order to prevent any unintended information leak that could illegitimately favour any method in terms of ECE.} We report 100 $\times$ ECE in our experiments. 

We confront our reference calibration methods (Head2Tail, CPCS, TransCal, and LasCal) against
our reference quantification systems (PACC, EMQ, KDEy) properly converted into calibration methods through the adaptation $\quant2\calib$ using 5 bins. In CS experiments, we also apply $\quant2\calib$ to the PCC quantification method, which is known to fare better under this type of shift \cite{Tasche:2022hh,Gonzalez:2024cs}. Analogously, we apply the $\acc2\calib$ method to generate direct adaptations of our reference classifier accuracy predictors (ATC, DoC, LEAP) setting the number of bins to 6.\footnote{We use 5 bins (an odd number) for $\quant2\calib$ and 6 bins (an even number) for $\acc2\calib$ to illustrate the fact that the former works with any number of bins, while the latter requires an even number. We did not perform model selection to optimize this hyperparameter.} 
In CS experiments, we propose $\text{LEAP}^{\text{PCC}-\text{6B}}_{\acc2\calib}$, a variant of LEAP that replaces the original prior-shift-oriented quantifier (KDEy) with PCC, which is expected to be more reliable under CS.
We also consider EMQ \cite{Saerens:2002uq} and the variant EMQ$^{\text{BCTS}}$ \cite{Alexandari:2020dn} that previously applies an additional calibration round based on Bias-Corrected Temperature Scaling (BCTS). Based on the same principle, we explore two new configurations: EMQ$^{\text{TransCal}}$ in CS experiments and EMQ$^{\text{LasCal}}$ in LS experiments, i.e., two new methods that, in place of BCTS, apply a calibrator specialized for CS (TransCal) or LS (LasCal), respectively. Finally, we also test PacCal$^\sigma$ and DMCal, the two new calibration methods proposed in Section~\ref{sec:methods:new}.

The results we have obtained for CS are reported in Table~\ref{tab:calib:covshift}. Somehow surprisingly, the reference methods are not excelling in this evaluation, nor even CPCS or TransCal, which were proposed for addressing CS. Among the four reference methods, the best one in terms of average ranking is TransCal, although Head2Tail obtains a higher number of best results (or results that are not statistically significantly different from the best one), which happens in 7 out of 18 cases. None of the reference methods surpass the result of the IID baseline Platt scaling, though. Additionally, a quick look at the win rates reveal that several methods (including Platt scaling) beat the four reference methods simultaneously and with statistical significance. In what follows we conjecture the possible reasons behind this unexpected outcome.

\begin{table}[tb!]
    \caption{Calibration performance under CS in terms of ECE}
    \vspace{0.5cm}
    \centering
    \resizebox{\textwidth}{!}{%
    \begin{tabular}{cc|c|cccc|ccccccccccccc} \toprule
\multicolumn{2}{c}{} & \multicolumn{1}{c|}{Baselines} & \multicolumn{4}{c|}{Reference} & \multicolumn{13}{c}{Adapted Methods} \\
\multicolumn{2}{c}{} & \begin{sideways}Platt\;\end{sideways} & \begin{sideways}Head2Tail\;\end{sideways} & \begin{sideways}CPCS\;\end{sideways} & \begin{sideways}TransCal\;\end{sideways} & \begin{sideways}LasCal\;\end{sideways} & \begin{sideways}$\text{PCC}^{5\text{B}}_{\rho 2\zeta}$\;\end{sideways} & \begin{sideways}PACC$_{\rho 2\zeta}^{5\text{B}}$\;\end{sideways} & \begin{sideways}EMQ$_{\rho 2\zeta}^{5\text{B}}$\;\end{sideways} & \begin{sideways}KDEy$_{\rho 2\zeta}^{5\text{B}}$\;\end{sideways} & \begin{sideways}ATC$_{\alpha 2\zeta}^{6\text{B}}$\;\end{sideways} & \begin{sideways}DoC$_{\alpha 2\zeta}^{6\text{B}}$\;\end{sideways} & \begin{sideways}LEAP$_{\alpha 2\zeta}^{6\text{B}}$\;\end{sideways} & \begin{sideways}LEAP$_{\alpha 2\zeta}^{\text{PCC}-6\text{B}}$\;\end{sideways} & \begin{sideways}EMQ\;\end{sideways} & \begin{sideways}EMQ$^{\text{BCTS}}$\;\end{sideways} & \begin{sideways}EMQ$^{\text{TransCal}}$\;\end{sideways} & \begin{sideways}PacCal$^{\sigma}$\;\end{sideways} & \begin{sideways}DMCal\;\end{sideways} \\\midrule
\multirow{6}{*}{\begin{sideways}BERT\;\end{sideways}} & \texttt{imdb}$\rightarrow$\texttt{rt} & $0.826$\cellcolor{green!38} & $2.450$\cellcolor{green!29} & $2.966$\cellcolor{green!26} & $1.623$\cellcolor{green!34} & $0.893$\cellcolor{green!38} & $0.873$\cellcolor{green!38} & $1.013$\cellcolor{green!37} & $1.682$\cellcolor{green!33} & $2.066$\cellcolor{green!31} & $0.842$\cellcolor{green!38} & $\textbf{0.584}$\cellcolor{green!40} & $1.006$\cellcolor{green!37} & $\textbf{0.626}$\cellcolor{green!39} & $3.863$\cellcolor{green!21} & $2.507$\cellcolor{green!29} & $14.849$\cellcolor{red!40} & $0.924$\cellcolor{green!38} & $2.251$\cellcolor{green!30} \\
 & \texttt{imdb}$\rightarrow$\texttt{yelp} & $\textbf{0.587}$\cellcolor{green!40} & $3.026$\cellcolor{red!1} & $2.314$\cellcolor{green!10} & $2.051$\cellcolor{green!15} & $0.743$\cellcolor{green!37} & $1.047$\cellcolor{green!32} & $0.699$\cellcolor{green!38} & $1.038$\cellcolor{green!32} & $0.875$\cellcolor{green!35} & $\textbf{0.634}$\cellcolor{green!39} & $\textbf{0.611}$\cellcolor{green!39} & $1.191$\cellcolor{green!29} & $0.886$\cellcolor{green!34} & $1.867$\cellcolor{green!18} & $3.141$\cellcolor{red!3} & $5.325$\cellcolor{red!40} & $\textbf{0.638}$\cellcolor{green!39} & $\textbf{0.630}$\cellcolor{green!39} \\
 & \texttt{rt}$\rightarrow$\texttt{imdb} & $\textbf{0.586}$\cellcolor{green!40} & $3.936$\cellcolor{green!29} & $3.345$\cellcolor{green!31} & $2.206$\cellcolor{green!34} & $0.724$\cellcolor{green!39} & $1.316$\cellcolor{green!37} & $\textbf{0.590}$\cellcolor{green!39} & $1.561$\cellcolor{green!36} & $1.122$\cellcolor{green!38} & $\textbf{0.622}$\cellcolor{green!39} & $0.755$\cellcolor{green!39} & $0.744$\cellcolor{green!39} & $\textbf{0.647}$\cellcolor{green!39} & $25.951$\cellcolor{red!40} & $16.198$\cellcolor{red!9} & $25.951$\cellcolor{red!40} & $1.980$\cellcolor{green!35} & $5.241$\cellcolor{green!25} \\
 & \texttt{rt}$\rightarrow$\texttt{yelp} & $\textbf{0.636}$\cellcolor{green!39} & $3.606$\cellcolor{green!30} & $3.336$\cellcolor{green!31} & $2.214$\cellcolor{green!34} & $0.764$\cellcolor{green!39} & $1.224$\cellcolor{green!37} & $0.693$\cellcolor{green!39} & $0.981$\cellcolor{green!38} & $1.742$\cellcolor{green!36} & $\textbf{0.574}$\cellcolor{green!40} & $0.813$\cellcolor{green!39} & $0.747$\cellcolor{green!39} & $0.950$\cellcolor{green!38} & $25.156$\cellcolor{red!40} & $2.948$\cellcolor{green!32} & $25.156$\cellcolor{red!40} & $2.080$\cellcolor{green!35} & $1.716$\cellcolor{green!36} \\
 & \texttt{yelp}$\rightarrow$\texttt{imdb} & $1.467$\cellcolor{green!30} & $\textbf{0.782}$\cellcolor{green!39} & $2.632$\cellcolor{green!15} & $0.876$\cellcolor{green!38} & $2.642$\cellcolor{green!15} & $\textbf{0.767}$\cellcolor{green!40} & $1.179$\cellcolor{green!34} & $2.125$\cellcolor{green!21} & $2.592$\cellcolor{green!15} & $1.445$\cellcolor{green!30} & $0.964$\cellcolor{green!37} & $2.357$\cellcolor{green!18} & $1.689$\cellcolor{green!27} & $2.742$\cellcolor{green!13} & $2.581$\cellcolor{green!15} & $6.785$\cellcolor{red!40} & $1.430$\cellcolor{green!31} & $2.393$\cellcolor{green!18} \\
 & \texttt{yelp}$\rightarrow$\texttt{rt} & $1.787$\cellcolor{green!23} & $\textbf{0.861}$\cellcolor{green!40} & $2.954$\cellcolor{green!2} & $1.114$\cellcolor{green!35} & $3.992$\cellcolor{red!16} & $0.911$\cellcolor{green!39} & $1.414$\cellcolor{green!30} & $2.331$\cellcolor{green!13} & $2.852$\cellcolor{green!4} & $1.833$\cellcolor{green!22} & $1.163$\cellcolor{green!34} & $2.788$\cellcolor{green!5} & $1.839$\cellcolor{green!22} & $2.472$\cellcolor{green!11} & $2.554$\cellcolor{green!9} & $5.315$\cellcolor{red!40} & $1.345$\cellcolor{green!31} & $2.381$\cellcolor{green!12} \\\midrule
\multirow{6}{*}{\begin{sideways}DistilBERT\;\end{sideways}} & \texttt{imdb}$\rightarrow$\texttt{rt} & $1.076$\cellcolor{green!24} & $\textbf{0.729}$\cellcolor{green!38} & $1.560$\cellcolor{green!4} & $0.922$\cellcolor{green!30} & $1.923$\cellcolor{red!10} & $\textbf{0.700}$\cellcolor{green!40} & $1.140$\cellcolor{green!21} & $1.836$\cellcolor{red!6} & $1.854$\cellcolor{red!7} & $2.643$\cellcolor{red!40} & $0.827$\cellcolor{green!34} & $2.138$\cellcolor{red!19} & $0.936$\cellcolor{green!30} & $0.954$\cellcolor{green!29} & $0.873$\cellcolor{green!32} & $1.088$\cellcolor{green!24} & $1.052$\cellcolor{green!25} & $0.904$\cellcolor{green!31} \\
 & \texttt{imdb}$\rightarrow$\texttt{yelp} & $0.680$\cellcolor{green!28} & $0.830$\cellcolor{green!18} & $1.128$\cellcolor{red!2} & $0.901$\cellcolor{green!13} & $1.143$\cellcolor{red!3} & $0.618$\cellcolor{green!32} & $0.660$\cellcolor{green!29} & $1.000$\cellcolor{green!6} & $0.981$\cellcolor{green!7} & $1.352$\cellcolor{red!18} & $\textbf{0.546}$\cellcolor{green!37} & $1.180$\cellcolor{red!6} & $\textbf{0.586}$\cellcolor{green!35} & $\textbf{0.550}$\cellcolor{green!37} & $\textbf{0.517}$\cellcolor{green!40} & $1.467$\cellcolor{red!26} & $1.661$\cellcolor{red!40} & $\textbf{0.539}$\cellcolor{green!38} \\
 & \texttt{rt}$\rightarrow$\texttt{imdb} & $\textbf{0.657}$\cellcolor{green!39} & $2.756$\cellcolor{green!33} & $1.520$\cellcolor{green!37} & $1.073$\cellcolor{green!38} & $0.785$\cellcolor{green!39} & $\textbf{0.670}$\cellcolor{green!39} & $1.019$\cellcolor{green!38} & $1.528$\cellcolor{green!37} & $1.242$\cellcolor{green!38} & $1.344$\cellcolor{green!37} & $\textbf{0.637}$\cellcolor{green!39} & $1.463$\cellcolor{green!37} & $0.801$\cellcolor{green!39} & $22.108$\cellcolor{red!27} & $\textbf{0.624}$\cellcolor{green!40} & $25.951$\cellcolor{red!40} & $0.812$\cellcolor{green!39} & $\textbf{0.626}$\cellcolor{green!39} \\
 & \texttt{rt}$\rightarrow$\texttt{yelp} & $0.830$\cellcolor{green!39} & $2.131$\cellcolor{green!34} & $1.811$\cellcolor{green!35} & $0.905$\cellcolor{green!38} & $1.131$\cellcolor{green!38} & $\textbf{0.580}$\cellcolor{green!40} & $1.370$\cellcolor{green!37} & $1.998$\cellcolor{green!35} & $1.422$\cellcolor{green!37} & $1.455$\cellcolor{green!37} & $0.651$\cellcolor{green!39} & $1.981$\cellcolor{green!35} & $1.246$\cellcolor{green!37} & $21.334$\cellcolor{red!27} & $1.395$\cellcolor{green!37} & $25.155$\cellcolor{red!40} & $0.767$\cellcolor{green!39} & $1.184$\cellcolor{green!38} \\
 & \texttt{yelp}$\rightarrow$\texttt{imdb} & $1.366$\cellcolor{green!16} & $\textbf{0.652}$\cellcolor{green!40} & $1.344$\cellcolor{green!17} & $0.814$\cellcolor{green!34} & $2.742$\cellcolor{red!28} & $\textbf{0.691}$\cellcolor{green!38} & $0.889$\cellcolor{green!32} & $1.713$\cellcolor{green!5} & $1.461$\cellcolor{green!13} & $1.430$\cellcolor{green!14} & $0.741$\cellcolor{green!37} & $1.966$\cellcolor{red!3} & $1.455$\cellcolor{green!13} & $1.795$\cellcolor{green!2} & $1.398$\cellcolor{green!15} & $3.078$\cellcolor{red!40} & $2.058$\cellcolor{red!6} & $1.387$\cellcolor{green!15} \\
 & \texttt{yelp}$\rightarrow$\texttt{rt} & $1.922$\cellcolor{green!12} & $\textbf{0.854}$\cellcolor{green!39} & $1.657$\cellcolor{green!18} & $\textbf{0.839}$\cellcolor{green!40} & $3.952$\cellcolor{red!40} & $\textbf{0.944}$\cellcolor{green!37} & $1.387$\cellcolor{green!25} & $2.312$\cellcolor{green!2} & $2.118$\cellcolor{green!7} & $2.277$\cellcolor{green!3} & $\textbf{0.992}$\cellcolor{green!36} & $2.739$\cellcolor{red!8} & $1.799$\cellcolor{green!15} & $1.682$\cellcolor{green!18} & $1.517$\cellcolor{green!22} & $1.790$\cellcolor{green!15} & $1.720$\cellcolor{green!17} & $1.512$\cellcolor{green!22} \\\midrule
\multirow{6}{*}{\begin{sideways}RoBERTa\;\end{sideways}} & \texttt{imdb}$\rightarrow$\texttt{rt} & $1.462$\cellcolor{green!30} & $\textbf{0.753}$\cellcolor{green!40} & $2.440$\cellcolor{green!17} & $0.996$\cellcolor{green!36} & $2.781$\cellcolor{green!13} & $0.864$\cellcolor{green!38} & $1.593$\cellcolor{green!29} & $2.724$\cellcolor{green!14} & $2.721$\cellcolor{green!14} & $4.027$\cellcolor{red!2} & $1.034$\cellcolor{green!36} & $2.573$\cellcolor{green!16} & $\textbf{0.836}$\cellcolor{green!38} & $3.169$\cellcolor{green!8} & $3.087$\cellcolor{green!9} & $6.875$\cellcolor{red!40} & $0.940$\cellcolor{green!37} & $2.731$\cellcolor{green!14} \\
 & \texttt{imdb}$\rightarrow$\texttt{yelp} & $\textbf{0.637}$\cellcolor{green!37} & $0.777$\cellcolor{green!29} & $1.676$\cellcolor{red!19} & $1.115$\cellcolor{green!10} & $0.841$\cellcolor{green!25} & $0.738$\cellcolor{green!31} & $\textbf{0.637}$\cellcolor{green!37} & $0.926$\cellcolor{green!21} & $0.909$\cellcolor{green!22} & $1.089$\cellcolor{green!12} & $\textbf{0.622}$\cellcolor{green!37} & $1.031$\cellcolor{green!15} & $\textbf{0.583}$\cellcolor{green!40} & $\textbf{0.600}$\cellcolor{green!39} & $\textbf{0.596}$\cellcolor{green!39} & $2.045$\cellcolor{red!40} & $1.509$\cellcolor{red!10} & $\textbf{0.589}$\cellcolor{green!39} \\
 & \texttt{rt}$\rightarrow$\texttt{imdb} & $\textbf{0.754}$\cellcolor{green!39} & $\textbf{0.743}$\cellcolor{green!39} & $1.429$\cellcolor{green!37} & $1.118$\cellcolor{green!38} & $1.613$\cellcolor{green!37} & $\textbf{0.686}$\cellcolor{green!40} & $1.469$\cellcolor{green!37} & $2.449$\cellcolor{green!34} & $1.315$\cellcolor{green!38} & $1.135$\cellcolor{green!38} & $\textbf{0.743}$\cellcolor{green!39} & $2.398$\cellcolor{green!34} & $1.937$\cellcolor{green!36} & $25.696$\cellcolor{red!39} & $4.139$\cellcolor{green!29} & $25.951$\cellcolor{red!40} & $1.318$\cellcolor{green!38} & $2.504$\cellcolor{green!34} \\
 & \texttt{rt}$\rightarrow$\texttt{yelp} & $\textbf{0.708}$\cellcolor{green!39} & $0.822$\cellcolor{green!39} & $1.119$\cellcolor{green!38} & $1.453$\cellcolor{green!37} & $1.123$\cellcolor{green!38} & $0.756$\cellcolor{green!39} & $1.216$\cellcolor{green!38} & $1.777$\cellcolor{green!36} & $1.214$\cellcolor{green!38} & $1.068$\cellcolor{green!38} & $\textbf{0.657}$\cellcolor{green!40} & $2.129$\cellcolor{green!35} & $1.608$\cellcolor{green!36} & $24.870$\cellcolor{red!39} & $3.924$\cellcolor{green!29} & $25.156$\cellcolor{red!40} & $1.125$\cellcolor{green!38} & $2.425$\cellcolor{green!34} \\
 & \texttt{yelp}$\rightarrow$\texttt{imdb} & $1.047$\cellcolor{green!26} & $0.610$\cellcolor{green!38} & $0.731$\cellcolor{green!35} & $\textbf{0.551}$\cellcolor{green!40} & $1.547$\cellcolor{green!13} & $0.615$\cellcolor{green!38} & $\textbf{0.578}$\cellcolor{green!39} & $0.899$\cellcolor{green!30} & $1.067$\cellcolor{green!26} & $0.889$\cellcolor{green!31} & $0.668$\cellcolor{green!36} & $1.210$\cellcolor{green!22} & $0.792$\cellcolor{green!33} & $\textbf{0.598}$\cellcolor{green!38} & $0.881$\cellcolor{green!31} & $1.163$\cellcolor{green!23} & $3.585$\cellcolor{red!40} & $0.852$\cellcolor{green!32} \\
 & \texttt{yelp}$\rightarrow$\texttt{rt} & $1.386$\cellcolor{green!18} & $0.814$\cellcolor{green!37} & $1.168$\cellcolor{green!25} & $\textbf{0.764}$\cellcolor{green!38} & $3.190$\cellcolor{red!40} & $\textbf{0.723}$\cellcolor{green!40} & $0.835$\cellcolor{green!36} & $1.537$\cellcolor{green!13} & $1.348$\cellcolor{green!19} & $2.074$\cellcolor{red!3} & $0.940$\cellcolor{green!32} & $1.812$\cellcolor{green!4} & $0.860$\cellcolor{green!35} & $\textbf{0.723}$\cellcolor{green!39} & $0.890$\cellcolor{green!34} & $1.193$\cellcolor{green!24} & $3.076$\cellcolor{red!36} & $0.979$\cellcolor{green!31} \\\bottomrule
\multirow{4}{*}{\begin{sideways}Wins\;\end{sideways}} &  $\Pr(M\succ {1}R)$ &$\dag 94.11\%$ & --- & --- & --- & --- & $\dag 97.17\%$ & $\dag 92.50\%$ & $81.50\%$ & $\dag 83.56\%$ & $\dag 82.94\%$ & $\dag 98.61\%$ & $77.39\%$ & $\dag 90.56\%$ & $58.28\%$ & $79.22\%$ & $38.67\%$ & $76.22\%$ & $\dag 86.50\%$ \\
 &  $\Pr(M\succ {2}R)$ &$\dag 71.44\%$ & --- & --- & --- & --- & $\dag 86.94\%$ & $\dag 71.39\%$ & $46.56\%$ & $48.67\%$ & $53.11\%$ & $\dag 89.67\%$ & $40.94\%$ & $\dag 69.06\%$ & $43.11\%$ & $60.06\%$ & $22.94\%$ & $51.94\%$ & $\dag 64.78\%$ \\
 &  $\Pr(M\succ {3}R)$ &$\dag 45.78\%$ & --- & --- & --- & --- & $\dag 60.67\%$ & $\dag 43.89\%$ & $23.94\%$ & $25.28\%$ & $31.11\%$ & $\dag 62.44\%$ & $22.89\%$ & $\dag 44.94\%$ & $28.33\%$ & $38.56\%$ & $12.33\%$ & $30.72\%$ & $40.61\%$ \\
 &  $\Pr(M\succ {4}R)$ &$\dag 25.28\%$ & --- & --- & --- & --- & $\dag 27.72\%$ & $\dag 22.67\%$ & $9.83\%$ & $9.33\%$ & $16.89\%$ & $\dag 30.44\%$ & $8.78\%$ & $\dag 25.72\%$ & $16.50\%$ & $21.22\%$ & $6.61\%$ & $14.22\%$ & $22.56\%$ \\\bottomrule
  & Ave Rank &\textbf{7.45} & 8.76 & 10.13 & 7.99 & 10.27 & \textbf{6.28} & \textbf{7.39} & 11.49 & 11.31 & 10.20 & \textbf{5.41} & 12.27 & 7.73 & 11.40 & 9.25 & 14.54 & 10.59 & 8.55 \\\bottomrule
\end{tabular}
    }%
    \label{tab:calib:covshift}
\end{table}

Perhaps the simplest plausible explanation is that we have failed to simulate CS. However, a closer look at these results (and others provided below) disproves this observation. In this regard, we might argue that experiments in which the IID baseline (Platt) fares worse are those affected by a higher degree of CS; indeed this happens prevalently on ``\texttt{yelp}$\rightarrow$\texttt{imdb}'' and ``\texttt{yelp}$\rightarrow$\texttt{rt}'', for all three classifiers. This seems sensible, as the reviews of the ``\texttt{yelp}'' domain (about businesses) are topically more different than those from ``\texttt{imdb}'' or ``\texttt{rt}'' (which are both about movie reviews). In this respect, it is interesting to see how Head2Tail, CPCS, and TransCal consistently beat the IID baseline, while LasCal (more suitable for LS) instead tends to fare poorly in said cases. That this phenomenon is not also observable when ``\texttt{yelp}'' acts as the target domain (i.e., in ``\texttt{imdb}$\rightarrow$\texttt{yelp}'' and ``\texttt{rt}$\rightarrow$\texttt{yelp}'') might be explained by the fact that the ``\texttt{yelp}'' dataset seems to be the ``easiest'' problem, in light of the accuracy values obtained by each classifier (reported in \ref{app:accuracy}). This characteristic might partially compensate for the effect caused by the shift effect between the distributions.

Concerning the rest of the candidates, all methods seem to perform reasonably well if we rely on the color-coding. However, this impression is misleading, since one method (EMQ$^{\text{TransCal}}$) has obtained extremely poor results (red cells), which skews the color scale and make all other results appear good (i.e., intensely green). This evinces something which is already known in the literature, i.e., that EMQ is extremely sensitive to the quality of the posterior probabilities given as input \cite{Esuli:2021le}. In this case, the application of TransCal as a pre-calibration step has lead to huge instabilities in the method. The variant EMQ$^{\text{BCTS}}$ instead obtains slightly better results than the vanilla version EMQ; still, all the EMQ variants fall short in terms of performance under CS when compare to other methods. A likely reason is that EMQ is specialised for LS, and its underlying assumptions clash with those of CS. Similar principles likely apply for EMQ$^{\text{5B}}_{\quant2\calib}$ as well.

However, this batch of experiments also offers some unexpected outcomes. Three of the adapted methods, including two adaptations of quantification methods (PCC$^{\text{5B}}_{\quant2\calib}$, PACC$^{\text{5B}}_{\quant2\calib}$) and one classifier accuracy predictor (DoC$^{\text{6B}}_{\acc2\calib}$) have attained the lowest errors overall, as witnessed by the fact that their average ranks appear in bold. It is particularly interesting the case for PCC$^{\text{5B}}_{\quant2\calib}$, which is an adaptation of the preferred quantifier for CS, and especially for DoC$^{\text{6B}}_{\acc2\calib}$, which has obtained a much lower error score. 

The newly proposed methods do not stand out in terms of performance. PacCal$^\sigma$ shows erratic behavior, with some out-of-scale errors (red cells), while DMCal instead seems to work reasonably well across all datasets, consistently beating at least two reference methods simultaneously in a statistically significant sense.

The results we have obtained for LS are reported in Table~\ref{tab:calib:labelshift}. In this case, all the reference methods work markedly better than the IID baseline Platt, with LasCal (the only calibration method specifically designed to address this type of shift) standing out by a large margin. None of the direct adaptations seem to be comparable in performance with the reference  methods from the calibration literature.

\begin{table}[tb!]
    \caption{Calibration performance under label shift in terms of ECE}
    \vspace{0.5cm}
    \centering
    \resizebox{\textwidth}{!}{%
\begin{tabular}{cc|c|cccc|ccccccccccc} \toprule
\multicolumn{2}{c}{} & \multicolumn{1}{c|}{Baselines} & \multicolumn{4}{c|}{Reference} & \multicolumn{11}{c}{Adapted Methods} \\
\multicolumn{2}{c}{} & \begin{sideways}Platt\;\end{sideways} & \begin{sideways}Head2Tail\;\end{sideways} & \begin{sideways}CPCS\;\end{sideways} & \begin{sideways}TransCal\;\end{sideways} & \begin{sideways}LasCal\;\end{sideways} & \begin{sideways}PACC$_{\rho 2\zeta}^{5\text{B}}$\;\end{sideways} & \begin{sideways}EMQ$_{\rho 2\zeta}^{5\text{B}}$\;\end{sideways} & \begin{sideways}KDEy$_{\rho 2\zeta}^{5\text{B}}$\;\end{sideways} & \begin{sideways}ATC$_{\alpha 2\zeta}^{6\text{B}}$\;\end{sideways} & \begin{sideways}DoC$_{\alpha 2\zeta}^{6\text{B}}$\;\end{sideways} & \begin{sideways}LEAP$_{\alpha 2\zeta}^{6\text{B}}$\;\end{sideways} & \begin{sideways}EMQ\;\end{sideways} & \begin{sideways}EMQ$^{\text{BCTS}}$\;\end{sideways} & \begin{sideways}EMQ$^{\text{LasCal}}$\;\end{sideways} & \begin{sideways}PacCal$^{\sigma}$\;\end{sideways} & \begin{sideways}DMCal\;\end{sideways} \\\midrule
\multirow{10}{*}{\begin{sideways}Logistic Regression\;\end{sideways}} & \texttt{cmc.3} & $7.170$\cellcolor{red!10} & $5.399$\cellcolor{green!6} & $5.390$\cellcolor{green!6} & $5.604$\cellcolor{green!4} & $4.946$\cellcolor{green!10} & $2.692$\cellcolor{green!32} & $3.242$\cellcolor{green!27} & $3.005$\cellcolor{green!29} & $9.067$\cellcolor{red!29} & $3.581$\cellcolor{green!24} & $2.825$\cellcolor{green!31} & $\textbf{2.138}$\cellcolor{green!38} & $\textbf{1.939}$\cellcolor{green!40} & $10.154$\cellcolor{red!40} & $8.835$\cellcolor{red!27} & $2.402$\cellcolor{green!35} \\
\bottomrule
\end{tabular}
    }%
    \label{tab:calib:labelshift}
\end{table}

As for the rest of the methods, it now appears evident that the EMQ variants dominate. Interestingly enough, EMQ performs better when the classifier being calibrated is Logistic Regression; this is sensible because this classifier is known to be already reasonably well calibrated (for the training distribution), and this is crucial for the method stability~\cite{Esuli:2021le}. In contrast, for Multilayer Perceptron (which, as a neural approach, is deemed to be overconfident), the pre-calibration round of BCTS seems to improve the results of EMQ.
Somehow unexpectedly, though, the variant EMQ$^{\text{LasCal}}$ did not improve over vanilla EMQ nor EMQ$^{\text{BCTS}}$, despite relying on a pre-calibration round which is more adapted to this type of shift. Apparently, this method sometimes renders EMQ very unstable (see the many red cells distributed across the column). A likely reason may be that the EMQ requires the original posterior probabilities to be well calibrated for the training distribution, and not for the test distribution.
In any case, LasCal alone works reasonably well across all datasets and classifiers, and especially so for Naive Bayes, which is arguably the worst calibrated classifier, and hence the most difficult case.

Another method that stands out in terms of performance, is the newly proposed DMCal, and particularly so for $k$ Nearest Neighbor (perhaps because the already discretized outputs of $k$-Nearest Neighbors simplify the binning of DMCal). This, together with the fact that (unlike the EMQ variants) DMCal also performed reasonably well under CS, renders this method an interesting versatile approach worth considering in future calibration research related to dataset shift.


Figure~\ref{fig:calib:errbyshift} show the ECE as a function of level of shift under CS and LS experiments side-by-side for a selection of methods. These figures clearly reveal a natural tendency to degrade the performance as the intensity of shift increases (which indirectly validates the fact that we were indeed generating shift). However, different methods exhibit different degrees of robustness; for example, TransCal and Head2Tail seem to be pretty stable against the intensity of CS, while CPCS and LasCal seem instead to be very unstable in this respect. 
A comparison between both plots disproves our hypothesis  according to which (at least on binary problems) Head2Tail should be more robust to LS than to CS.
These plots also reveal that PCC$^{\text{5B}}_{\quant2\calib}$ seems preferable to DoC$^{\text{6B}}_{\acc2\calib}$ when the level of CS is extremely high. Similarly interesting is the fact that PACC$^{\text{5B}}_{\quant2\calib}$ (a non-parametric calibrator) 
behaves very similarly to Platt scaling (a parametric one), somehow suggesting that the calibration maps both methods find out tend to be very similar.
In terms of LS (right panel), it is interesting to see how LasCal, EMQ$^{\text{BCTS}}$ and (to a lesser extent) DMCal behave very robustly against the intensity of the shift.

The CD-diagrams (Figure~\ref{fig:calib:cddiagram}) confirm that DoC$^{\text{6B}}_{\acc2\calib}$ and PCC$^{\text{5B}}_{\quant2\calib}$ fare significantly better than the competitors under CS, while EMQ$^{\text{BCTS}}$, EMQ, and the newly proposed DMCal seem to dominate the LS arena, followed by LasCal.

\begin{figure}
    \centering
    \includegraphics[width=0.48\linewidth]{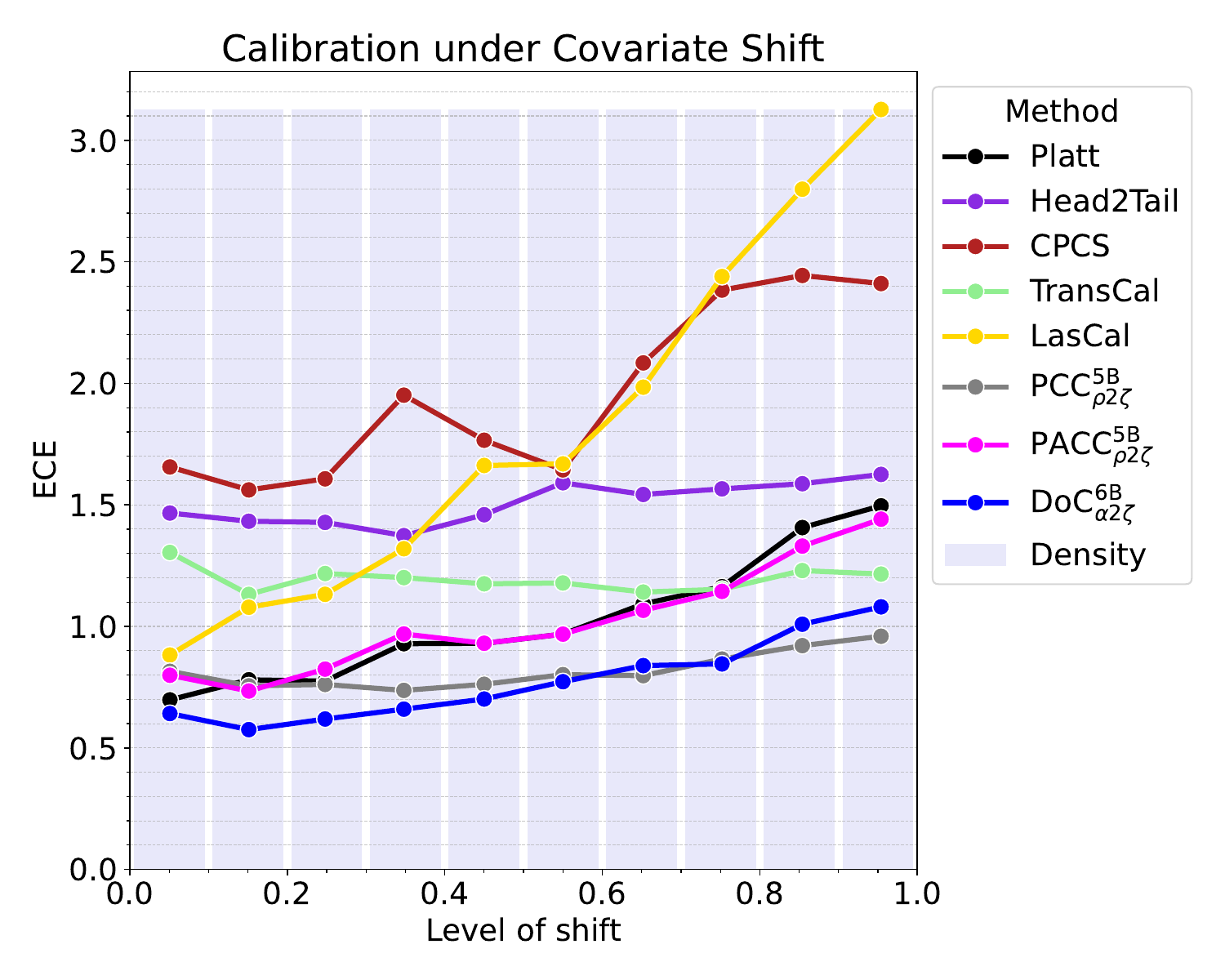}
    \includegraphics[width=0.48\linewidth]{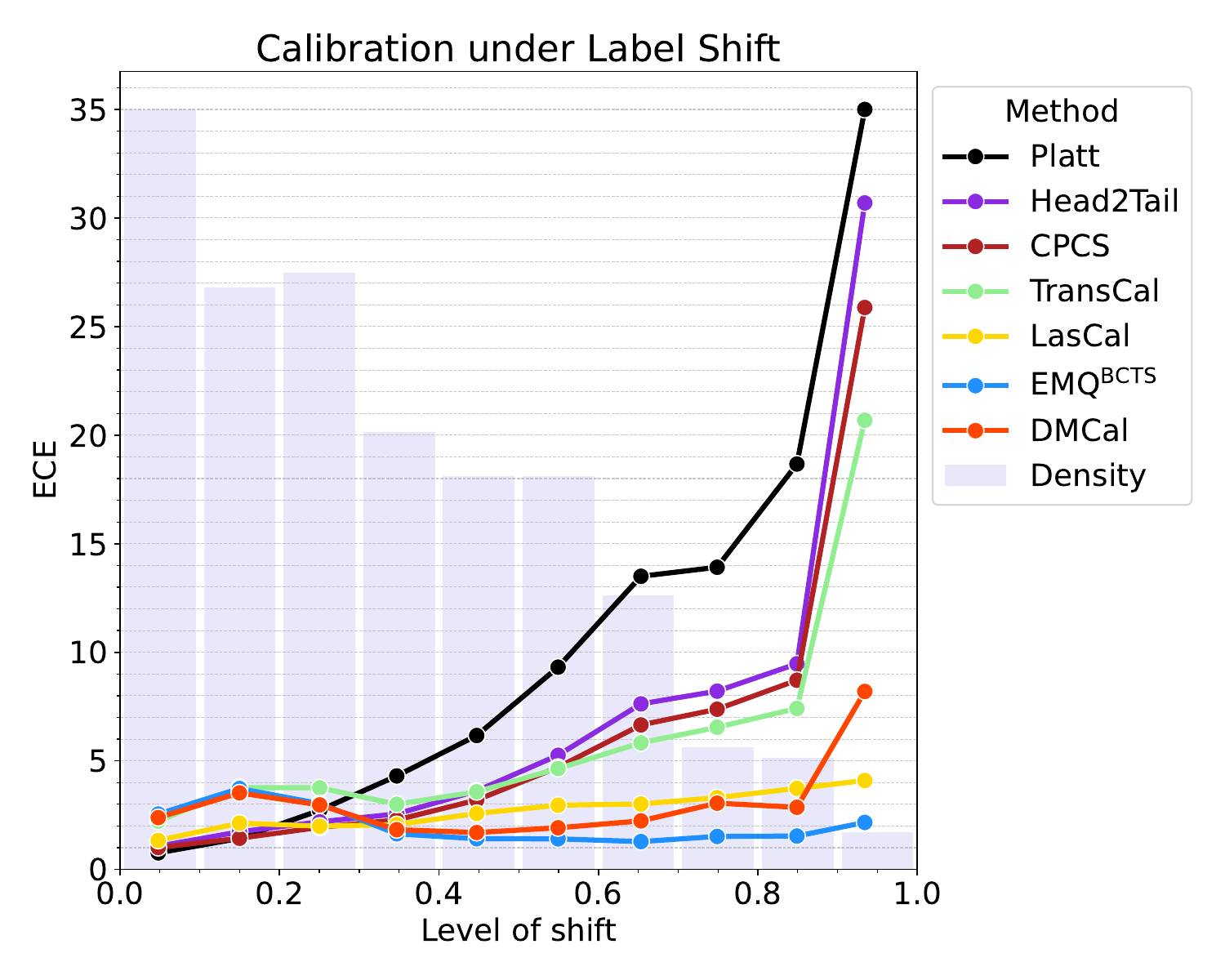}
    \caption{Calibration error in terms of ECE as a function of shift intensity in CS experiments (left panel) and LS (right panel).}
    \label{fig:calib:errbyshift}
    
    \vspace{1cm}

    \centering
    \resizebox{0.48\textwidth}{!}{%
    \begin{tikzpicture}[
  treatment line/.style={rounded corners=1.5pt, line cap=round, shorten >=1pt},
  treatment label/.style={font=\small},
  group line/.style={ultra thick},
]

\begin{axis}[
  clip={false},
  axis x line={center},
  axis y line={none},
  axis line style={-},
  xmin={1},
  ymax={0},
  scale only axis={true},
  width={\axisdefaultwidth},
  ticklabel style={anchor=south, yshift=1.3*\pgfkeysvalueof{/pgfplots/major tick length}, font=\small},
  every tick/.style={draw=black},
  major tick style={yshift=.5*\pgfkeysvalueof{/pgfplots/major tick length}},
  minor tick style={yshift=.5*\pgfkeysvalueof{/pgfplots/minor tick length}},
  title style={yshift=\baselineskip},
  xmax={18},
  ymin={-10.5},
  height={11\baselineskip},
  title={Calibration under Covariate Shift},
]

\draw[treatment line] ([yshift=-2pt] axis cs:5.4061111111111115, 0) |- (axis cs:3.9061111111111115, -2.0)
  node[treatment label, anchor=east] {DoC$_{\alpha 2\zeta}^{6\text{B}}$};
\draw[treatment line] ([yshift=-2pt] axis cs:6.276111111111111, 0) |- (axis cs:3.9061111111111115, -3.0)
  node[treatment label, anchor=east] {$\text{PCC}^{5\text{B}}_{\rho 2\zeta}$};
\draw[treatment line] ([yshift=-2pt] axis cs:7.394166666666667, 0) |- (axis cs:3.9061111111111115, -4.0)
  node[treatment label, anchor=east] {PACC$_{\rho 2\zeta}^{5\text{B}}$};
\draw[treatment line] ([yshift=-2pt] axis cs:7.454722222222222, 0) |- (axis cs:3.9061111111111115, -5.0)
  node[treatment label, anchor=east] {Platt};
\draw[treatment line] ([yshift=-2pt] axis cs:7.726388888888889, 0) |- (axis cs:3.9061111111111115, -6.0)
  node[treatment label, anchor=east] {LEAP$_{\alpha 2\zeta}^{\text{PCC}-6\text{B}}$};
\draw[treatment line] ([yshift=-2pt] axis cs:7.9880555555555555, 0) |- (axis cs:3.9061111111111115, -7.0)
  node[treatment label, anchor=east] {\textbf{TransCal}};
\draw[treatment line] ([yshift=-2pt] axis cs:8.553888888888888, 0) |- (axis cs:3.9061111111111115, -8.0)
  node[treatment label, anchor=east] {DMCal};
\draw[treatment line] ([yshift=-2pt] axis cs:8.761388888888888, 0) |- (axis cs:3.9061111111111115, -9.0)
  node[treatment label, anchor=east] {\textbf{Head2Tail}};
\draw[treatment line] ([yshift=-2pt] axis cs:9.251111111111111, 0) |- (axis cs:3.9061111111111115, -10.0)
  node[treatment label, anchor=east] {EMQ$^{\text{BCTS}}$};
\draw[treatment line] ([yshift=-2pt] axis cs:10.1325, 0) |- (axis cs:16.039444444444445, -10.0)
  node[treatment label, anchor=west] {\textbf{CPCS}};
\draw[treatment line] ([yshift=-2pt] axis cs:10.196111111111112, 0) |- (axis cs:16.039444444444445, -9.0)
  node[treatment label, anchor=west] {ATC$_{\alpha 2\zeta}^{6\text{B}}$};
\draw[treatment line] ([yshift=-2pt] axis cs:10.27, 0) |- (axis cs:16.039444444444445, -8.0)
  node[treatment label, anchor=west] {\textbf{LasCal}};
\draw[treatment line] ([yshift=-2pt] axis cs:10.589166666666667, 0) |- (axis cs:16.039444444444445, -7.0)
  node[treatment label, anchor=west] {PacCal$^{\sigma}$};
\draw[treatment line] ([yshift=-2pt] axis cs:11.305833333333334, 0) |- (axis cs:16.039444444444445, -6.0)
  node[treatment label, anchor=west] {KDEy$_{\rho 2\zeta}^{5\text{B}}$};
\draw[treatment line] ([yshift=-2pt] axis cs:11.4025, 0) |- (axis cs:16.039444444444445, -5.0)
  node[treatment label, anchor=west] {EMQ};
\draw[treatment line] ([yshift=-2pt] axis cs:11.486666666666666, 0) |- (axis cs:16.039444444444445, -4.0)
  node[treatment label, anchor=west] {EMQ$_{\rho 2\zeta}^{5\text{B}}$};
\draw[treatment line] ([yshift=-2pt] axis cs:12.265833333333333, 0) |- (axis cs:16.039444444444445, -3.0)
  node[treatment label, anchor=west] {LEAP$_{\alpha 2\zeta}^{6\text{B}}$};
\draw[treatment line] ([yshift=-2pt] axis cs:14.539444444444445, 0) |- (axis cs:16.039444444444445, -2.0)
  node[treatment label, anchor=west] {EMQ$^{\text{TransCal}}$};
\draw[group line] (axis cs:8.761388888888888, -4.666666666666667) -- (axis cs:10.589166666666667, -4.666666666666667);
\draw[group line] (axis cs:10.1325, -2.6666666666666665) -- (axis cs:11.486666666666666, -2.6666666666666665);
\draw[group line] (axis cs:8.553888888888888, -5.333333333333333) -- (axis cs:8.761388888888888, -5.333333333333333);
\draw[group line] (axis cs:10.27, -2.0) -- (axis cs:12.265833333333333, -2.0);
\draw[group line] (axis cs:7.394166666666667, -2.6666666666666665) -- (axis cs:7.9880555555555555, -2.6666666666666665);

\end{axis}
\end{tikzpicture}
    }
    \resizebox{0.48\textwidth}{!}{%
    \begin{tikzpicture}[
  treatment line/.style={rounded corners=1.5pt, line cap=round, shorten >=1pt},
  treatment label/.style={font=\small},
  group line/.style={ultra thick},
]

\begin{axis}[
  clip={false},
  axis x line={center},
  axis y line={none},
  axis line style={-},
  xmin={1},
  ymax={0},
  scale only axis={true},
  width={\axisdefaultwidth},
  ticklabel style={anchor=south, yshift=1.3*\pgfkeysvalueof{/pgfplots/major tick length}, font=\small},
  every tick/.style={draw=black},
  major tick style={yshift=.5*\pgfkeysvalueof{/pgfplots/major tick length}},
  minor tick style={yshift=.5*\pgfkeysvalueof{/pgfplots/minor tick length}},
  title style={yshift=\baselineskip},
  xmax={16},
  ymin={-9.5},
  height={10\baselineskip},
  xtick={1,4,7,10,13,16},
  minor x tick num={1},
  title={Calibration under Label Shift},
]

\draw[treatment line] ([yshift=-2pt] axis cs:5.479375, 0) |- (axis cs:4.146041666666667, -2.0)
  node[treatment label, anchor=east] {EMQ$^{\text{BCTS}}$};
\draw[treatment line] ([yshift=-2pt] axis cs:5.708375, 0) |- (axis cs:4.146041666666667, -3.0)
  node[treatment label, anchor=east] {EMQ};
\draw[treatment line] ([yshift=-2pt] axis cs:5.772375, 0) |- (axis cs:4.146041666666667, -4.0)
  node[treatment label, anchor=east] {DMCal};
\draw[treatment line] ([yshift=-2pt] axis cs:7.113, 0) |- (axis cs:4.146041666666667, -5.0)
  node[treatment label, anchor=east] {\textbf{LasCal}};
\draw[treatment line] ([yshift=-2pt] axis cs:7.312625, 0) |- (axis cs:4.146041666666667, -6.0)
  node[treatment label, anchor=east] {\textbf{CPCS}};
\draw[treatment line] ([yshift=-2pt] axis cs:7.784875, 0) |- (axis cs:4.146041666666667, -7.0)
  node[treatment label, anchor=east] {EMQ$^{\text{LasCal}}$};
\draw[treatment line] ([yshift=-2pt] axis cs:8.504625, 0) |- (axis cs:4.146041666666667, -8.0)
  node[treatment label, anchor=east] {\textbf{Head2Tail}};
\draw[treatment line] ([yshift=-2pt] axis cs:8.794625, 0) |- (axis cs:4.146041666666667, -9.0)
  node[treatment label, anchor=east] {PACC$_{\rho 2\zeta}^{5\text{B}}$};
\draw[treatment line] ([yshift=-2pt] axis cs:8.804, 0) |- (axis cs:13.046458333333334, -9.0)
  node[treatment label, anchor=west] {\textbf{TransCal}};
\draw[treatment line] ([yshift=-2pt] axis cs:8.982, 0) |- (axis cs:13.046458333333334, -8.0)
  node[treatment label, anchor=west] {EMQ$_{\rho 2\zeta}^{5\text{B}}$};
\draw[treatment line] ([yshift=-2pt] axis cs:9.2085, 0) |- (axis cs:13.046458333333334, -7.0)
  node[treatment label, anchor=west] {KDEy$_{\rho 2\zeta}^{5\text{B}}$};
\draw[treatment line] ([yshift=-2pt] axis cs:9.75, 0) |- (axis cs:13.046458333333334, -6.0)
  node[treatment label, anchor=west] {Platt};
\draw[treatment line] ([yshift=-2pt] axis cs:9.792375, 0) |- (axis cs:13.046458333333334, -5.0)
  node[treatment label, anchor=west] {DoC$_{\alpha 2\zeta}^{6\text{B}}$};
\draw[treatment line] ([yshift=-2pt] axis cs:9.997875, 0) |- (axis cs:13.046458333333334, -4.0)
  node[treatment label, anchor=west] {LEAP$_{\alpha 2\zeta}^{6\text{B}}$};
\draw[treatment line] ([yshift=-2pt] axis cs:11.28225, 0) |- (axis cs:13.046458333333334, -3.0)
  node[treatment label, anchor=west] {ATC$_{\alpha 2\zeta}^{6\text{B}}$};
\draw[treatment line] ([yshift=-2pt] axis cs:11.713125, 0) |- (axis cs:13.046458333333334, -2.0)
  node[treatment label, anchor=west] {PacCal$^{\sigma}$};
\draw[group line] (axis cs:5.708375, -2.0) -- (axis cs:5.772375, -2.0);
\draw[group line] (axis cs:8.504625, -5.333333333333333) -- (axis cs:8.804, -5.333333333333333);
\draw[group line] (axis cs:7.784875, -4.666666666666667) -- (axis cs:8.504625, -4.666666666666667);
\draw[group line] (axis cs:9.792375, -2.6666666666666665) -- (axis cs:9.997875, -2.6666666666666665);
\draw[group line] (axis cs:5.479375, -1.3333333333333333) -- (axis cs:5.708375, -1.3333333333333333);
\draw[group line] (axis cs:8.804, -4.666666666666667) -- (axis cs:9.2085, -4.666666666666667);

\end{axis}
\end{tikzpicture}
    }
    \caption{CD-diagrams for CS experiments (left panel) and LS experiments (right panel). Reference calibration methods are highlighted in boldtype.}
    \label{fig:calib:cddiagram}
\end{figure}

These results are interesting, since they clearly show how different techniques from quantification and classifier accuracy prediction can be effectively used in the calibration field yielding unexpected success.
Notwithstanding this, one might argue that ECE does not tell the whole story, since some of the calibration methods we propose are not guaranteed to be accuracy-preserving (i.e., in the process of calibrating the posteriors, they may well shift the  decision boundary in the calibration map). In Tables~\ref{tab:calib:covshift:brier} and \ref{tab:calib:labelshift:brier} (in the Appendix), we report  Brier Scores values for CS and LS, respectively. These results reveal that low ECE has not come at the cost of accuracy (i.e., the best performing methods in terms of ECE also tend to perform well in terms of Brier Score), with the sole exception of LasCal under LS, which drops some positions in the average ranking. 





\subsubsection{Quantification Experiments}
\label{sec:exp:quantification}

In this section, we report the results we have obtained for the task of quantification. The evaluation measure we adopt is the absolute error (AE) \cite{Sebastiani:2020qf}, given by:
\begin{equation}
    \text{AE}=|p-\hat{p}|
\end{equation}
\noindent where $p\in[0,1]$ is the true prevalence value of the positive class, and $\hat{p}\in[0,1]$ is the estimated prevalence value obtained by the quantification method.

In this case, we consider PACC, EMQ, and KDEy as our reference methods, and apply the $\acc2\quant$ and $\calib2\quant$ adaptation algorithms to generate new quantification methods from classifier accuracy predictors and calibrators, respectively. When the type of shift is CS, we also include PCC as an additional reference quantification method, and we also consider LEAP$^{\text{PCC}}_{\acc2\quant}$, an adaptation of LEAP equipped with PCC as the surrogate quantifier. For LS experiments, we also consider EMQ$^{\text{BCTS}}$ as an additional reference method. The IID baseline is CC, a method that simply classifies and counts the fraction of predicted positive instances in the test set. We do not consider adaptations from PacCal$^\sigma$ nor DMCal, since those methods heavily rely on  quantification methods and would therefore be redundant to the comparison.

Table~\ref{tab:quant:covshift} reports the results for CS simulations. As expected, PCC is the top-performing reference quantifier; the rest of the proper quantification methods (PACC, EMQ, KDEy) fall short in terms of AE given that the underlying assumptions on which these methods rely upon do not align well with the characteristics of the data (see also \cite{Gonzalez:2024cs}). Indeed, the IID baseline (CC) tends to outperform three methods (presumably PACC, EMQ, and KDEy) consistently and with statistical significance (see the win rates).

\begin{table}[b!]
    \caption{Quantification performance under CS in terms of AE}
    \vspace{0.5cm}
    \centering
    \resizebox{\textwidth}{!}{%
    \begin{tabular}{cc|c|cccc|ccccccc} \toprule
\multicolumn{2}{c}{} & \multicolumn{1}{c|}{Baselines} & \multicolumn{4}{c|}{Reference} & \multicolumn{7}{c}{Adapted Methods} \\
\multicolumn{2}{c}{} & \begin{sideways}CC\;\end{sideways} & \begin{sideways}PCC\;\end{sideways} & \begin{sideways}PACC\;\end{sideways} & \begin{sideways}EMQ\;\end{sideways} & \begin{sideways}KDEy\;\end{sideways} & \begin{sideways}ATC$_{\alpha 2 \rho}$\;\end{sideways} & \begin{sideways}DoC$_{\alpha 2 \rho}$\;\end{sideways} & \begin{sideways}LEAP$_{\alpha 2 \rho}$\;\end{sideways} & \begin{sideways}LEAP$_{\alpha 2 \rho}^{\text{PCC}}$\;\end{sideways} & \begin{sideways}CPCS$_{\zeta 2 \rho}$\;\end{sideways} & \begin{sideways}TransCal$_{\zeta 2 \rho}$\;\end{sideways} & \begin{sideways}LasCal$_{\zeta 2 \rho}$\;\end{sideways} \\\midrule
\multirow{6}{*}{\begin{sideways}BERT\;\end{sideways}} & \texttt{imdb}$\rightarrow$\texttt{rt} & $0.045$\cellcolor{green!26} & $\textbf{0.027}$\cellcolor{green!38} & $0.123$\cellcolor{red!21} & $0.131$\cellcolor{red!27} & $0.149$\cellcolor{red!38} & $0.065$\cellcolor{green!14} & $0.043$\cellcolor{green!28} & $0.151$\cellcolor{red!40} & $\textbf{0.027}$\cellcolor{green!38} & $0.038$\cellcolor{green!31} & $\textbf{0.025}$\cellcolor{green!40} & $0.031$\cellcolor{green!35} \\
 & \texttt{imdb}$\rightarrow$\texttt{yelp} & $0.072$\cellcolor{green!17} & $0.037$\cellcolor{green!35} & $0.068$\cellcolor{green!19} & $0.182$\cellcolor{red!40} & $0.073$\cellcolor{green!17} & $0.042$\cellcolor{green!33} & $\textbf{0.031}$\cellcolor{green!38} & $0.130$\cellcolor{red!12} & $0.037$\cellcolor{green!35} & $0.053$\cellcolor{green!27} & $\textbf{0.029}$\cellcolor{green!40} & $0.047$\cellcolor{green!30} \\
 & \texttt{rt}$\rightarrow$\texttt{imdb} & $0.075$\cellcolor{green!28} & $0.031$\cellcolor{green!39} & $0.143$\cellcolor{green!12} & $0.367$\cellcolor{red!40} & $0.161$\cellcolor{green!8} & $0.047$\cellcolor{green!35} & $0.034$\cellcolor{green!38} & $0.169$\cellcolor{green!6} & $0.031$\cellcolor{green!39} & $0.057$\cellcolor{green!32} & $\textbf{0.027}$\cellcolor{green!40} & $0.038$\cellcolor{green!37} \\
 & \texttt{rt}$\rightarrow$\texttt{yelp} & $0.032$\cellcolor{green!38} & $\textbf{0.027}$\cellcolor{green!39} & $0.078$\cellcolor{green!25} & $0.120$\cellcolor{green!13} & $0.107$\cellcolor{green!16} & $0.034$\cellcolor{green!37} & $\textbf{0.028}$\cellcolor{green!39} & $0.306$\cellcolor{red!40} & $\textbf{0.027}$\cellcolor{green!39} & $\textbf{0.028}$\cellcolor{green!39} & $0.029$\cellcolor{green!39} & $\textbf{0.027}$\cellcolor{green!40} \\
 & \texttt{yelp}$\rightarrow$\texttt{imdb} & $0.112$\cellcolor{red!6} & $0.083$\cellcolor{green!12} & $0.137$\cellcolor{red!22} & $0.164$\cellcolor{red!40} & $0.148$\cellcolor{red!29} & $0.094$\cellcolor{green!5} & $0.094$\cellcolor{green!5} & $0.160$\cellcolor{red!37} & $0.083$\cellcolor{green!12} & $0.101$\cellcolor{green!0} & $\textbf{0.040}$\cellcolor{green!40} & $0.091$\cellcolor{green!7} \\
 & \texttt{yelp}$\rightarrow$\texttt{rt} & $0.100$\cellcolor{red!4} & $0.074$\cellcolor{green!14} & $0.120$\cellcolor{red!17} & $0.154$\cellcolor{red!40} & $0.140$\cellcolor{red!30} & $0.093$\cellcolor{green!1} & $0.084$\cellcolor{green!6} & $0.135$\cellcolor{red!27} & $0.074$\cellcolor{green!14} & $0.087$\cellcolor{green!4} & $\textbf{0.035}$\cellcolor{green!40} & $0.082$\cellcolor{green!8} \\\midrule
\multirow{6}{*}{\begin{sideways}DistilBERT\;\end{sideways}} & \texttt{imdb}$\rightarrow$\texttt{rt} & $0.036$\cellcolor{green!16} & $0.036$\cellcolor{green!18} & $0.051$\cellcolor{red!11} & $0.056$\cellcolor{red!22} & $0.053$\cellcolor{red!15} & $0.038$\cellcolor{green!13} & $0.035$\cellcolor{green!18} & $0.066$\cellcolor{red!40} & $0.036$\cellcolor{green!18} & $0.033$\cellcolor{green!23} & $\textbf{0.024}$\cellcolor{green!40} & $0.033$\cellcolor{green!23} \\
 & \texttt{imdb}$\rightarrow$\texttt{yelp} & $\textbf{0.019}$\cellcolor{green!40} & $0.022$\cellcolor{green!28} & $0.025$\cellcolor{green!18} & $0.027$\cellcolor{green!11} & $0.024$\cellcolor{green!21} & $0.024$\cellcolor{green!21} & $\textbf{0.020}$\cellcolor{green!33} & $0.044$\cellcolor{red!40} & $0.022$\cellcolor{green!28} & $\textbf{0.019}$\cellcolor{green!37} & $0.023$\cellcolor{green!25} & $\textbf{0.019}$\cellcolor{green!39} \\
 & \texttt{rt}$\rightarrow$\texttt{imdb} & $0.057$\cellcolor{red!12} & $0.033$\cellcolor{green!30} & $0.045$\cellcolor{green!9} & $0.062$\cellcolor{red!22} & $0.047$\cellcolor{green!5} & $\textbf{0.031}$\cellcolor{green!35} & $\textbf{0.030}$\cellcolor{green!35} & $0.072$\cellcolor{red!40} & $0.033$\cellcolor{green!30} & $0.038$\cellcolor{green!21} & $\textbf{0.028}$\cellcolor{green!40} & $0.040$\cellcolor{green!17} \\
 & \texttt{rt}$\rightarrow$\texttt{yelp} & $0.094$\cellcolor{green!0} & $0.064$\cellcolor{green!20} & $0.125$\cellcolor{red!20} & $0.131$\cellcolor{red!24} & $0.112$\cellcolor{red!12} & $0.062$\cellcolor{green!21} & $0.056$\cellcolor{green!25} & $0.154$\cellcolor{red!40} & $0.064$\cellcolor{green!20} & $0.057$\cellcolor{green!24} & $\textbf{0.034}$\cellcolor{green!40} & $0.073$\cellcolor{green!14} \\
 & \texttt{yelp}$\rightarrow$\texttt{imdb} & $0.099$\cellcolor{red!11} & $0.083$\cellcolor{green!2} & $0.102$\cellcolor{red!13} & $0.133$\cellcolor{red!40} & $0.112$\cellcolor{red!22} & $0.084$\cellcolor{green!1} & $0.067$\cellcolor{green!16} & $0.116$\cellcolor{red!25} & $0.083$\cellcolor{green!2} & $0.086$\cellcolor{red!0} & $\textbf{0.039}$\cellcolor{green!40} & $0.085$\cellcolor{green!0} \\
 & \texttt{yelp}$\rightarrow$\texttt{rt} & $0.083$\cellcolor{red!11} & $0.068$\cellcolor{green!3} & $0.081$\cellcolor{red!9} & $0.112$\cellcolor{red!40} & $0.090$\cellcolor{red!18} & $0.067$\cellcolor{green!5} & $0.051$\cellcolor{green!21} & $0.091$\cellcolor{red!18} & $0.068$\cellcolor{green!3} & $0.058$\cellcolor{green!14} & $\textbf{0.032}$\cellcolor{green!40} & $0.071$\cellcolor{green!1} \\\midrule
\multirow{6}{*}{\begin{sideways}RoBERTa\;\end{sideways}} & \texttt{imdb}$\rightarrow$\texttt{rt} & $0.097$\cellcolor{green!1} & $0.073$\cellcolor{green!17} & $0.130$\cellcolor{red!20} & $0.160$\cellcolor{red!40} & $0.150$\cellcolor{red!33} & $0.076$\cellcolor{green!14} & $0.080$\cellcolor{green!12} & $0.140$\cellcolor{red!26} & $0.073$\cellcolor{green!17} & $0.081$\cellcolor{green!12} & $\textbf{0.038}$\cellcolor{green!40} & $0.079$\cellcolor{green!13} \\
 & \texttt{imdb}$\rightarrow$\texttt{yelp} & $0.019$\cellcolor{green!33} & $0.017$\cellcolor{green!38} & $\textbf{0.018}$\cellcolor{green!35} & $\textbf{0.019}$\cellcolor{green!34} & $\textbf{0.018}$\cellcolor{green!36} & $0.021$\cellcolor{green!27} & $\textbf{0.016}$\cellcolor{green!40} & $0.047$\cellcolor{red!40} & $0.017$\cellcolor{green!38} & $0.019$\cellcolor{green!34} & $0.023$\cellcolor{green!21} & $\textbf{0.017}$\cellcolor{green!39} \\
 & \texttt{rt}$\rightarrow$\texttt{imdb} & $0.109$\cellcolor{green!11} & $0.072$\cellcolor{green!26} & $0.233$\cellcolor{red!40} & $0.207$\cellcolor{red!29} & $0.187$\cellcolor{red!20} & $0.108$\cellcolor{green!11} & $0.076$\cellcolor{green!25} & $0.193$\cellcolor{red!23} & $0.072$\cellcolor{green!26} & $0.075$\cellcolor{green!25} & $\textbf{0.040}$\cellcolor{green!40} & $0.081$\cellcolor{green!23} \\
 & \texttt{rt}$\rightarrow$\texttt{yelp} & $0.099$\cellcolor{green!12} & $0.064$\cellcolor{green!27} & $0.221$\cellcolor{red!40} & $0.201$\cellcolor{red!31} & $0.179$\cellcolor{red!21} & $0.102$\cellcolor{green!11} & $0.067$\cellcolor{green!26} & $0.196$\cellcolor{red!29} & $0.064$\cellcolor{green!27} & $0.063$\cellcolor{green!27} & $\textbf{0.035}$\cellcolor{green!40} & $0.071$\cellcolor{green!24} \\
 & \texttt{yelp}$\rightarrow$\texttt{imdb} & $0.041$\cellcolor{green!11} & $0.037$\cellcolor{green!18} & $0.055$\cellcolor{red!11} & $0.053$\cellcolor{red!7} & $0.061$\cellcolor{red!21} & $0.047$\cellcolor{green!1} & $0.048$\cellcolor{red!0} & $0.072$\cellcolor{red!40} & $0.037$\cellcolor{green!18} & $0.038$\cellcolor{green!17} & $\textbf{0.024}$\cellcolor{green!40} & $0.037$\cellcolor{green!18} \\
 & \texttt{yelp}$\rightarrow$\texttt{rt} & $0.027$\cellcolor{green!8} & $0.023$\cellcolor{green!22} & $0.019$\cellcolor{green!33} & $0.042$\cellcolor{red!34} & $0.021$\cellcolor{green!25} & $0.021$\cellcolor{green!28} & $\textbf{0.017}$\cellcolor{green!40} & $0.044$\cellcolor{red!40} & $0.023$\cellcolor{green!22} & $0.023$\cellcolor{green!21} & $0.020$\cellcolor{green!28} & $0.022$\cellcolor{green!23} \\\bottomrule
\multirow{4}{*}{\begin{sideways}Wins\;\end{sideways}} &  $\Pr(M\succ {1}R)$ &$\dag 90.56\%$ & --- & --- & --- & --- & $\dag 89.89\%$ & $\dag 96.17\%$ & $54.28\%$ & $\dag 95.50\%$ & $\dag 93.78\%$ & $\dag 88.67\%$ & $\dag 95.83\%$ \\
 &  $\Pr(M\succ {2}R)$ &$\dag 75.00\%$ & --- & --- & --- & --- & $\dag 81.39\%$ & $\dag 88.06\%$ & $40.11\%$ & $\dag 85.56\%$ & $\dag 82.44\%$ & $\dag 84.94\%$ & $\dag 84.28\%$ \\
 &  $\Pr(M\succ {3}R)$ &$\dag 64.06\%$ & --- & --- & --- & --- & $\dag 72.83\%$ & $\dag 79.44\%$ & $30.33\%$ & $\dag 77.33\%$ & $\dag 73.22\%$ & $\dag 76.94\%$ & $\dag 76.11\%$ \\
 &  $\Pr(M\succ {4}R)$ &$14.06\%$ & --- & --- & --- & --- & $\dag 29.17\%$ & $\dag 36.50\%$ & $12.28\%$ & $\dag 39.72\%$ & $\dag 27.39\%$ & $\dag 62.78\%$ & $18.06\%$ \\\bottomrule
  & Ave Rank &7.28 & \textbf{4.69} & 8.12 & 9.84 & 8.52 & 5.94 & \textbf{4.72} & 9.40 & \textbf{4.64} & 5.68 & \textbf{3.66} & 5.52 \\\bottomrule
\end{tabular}
    }%
    \label{tab:quant:covshift}
\end{table}

All the newly proposed methods, with the sole exception of LEAP$_{\acc2\quant}$ (which internally relies on KDEy) perform better than the reference methods PACC, EMQ, and KDEy in most cases. The top performing methods overall include DoC$_{\acc2\quant}$, the proposed variant LEAP$^{\text{PCC}}_{\acc2\quant}$, and TransCal$_{\calib2\quant}$. This is interesting for a number of reasons. The first one is that, together with the calibration experiments, DoC gets consolidated as a versatile method for CS problems. Second, because a variant of LEAP equipped with PCC is something that has never been investigated in past literature and, when adapted to quantification via the $\acc2\quant$ method, it seems to work even better than vanilla PCC; LEAP is based on the resolution of a system of linear equations representing problem constrains, which may help improve the stability of the method. Third and foremost, because TransCal$_{\calib2\quant}$ has attained astonishing unexpected results, clearly beating any other method from the pool consistently and across all classifiers. The fact that TransCal performed comparatively better in quantification experiments than in calibration ones might be an indication that its post-hoc calibrated posteriors are better calibrated at the dataset level (i.e., in terms of overall expected bias) rather than at the bin level (i.e., are less informative at a finer-grained level). This suggests the inner workings of TransCal, based on importance weighting, deserve closer investigation for quantification under CS.

The quantification experiments under LS are reported in Table~\ref{tab:quant:labelshift}. In this case, there are no surprises: the reference quantification methods clearly dominate the experimental comparison, with KDEy standing as the top performer of the lot (this is in line with the results of~\cite{KDEyMoreo:2025}).
Among the adapted methods, the best performer is DoC$_{\acc2\quant}$. Although not comparable in performance with respect to the reference methods, it is interesting to note that DoC keeps performing reasonably well across different tasks and types of shift.

\begin{table}[tb!]
    \caption{Quantification performance under label shift in terms of AE}
    \vspace{0.5cm}
    \centering
    \resizebox{\textwidth}{!}{%
    \begin{tabular}{cc|c|cccc|cccccccc} \toprule
\multicolumn{2}{c}{} & \multicolumn{1}{c|}{Baselines} & \multicolumn{4}{c|}{Reference} & \multicolumn{8}{c}{Adapted Methods} \\
\multicolumn{2}{c}{} & \begin{sideways}CC\;\end{sideways} & \begin{sideways}PACC\;\end{sideways} & \begin{sideways}EMQ\;\end{sideways} & \begin{sideways}EMQ$^{\text{BCTS}}$\;\end{sideways} & \begin{sideways}KDEy\;\end{sideways} & \begin{sideways}ATC$_{\alpha 2 \rho}$\;\end{sideways} & \begin{sideways}DoC$_{\alpha 2 \rho}$\;\end{sideways} & \begin{sideways}LEAP$_{\alpha 2 \rho}$\;\end{sideways} & \begin{sideways}CPCS$_{\zeta 2 \rho}$\;\end{sideways} & \begin{sideways}TransCal$_{\zeta 2 \rho}$\;\end{sideways} & \begin{sideways}LasCal$_{\zeta 2 \rho}$\;\end{sideways} & \begin{sideways}Head2Tail$_{\zeta 2 \rho}$\;\end{sideways} & \begin{sideways}EMQ$^{\text{LasCal}}$\;\end{sideways} \\\midrule
\multirow{10}{*}{\begin{sideways}Logistic Regression\;\end{sideways}} & \texttt{cmc.3} & $0.323$\cellcolor{red!40} & $0.117$\cellcolor{green!30} & $0.122$\cellcolor{green!28} & $0.105$\cellcolor{green!34} & $0.112$\cellcolor{green!32} & $0.238$\cellcolor{red!10} & $\textbf{0.090}$\cellcolor{green!40} & $0.198$\cellcolor{green!2} & $0.243$\cellcolor{red!12} & $0.244$\cellcolor{red!13} & $0.208$\cellcolor{red!0} & $0.245$\cellcolor{red!13} & $0.251$\cellcolor{red!15} \\
 & \texttt{yeast} & $0.239$\cellcolor{red!40} & $0.079$\cellcolor{green!31} & $0.109$\cellcolor{green!18} & $0.074$\cellcolor{green!33} & $\textbf{0.060}$\cellcolor{green!40} & $0.192$\cellcolor{red!18} & $0.090$\cellcolor{green!26} & $0.092$\cellcolor{green!25} & $0.201$\cellcolor{red!22} & $0.223$\cellcolor{red!32} & $0.162$\cellcolor{red!5} & $0.212$\cellcolor{red!27} & $0.217$\cellcolor{red!30} \\
 & \texttt{semeion} & $0.134$\cellcolor{red!13} & $0.028$\cellcolor{green!31} & $0.033$\cellcolor{green!29} & $\textbf{0.009}$\cellcolor{green!40} & $0.026$\cellcolor{green!32} & $0.093$\cellcolor{green!4} & $0.026$\cellcolor{green!32} & $0.025$\cellcolor{green!33} & $0.127$\cellcolor{red!9} & $0.198$\cellcolor{red!40} & $0.125$\cellcolor{red!9} & $0.122$\cellcolor{red!7} & $0.091$\cellcolor{green!5} \\
 & \texttt{wine-q-red} & $0.126$\cellcolor{green!0} & $0.044$\cellcolor{green!37} & $\textbf{0.038}$\cellcolor{green!40} & $\textbf{0.041}$\cellcolor{green!38} & $\textbf{0.044}$\cellcolor{green!37} & $0.160$\cellcolor{red!15} & $0.077$\cellcolor{green!22} & $0.078$\cellcolor{green!22} & $0.179$\cellcolor{red!24} & $0.214$\cellcolor{red!40} & $0.178$\cellcolor{red!23} & $0.204$\cellcolor{red!35} & $\textbf{0.039}$\cellcolor{green!39} \\
 & \texttt{ctg.1} & $0.084$\cellcolor{green!5} & $0.033$\cellcolor{green!31} & $\textbf{0.015}$\cellcolor{green!40} & $0.015$\cellcolor{green!39} & $\textbf{0.015}$\cellcolor{green!39} & $0.082$\cellcolor{green!6} & $0.030$\cellcolor{green!32} & $0.032$\cellcolor{green!31} & $0.110$\cellcolor{red!7} & $0.176$\cellcolor{red!40} & $0.118$\cellcolor{red!11} & $0.114$\cellcolor{red!9} & $0.025$\cellcolor{green!34} \\
 & \texttt{ctg.2} & $0.150$\cellcolor{red!13} & $0.055$\cellcolor{green!24} & $\textbf{0.016}$\cellcolor{green!40} & $0.048$\cellcolor{green!27} & $0.030$\cellcolor{green!34} & $0.210$\cellcolor{red!38} & $0.032$\cellcolor{green!33} & $0.025$\cellcolor{green!36} & $0.178$\cellcolor{red!25} & $0.205$\cellcolor{red!35} & $0.146$\cellcolor{red!12} & $0.187$\cellcolor{red!29} & $0.214$\cellcolor{red!40} \\
 & \texttt{ctg.3} & $0.123$\cellcolor{red!14} & $0.024$\cellcolor{green!33} & $\textbf{0.010}$\cellcolor{green!40} & $\textbf{0.010}$\cellcolor{green!39} & $0.022$\cellcolor{green!34} & $0.082$\cellcolor{green!5} & $0.072$\cellcolor{green!10} & $0.095$\cellcolor{red!1} & $0.123$\cellcolor{red!14} & $0.175$\cellcolor{red!40} & $0.129$\cellcolor{red!17} & $0.125$\cellcolor{red!15} & $0.096$\cellcolor{red!1} \\
 & \texttt{spambase} & $0.043$\cellcolor{green!24} & $\textbf{0.015}$\cellcolor{green!39} & $0.015$\cellcolor{green!38} & $0.015$\cellcolor{green!38} & $\textbf{0.013}$\cellcolor{green!40} & $0.049$\cellcolor{green!20} & $0.029$\cellcolor{green!31} & $0.020$\cellcolor{green!36} & $0.066$\cellcolor{green!11} & $0.162$\cellcolor{red!40} & $0.062$\cellcolor{green!13} & $0.082$\cellcolor{green!3} & $0.031$\cellcolor{green!30} \\
 & \texttt{wine-q-white} & $0.232$\cellcolor{red!40} & $0.047$\cellcolor{green!35} & $0.041$\cellcolor{green!37} & $\textbf{0.035}$\cellcolor{green!40} & $0.041$\cellcolor{green!37} & $0.193$\cellcolor{red!24} & $0.068$\cellcolor{green!26} & $0.062$\cellcolor{green!29} & $0.213$\cellcolor{red!32} & $0.226$\cellcolor{red!37} & $0.201$\cellcolor{red!27} & $0.219$\cellcolor{red!34} & $0.130$\cellcolor{green!1} \\
 & \texttt{pageblocks.5} & $0.339$\cellcolor{red!24} & $\textbf{0.034}$\cellcolor{green!40} & $0.045$\cellcolor{green!37} & $0.045$\cellcolor{green!37} & $0.046$\cellcolor{green!37} & $0.213$\cellcolor{green!2} & $0.068$\cellcolor{green!32} & $0.060$\cellcolor{green!34} & $0.339$\cellcolor{red!24} & $0.213$\cellcolor{green!2} & $0.168$\cellcolor{green!11} & $0.324$\cellcolor{red!21} & $0.412$\cellcolor{red!40} \\\midrule
\multirow{10}{*}{\begin{sideways}Naïve Bayes\;\end{sideways}} & \texttt{cmc.3} & $0.241$\cellcolor{red!40} & $\textbf{0.111}$\cellcolor{green!29} & $\textbf{0.091}$\cellcolor{green!40} & $\textbf{0.106}$\cellcolor{green!32} & $0.115$\cellcolor{green!26} & $0.174$\cellcolor{red!4} & $\textbf{0.119}$\cellcolor{green!24} & $\textbf{0.119}$\cellcolor{green!25} & $0.224$\cellcolor{red!30} & $0.236$\cellcolor{red!37} & $0.202$\cellcolor{red!19} & $0.231$\cellcolor{red!35} & $0.161$\cellcolor{green!2} \\
 & \texttt{yeast} & $0.508$\cellcolor{red!37} & $0.336$\cellcolor{green!1} & $0.508$\cellcolor{red!37} & $0.508$\cellcolor{red!37} & $0.336$\cellcolor{green!1} & $\textbf{0.170}$\cellcolor{green!40} & $0.238$\cellcolor{green!24} & $0.518$\cellcolor{red!40} & $0.249$\cellcolor{green!21} & $0.218$\cellcolor{green!29} & $0.279$\cellcolor{green!15} & $0.252$\cellcolor{green!21} & $0.505$\cellcolor{red!37} \\
 & \texttt{semeion} & $0.100$\cellcolor{green!39} & $\textbf{0.102}$\cellcolor{green!38} & $0.100$\cellcolor{green!39} & $0.100$\cellcolor{green!39} & $\textbf{0.100}$\cellcolor{green!39} & $\textbf{0.097}$\cellcolor{green!40} & $0.232$\cellcolor{green!14} & $0.527$\cellcolor{red!40} & $0.131$\cellcolor{green!33} & $0.206$\cellcolor{green!19} & $0.157$\cellcolor{green!28} & $0.168$\cellcolor{green!26} & $0.100$\cellcolor{green!39} \\
 & \texttt{wine-q-red} & $0.141$\cellcolor{green!1} & $0.068$\cellcolor{green!37} & $0.103$\cellcolor{green!19} & $0.130$\cellcolor{green!6} & $\textbf{0.063}$\cellcolor{green!40} & $0.162$\cellcolor{red!9} & $0.156$\cellcolor{red!6} & $0.184$\cellcolor{red!20} & $0.194$\cellcolor{red!25} & $0.211$\cellcolor{red!34} & $0.193$\cellcolor{red!25} & $0.222$\cellcolor{red!40} & $\textbf{0.066}$\cellcolor{green!38} \\
 & \texttt{ctg.1} & $0.112$\cellcolor{red!6} & $0.024$\cellcolor{green!37} & $0.059$\cellcolor{green!19} & $\textbf{0.019}$\cellcolor{green!40} & $\textbf{0.020}$\cellcolor{green!39} & $0.098$\cellcolor{green!0} & $0.052$\cellcolor{green!23} & $0.028$\cellcolor{green!35} & $0.125$\cellcolor{red!12} & $0.173$\cellcolor{red!36} & $0.133$\cellcolor{red!16} & $0.179$\cellcolor{red!40} & $0.100$\cellcolor{red!0} \\
 & \texttt{ctg.2} & $0.122$\cellcolor{green!18} & $0.044$\cellcolor{green!37} & $0.128$\cellcolor{green!16} & $0.128$\cellcolor{green!16} & $\textbf{0.036}$\cellcolor{green!40} & $0.168$\cellcolor{green!6} & $0.145$\cellcolor{green!12} & $0.350$\cellcolor{red!40} & $0.187$\cellcolor{green!1} & $0.200$\cellcolor{red!1} & $0.159$\cellcolor{green!8} & $0.194$\cellcolor{red!0} & $0.273$\cellcolor{red!20} \\
 & \texttt{ctg.3} & $0.153$\cellcolor{red!1} & $0.036$\cellcolor{green!36} & $0.140$\cellcolor{green!3} & $0.140$\cellcolor{green!3} & $\textbf{0.027}$\cellcolor{green!40} & $0.110$\cellcolor{green!12} & $0.133$\cellcolor{green!5} & $0.271$\cellcolor{red!40} & $0.139$\cellcolor{green!3} & $0.181$\cellcolor{red!10} & $0.137$\cellcolor{green!3} & $0.145$\cellcolor{green!1} & $0.256$\cellcolor{red!35} \\
 & \texttt{spambase} & $0.119$\cellcolor{green!16} & $\textbf{0.024}$\cellcolor{green!39} & $0.119$\cellcolor{green!15} & $0.119$\cellcolor{green!15} & $\textbf{0.024}$\cellcolor{green!40} & $0.085$\cellcolor{green!24} & $0.196$\cellcolor{red!3} & $0.341$\cellcolor{red!40} & $0.141$\cellcolor{green!10} & $0.173$\cellcolor{green!2} & $0.155$\cellcolor{green!6} & $0.153$\cellcolor{green!7} & $0.104$\cellcolor{green!19} \\
 & \texttt{wine-q-white} & $0.195$\cellcolor{red!26} & $\textbf{0.063}$\cellcolor{green!35} & $0.125$\cellcolor{green!6} & $\textbf{0.054}$\cellcolor{green!40} & $\textbf{0.063}$\cellcolor{green!35} & $0.202$\cellcolor{red!29} & $0.113$\cellcolor{green!12} & $0.093$\cellcolor{green!21} & $0.211$\cellcolor{red!33} & $0.220$\cellcolor{red!38} & $0.206$\cellcolor{red!31} & $0.223$\cellcolor{red!40} & $0.207$\cellcolor{red!32} \\
 & \texttt{pageblocks.5} & $0.282$\cellcolor{red!18} & $\textbf{0.043}$\cellcolor{green!40} & $0.131$\cellcolor{green!18} & $0.131$\cellcolor{green!18} & $0.076$\cellcolor{green!31} & $0.223$\cellcolor{red!4} & $0.162$\cellcolor{green!10} & $0.368$\cellcolor{red!40} & $0.216$\cellcolor{red!2} & $0.217$\cellcolor{red!2} & $0.160$\cellcolor{green!11} & $0.274$\cellcolor{red!16} & $0.263$\cellcolor{red!14} \\\midrule
\multirow{10}{*}{\begin{sideways}k Nearest Neighbor\;\end{sideways}} & \texttt{cmc.3} & $0.319$\cellcolor{red!40} & $\textbf{0.094}$\cellcolor{green!40} & $0.128$\cellcolor{green!27} & $0.128$\cellcolor{green!27} & $\textbf{0.108}$\cellcolor{green!34} & $0.210$\cellcolor{red!1} & $\textbf{0.109}$\cellcolor{green!34} & $0.165$\cellcolor{green!14} & $0.231$\cellcolor{red!8} & $0.241$\cellcolor{red!12} & $0.184$\cellcolor{green!8} & $0.238$\cellcolor{red!11} & $0.209$\cellcolor{red!0} \\
 & \texttt{yeast} & $0.246$\cellcolor{red!40} & $0.094$\cellcolor{green!22} & $\textbf{0.064}$\cellcolor{green!35} & $\textbf{0.064}$\cellcolor{green!35} & $\textbf{0.059}$\cellcolor{green!37} & $0.219$\cellcolor{red!29} & $\textbf{0.052}$\cellcolor{green!40} & $\textbf{0.055}$\cellcolor{green!38} & $0.199$\cellcolor{red!20} & $0.219$\cellcolor{red!28} & $0.183$\cellcolor{red!13} & $0.202$\cellcolor{red!22} & $0.172$\cellcolor{red!9} \\
 & \texttt{semeion} & $0.172$\cellcolor{red!26} & $\textbf{0.015}$\cellcolor{green!38} & $\textbf{0.012}$\cellcolor{green!39} & $\textbf{0.012}$\cellcolor{green!39} & $\textbf{0.012}$\cellcolor{green!40} & $0.056$\cellcolor{green!21} & $0.045$\cellcolor{green!26} & $0.072$\cellcolor{green!14} & $0.151$\cellcolor{red!17} & $0.205$\cellcolor{red!40} & $0.155$\cellcolor{red!19} & $0.154$\cellcolor{red!18} & $0.016$\cellcolor{green!38} \\
 & \texttt{wine-q-red} & $0.144$\cellcolor{red!3} & $\textbf{0.056}$\cellcolor{green!39} & $0.070$\cellcolor{green!32} & $0.070$\cellcolor{green!32} & $\textbf{0.054}$\cellcolor{green!40} & $0.155$\cellcolor{red!8} & $0.101$\cellcolor{green!17} & $0.165$\cellcolor{red!13} & $0.194$\cellcolor{red!27} & $0.220$\cellcolor{red!40} & $0.201$\cellcolor{red!30} & $0.214$\cellcolor{red!37} & $\textbf{0.056}$\cellcolor{green!39} \\
 & \texttt{ctg.1} & $0.122$\cellcolor{red!9} & $0.041$\cellcolor{green!28} & $\textbf{0.017}$\cellcolor{green!40} & $\textbf{0.017}$\cellcolor{green!40} & $0.023$\cellcolor{green!36} & $0.074$\cellcolor{green!12} & $0.026$\cellcolor{green!35} & $0.023$\cellcolor{green!36} & $0.147$\cellcolor{red!21} & $0.186$\cellcolor{red!40} & $0.151$\cellcolor{red!23} & $0.146$\cellcolor{red!21} & $\textbf{0.020}$\cellcolor{green!38} \\
 & \texttt{ctg.2} & $0.228$\cellcolor{red!40} & $0.072$\cellcolor{green!19} & $\textbf{0.017}$\cellcolor{green!40} & $\textbf{0.017}$\cellcolor{green!40} & $0.031$\cellcolor{green!34} & $0.156$\cellcolor{red!12} & $0.028$\cellcolor{green!35} & $0.027$\cellcolor{green!36} & $0.176$\cellcolor{red!20} & $0.201$\cellcolor{red!29} & $0.139$\cellcolor{red!6} & $0.182$\cellcolor{red!22} & $0.032$\cellcolor{green!34} \\
 & \texttt{ctg.3} & $0.219$\cellcolor{red!40} & $0.036$\cellcolor{green!32} & $\textbf{0.017}$\cellcolor{green!39} & $\textbf{0.017}$\cellcolor{green!39} & $\textbf{0.018}$\cellcolor{green!39} & $0.098$\cellcolor{green!7} & $\textbf{0.021}$\cellcolor{green!38} & $\textbf{0.018}$\cellcolor{green!39} & $0.188$\cellcolor{red!27} & $0.193$\cellcolor{red!29} & $0.166$\cellcolor{red!19} & $0.189$\cellcolor{red!28} & $\textbf{0.017}$\cellcolor{green!40} \\
 & \texttt{spambase} & $0.080$\cellcolor{green!7} & $\textbf{0.017}$\cellcolor{green!39} & $0.020$\cellcolor{green!37} & $0.020$\cellcolor{green!37} & $\textbf{0.016}$\cellcolor{green!40} & $0.059$\cellcolor{green!17} & $0.031$\cellcolor{green!32} & $0.022$\cellcolor{green!36} & $0.090$\cellcolor{green!2} & $0.174$\cellcolor{red!40} & $0.091$\cellcolor{green!1} & $0.096$\cellcolor{red!0} & $0.021$\cellcolor{green!37} \\
 & \texttt{wine-q-white} & $0.172$\cellcolor{red!18} & $\textbf{0.044}$\cellcolor{green!39} & $0.098$\cellcolor{green!14} & $0.098$\cellcolor{green!14} & $\textbf{0.042}$\cellcolor{green!40} & $0.176$\cellcolor{red!19} & $0.154$\cellcolor{red!10} & $0.217$\cellcolor{red!38} & $0.206$\cellcolor{red!33} & $0.221$\cellcolor{red!40} & $0.196$\cellcolor{red!28} & $0.210$\cellcolor{red!35} & $0.052$\cellcolor{green!35} \\
 & \texttt{pageblocks.5} & $0.351$\cellcolor{red!40} & $\textbf{0.031}$\cellcolor{green!40} & $0.107$\cellcolor{green!20} & $0.107$\cellcolor{green!20} & $0.059$\cellcolor{green!32} & $0.217$\cellcolor{red!6} & $0.094$\cellcolor{green!24} & $0.082$\cellcolor{green!27} & $0.252$\cellcolor{red!15} & $0.213$\cellcolor{red!5} & $0.142$\cellcolor{green!12} & $0.304$\cellcolor{red!28} & $0.088$\cellcolor{green!25} \\\midrule
\multirow{10}{*}{\begin{sideways}Multi-layer Perceptron\;\end{sideways}} & \texttt{cmc.3} & $0.218$\cellcolor{red!6} & $0.083$\cellcolor{green!33} & $0.092$\cellcolor{green!31} & $\textbf{0.062}$\cellcolor{green!40} & $0.070$\cellcolor{green!37} & $0.203$\cellcolor{red!2} & $0.154$\cellcolor{green!12} & $0.163$\cellcolor{green!9} & $0.222$\cellcolor{red!8} & $0.232$\cellcolor{red!10} & $0.191$\cellcolor{green!1} & $0.227$\cellcolor{red!9} & $0.328$\cellcolor{red!40} \\
 & \texttt{yeast} & $0.195$\cellcolor{red!28} & $\textbf{0.046}$\cellcolor{green!38} & $0.063$\cellcolor{green!31} & $\textbf{0.044}$\cellcolor{green!40} & $\textbf{0.046}$\cellcolor{green!38} & $0.213$\cellcolor{red!37} & $0.102$\cellcolor{green!13} & $0.101$\cellcolor{green!13} & $0.199$\cellcolor{red!30} & $0.219$\cellcolor{red!40} & $0.179$\cellcolor{red!21} & $0.202$\cellcolor{red!32} & $0.202$\cellcolor{red!32} \\
 & \texttt{semeion} & $0.075$\cellcolor{green!10} & $0.020$\cellcolor{green!34} & $\textbf{0.010}$\cellcolor{green!39} & $\textbf{0.009}$\cellcolor{green!40} & $0.014$\cellcolor{green!37} & $0.046$\cellcolor{green!23} & $0.023$\cellcolor{green!33} & $0.024$\cellcolor{green!33} & $0.079$\cellcolor{green!8} & $0.189$\cellcolor{red!40} & $0.084$\cellcolor{green!6} & $0.079$\cellcolor{green!8} & $0.056$\cellcolor{green!19} \\
 & \texttt{wine-q-red} & $0.119$\cellcolor{green!1} & $\textbf{0.040}$\cellcolor{green!38} & $0.054$\cellcolor{green!31} & $\textbf{0.037}$\cellcolor{green!39} & $\textbf{0.036}$\cellcolor{green!40} & $0.163$\cellcolor{red!19} & $0.082$\cellcolor{green!18} & $0.064$\cellcolor{green!27} & $0.176$\cellcolor{red!25} & $0.207$\cellcolor{red!40} & $0.177$\cellcolor{red!25} & $0.197$\cellcolor{red!35} & $\textbf{0.042}$\cellcolor{green!37} \\
 & \texttt{ctg.1} & $0.082$\cellcolor{green!4} & $0.018$\cellcolor{green!36} & $0.016$\cellcolor{green!37} & $\textbf{0.011}$\cellcolor{green!40} & $\textbf{0.012}$\cellcolor{green!39} & $0.096$\cellcolor{red!3} & $0.024$\cellcolor{green!33} & $0.036$\cellcolor{green!27} & $0.103$\cellcolor{red!6} & $0.168$\cellcolor{red!40} & $0.107$\cellcolor{red!8} & $0.106$\cellcolor{red!8} & $0.053$\cellcolor{green!18} \\
 & \texttt{ctg.2} & $0.181$\cellcolor{red!29} & $\textbf{0.037}$\cellcolor{green!38} & $0.053$\cellcolor{green!30} & $\textbf{0.033}$\cellcolor{green!40} & $\textbf{0.036}$\cellcolor{green!38} & $0.180$\cellcolor{red!29} & $0.079$\cellcolor{green!18} & $0.126$\cellcolor{red!3} & $0.182$\cellcolor{red!30} & $0.202$\cellcolor{red!40} & $0.156$\cellcolor{red!18} & $0.187$\cellcolor{red!32} & $0.190$\cellcolor{red!34} \\
 & \texttt{ctg.3} & $0.112$\cellcolor{red!11} & $0.026$\cellcolor{green!32} & $\textbf{0.010}$\cellcolor{green!40} & $0.020$\cellcolor{green!35} & $0.041$\cellcolor{green!24} & $0.083$\cellcolor{green!3} & $0.158$\cellcolor{red!34} & $0.138$\cellcolor{red!24} & $0.110$\cellcolor{red!9} & $0.170$\cellcolor{red!40} & $0.118$\cellcolor{red!14} & $0.111$\cellcolor{red!10} & $0.109$\cellcolor{red!9} \\
 & \texttt{spambase} & $0.039$\cellcolor{green!25} & $\textbf{0.014}$\cellcolor{green!39} & $\textbf{0.016}$\cellcolor{green!38} & $\textbf{0.016}$\cellcolor{green!38} & $\textbf{0.014}$\cellcolor{green!40} & $0.050$\cellcolor{green!19} & $0.027$\cellcolor{green!32} & $0.022$\cellcolor{green!35} & $0.060$\cellcolor{green!13} & $0.153$\cellcolor{red!40} & $0.051$\cellcolor{green!18} & $0.060$\cellcolor{green!13} & $0.035$\cellcolor{green!27} \\
 & \texttt{wine-q-white} & $0.172$\cellcolor{red!21} & $0.033$\cellcolor{green!38} & $\textbf{0.034}$\cellcolor{green!37} & $\textbf{0.030}$\cellcolor{green!39} & $\textbf{0.029}$\cellcolor{green!40} & $0.156$\cellcolor{red!14} & $0.068$\cellcolor{green!23} & $0.072$\cellcolor{green!21} & $0.192$\cellcolor{red!30} & $0.214$\cellcolor{red!40} & $0.185$\cellcolor{red!27} & $0.191$\cellcolor{red!30} & $0.072$\cellcolor{green!21} \\
 & \texttt{pageblocks.5} & $0.219$\cellcolor{red!29} & $0.029$\cellcolor{green!37} & $0.029$\cellcolor{green!37} & $0.029$\cellcolor{green!37} & $\textbf{0.022}$\cellcolor{green!40} & $0.148$\cellcolor{red!4} & $0.059$\cellcolor{green!26} & $0.038$\cellcolor{green!34} & $0.213$\cellcolor{red!27} & $0.191$\cellcolor{red!19} & $0.177$\cellcolor{red!14} & $0.204$\cellcolor{red!24} & $0.248$\cellcolor{red!40} \\\bottomrule
\multirow{4}{*}{\begin{sideways}Wins\;\end{sideways}} &  $\Pr(M\succ {1}R)$ &$25.47\%$ & --- & --- & --- & --- & $32.30\%$ & $54.67\%$ & $48.98\%$ & $25.45\%$ & $20.95\%$ & $31.40\%$ & $23.15\%$ & $49.33\%$ \\
 &  $\Pr(M\succ {2}R)$ &$21.43\%$ & --- & --- & --- & --- & $27.27\%$ & $43.43\%$ & $38.70\%$ & $21.10\%$ & $17.55\%$ & $25.97\%$ & $19.48\%$ & $39.40\%$ \\
 &  $\Pr(M\succ {3}R)$ &$11.10\%$ & --- & --- & --- & --- & $14.88\%$ & $26.52\%$ & $23.33\%$ & $12.40\%$ & $10.12\%$ & $14.60\%$ & $11.35\%$ & $21.70\%$ \\
 &  $\Pr(M\succ {4}R)$ &$7.72\%$ & --- & --- & --- & --- & $10.47\%$ & $16.75\%$ & $14.17\%$ & $8.20\%$ & $6.42\%$ & $9.65\%$ & $7.83\%$ & $13.45\%$ \\\bottomrule
  & Ave Rank &9.04 & \textbf{4.50} & \textbf{4.87} & \textbf{4.59} & \textbf{4.09} & 7.60 & 6.10 & 6.90 & 8.96 & 9.72 & 8.14 & 9.51 & 6.98 \\\bottomrule
\end{tabular}
    }%
    \label{tab:quant:labelshift}
\end{table}

Figure~\ref{fig:quant:errbyshift} show the ``errors-by-shift'' plots of quantification experiments, side-by-side for both types of shift, for a selection of methods. Interestingly enough, TransCal$_{\calib2\quant}$ is not only the best performer under CS, but is also the most robust to the intensity of this shift. The rest of the methods' performance suffer when the level of CS increases. Concerning the level of LS, not only the three reference methods PACC, EMQ, and KDEy, but also the new adaptation DoC$_{\acc2\quant}$, seem to be pretty robust against it. Despite the fact that LasCal is designed to counter LS, its performance degrades notably as the intensity of the shift increases.

Figure~\ref{fig:quant:cddiagram} display the CD-diagrams for the quantification experiments for both types of shift side by side.
It is interesting to see how TransCal$_{\calib2\quant}$ dominates the rank for CS in isolation, i.e., without being included in any critical difference group. The same plot reveals LEAP$^{\text{PCC}}_{\acc2\quant}$ and DoC$_{\acc2\quant}$ are very close to each other, in a statistically significant sense, to PCC under CS.
By comparing the CD-diagrams for both types of shift, it is interesting to see how these are almost specular, i.e., the best performing methods for CS are among the worst performing methods for LS, and viceversa (see, e.g., TransCal$_{\calib2\quant}$). 
This looks like an indication that, in order to make an informed decision about which method to use, we should be pretty sure about the type of shift at play, which in turn suggests more research on dataset shift detection \cite{Rabanser:2019ba} urges.
A notable exception to this phenomenon is DoC$_{\acc2\quant}$ which occupies a relatively high rank in both types of shift.

\begin{figure}
    \centering
    \includegraphics[width=0.48\linewidth]{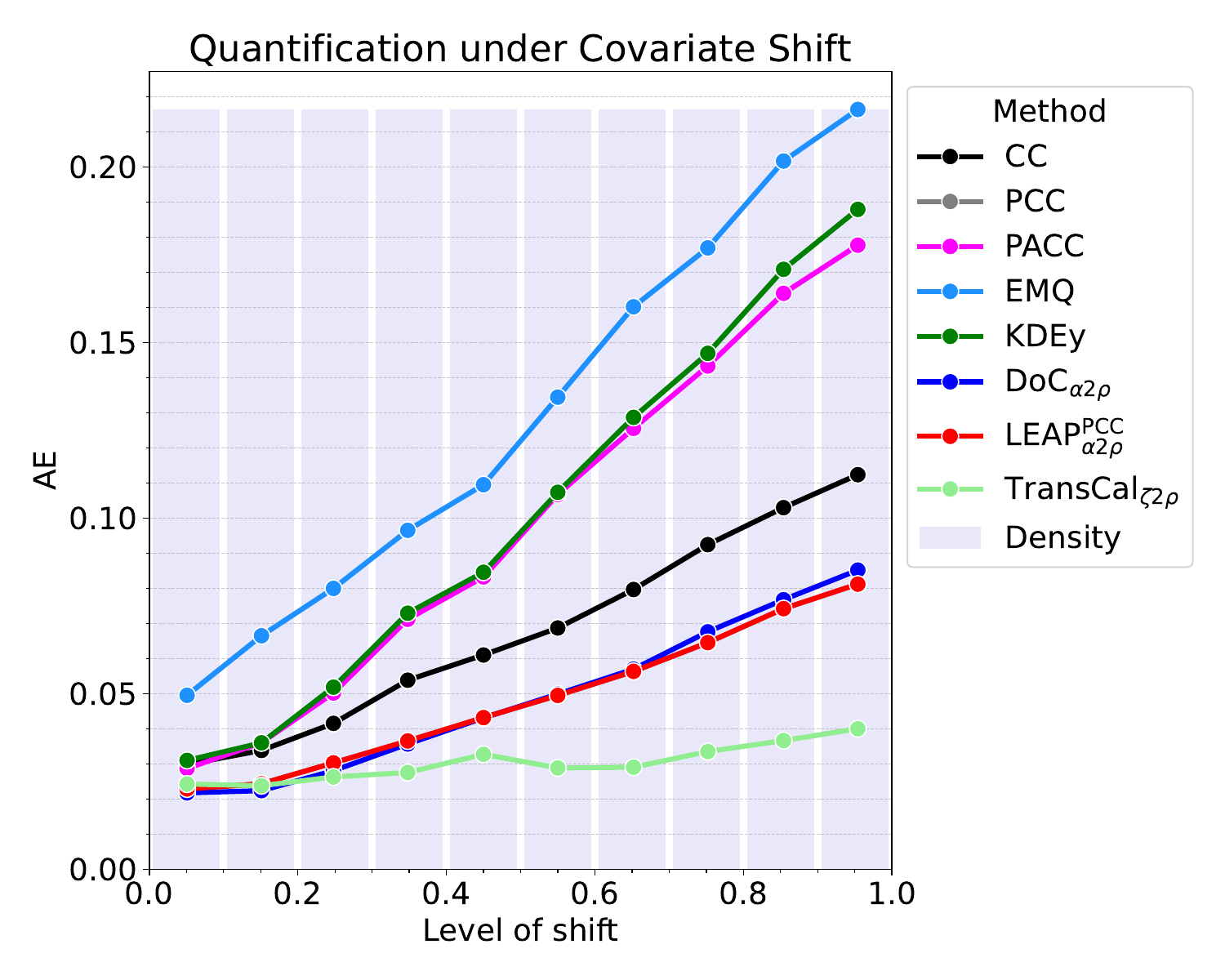}
    \includegraphics[width=0.48\linewidth]{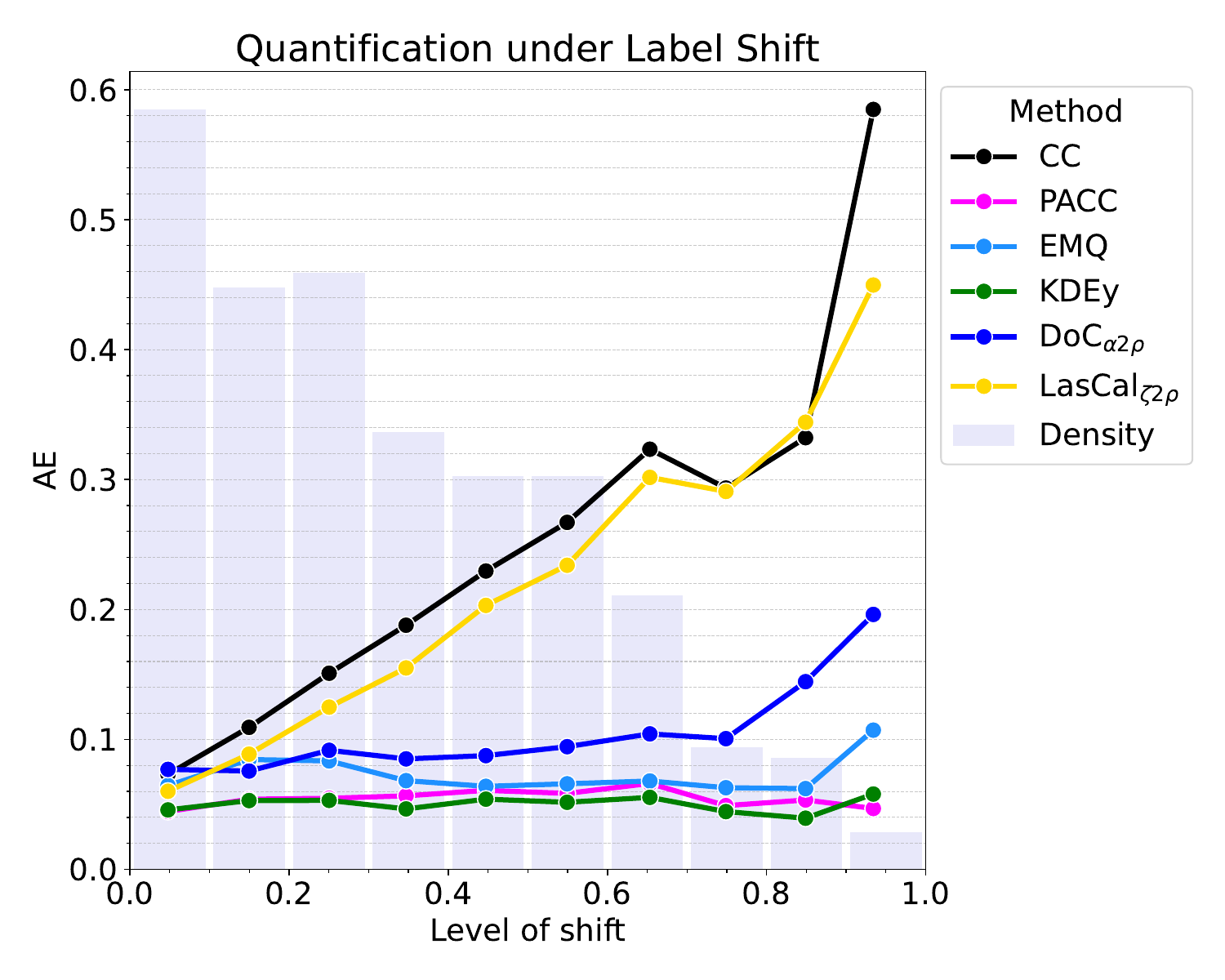}
    \caption{Quantification error in terms of AE as a function of shift intensity in CS experiments (left panel) and LS (right panel).}
    \label{fig:quant:errbyshift}

    \vspace{1cm}

    \centering
    \resizebox{0.48\textwidth}{!}{%
    \begin{tikzpicture}[
  treatment line/.style={rounded corners=1.5pt, line cap=round, shorten >=1pt},
  treatment label/.style={font=\small},
  group line/.style={ultra thick},
]

\begin{axis}[
  clip={false},
  axis x line={center},
  axis y line={none},
  axis line style={-},
  xmin={1},
  ymax={0},
  scale only axis={true},
  width={\axisdefaultwidth},
  ticklabel style={anchor=south, yshift=1.3*\pgfkeysvalueof{/pgfplots/major tick length}, font=\small},
  every tick/.style={draw=black},
  major tick style={yshift=.5*\pgfkeysvalueof{/pgfplots/major tick length}},
  minor tick style={yshift=.5*\pgfkeysvalueof{/pgfplots/minor tick length}},
  title style={yshift=\baselineskip},
  xmax={12},
  ymin={-7.5},
  height={8\baselineskip},
  title={Quantification under Covariate Shift},
]

\draw[treatment line] ([yshift=-2pt] axis cs:3.6633333333333336, 0) |- (axis cs:2.6633333333333336, -2.0)
  node[treatment label, anchor=east] {TransCal$_{\zeta 2 \rho}$};
\draw[treatment line] ([yshift=-2pt] axis cs:4.6386111111111115, 0) |- (axis cs:2.6633333333333336, -3.0)
  node[treatment label, anchor=east] {LEAP$_{\alpha 2 \rho}^{\text{PCC}}$};
\draw[treatment line] ([yshift=-2pt] axis cs:4.693611111111111, 0) |- (axis cs:2.6633333333333336, -4.0)
  node[treatment label, anchor=east] {\textbf{PCC}};
\draw[treatment line] ([yshift=-2pt] axis cs:4.721111111111111, 0) |- (axis cs:2.6633333333333336, -5.0)
  node[treatment label, anchor=east] {DoC$_{\alpha 2 \rho}$};
\draw[treatment line] ([yshift=-2pt] axis cs:5.5183333333333335, 0) |- (axis cs:2.6633333333333336, -6.0)
  node[treatment label, anchor=east] {LasCal$_{\zeta 2 \rho}$};
\draw[treatment line] ([yshift=-2pt] axis cs:5.675833333333333, 0) |- (axis cs:2.6633333333333336, -7.0)
  node[treatment label, anchor=east] {CPCS$_{\zeta 2 \rho}$};
\draw[treatment line] ([yshift=-2pt] axis cs:5.941111111111111, 0) |- (axis cs:10.835, -7.0)
  node[treatment label, anchor=west] {ATC$_{\alpha 2 \rho}$};
\draw[treatment line] ([yshift=-2pt] axis cs:7.275277777777778, 0) |- (axis cs:10.835, -6.0)
  node[treatment label, anchor=west] {CC};
\draw[treatment line] ([yshift=-2pt] axis cs:8.116111111111111, 0) |- (axis cs:10.835, -5.0)
  node[treatment label, anchor=west] {\textbf{PACC}};
\draw[treatment line] ([yshift=-2pt] axis cs:8.524444444444445, 0) |- (axis cs:10.835, -4.0)
  node[treatment label, anchor=west] {\textbf{KDEy}};
\draw[treatment line] ([yshift=-2pt] axis cs:9.397222222222222, 0) |- (axis cs:10.835, -3.0)
  node[treatment label, anchor=west] {LEAP$_{\alpha 2 \rho}$};
\draw[treatment line] ([yshift=-2pt] axis cs:9.835, 0) |- (axis cs:10.835, -2.0)
  node[treatment label, anchor=west] {\textbf{EMQ}};
\draw[group line] (axis cs:4.6386111111111115, -2.0) -- (axis cs:4.721111111111111, -2.0);
\draw[group line] (axis cs:9.397222222222222, -1.3333333333333333) -- (axis cs:9.835, -1.3333333333333333);

\end{axis}
\end{tikzpicture}
    }
    \resizebox{0.48\textwidth}{!}{%
    \begin{tikzpicture}[
  treatment line/.style={rounded corners=1.5pt, line cap=round, shorten >=1pt},
  treatment label/.style={font=\small},
  group line/.style={ultra thick},
]

\begin{axis}[
  clip={false},
  axis x line={center},
  axis y line={none},
  axis line style={-},
  xmin={1},
  ymax={0},
  scale only axis={true},
  width={\axisdefaultwidth},
  ticklabel style={anchor=south, yshift=1.3*\pgfkeysvalueof{/pgfplots/major tick length}, font=\small},
  every tick/.style={draw=black},
  major tick style={yshift=.5*\pgfkeysvalueof{/pgfplots/major tick length}},
  minor tick style={yshift=.5*\pgfkeysvalueof{/pgfplots/minor tick length}},
  title style={yshift=\baselineskip},
  xmax={13},
  ymin={-7.5},
  height={8\baselineskip},
  xtick={1,4,7,10,13},
  minor x tick num={1},
  title={Quantification under Label Shift},
]

\draw[treatment line] ([yshift=-2pt] axis cs:4.07925, 0) |- (axis cs:2.995916666666667, -2.5)
  node[treatment label, anchor=east] {\textbf{KDEy}};
\draw[treatment line] ([yshift=-2pt] axis cs:4.51225, 0) |- (axis cs:2.995916666666667, -3.5)
  node[treatment label, anchor=east] {\textbf{PACC}};
\draw[treatment line] ([yshift=-2pt] axis cs:4.603875, 0) |- (axis cs:2.995916666666667, -4.5)
  node[treatment label, anchor=east] {\textbf{EMQ$^{\text{BCTS}}$}};
\draw[treatment line] ([yshift=-2pt] axis cs:4.847875, 0) |- (axis cs:2.995916666666667, -5.5)
  node[treatment label, anchor=east] {\textbf{EMQ}};
\draw[treatment line] ([yshift=-2pt] axis cs:6.11025, 0) |- (axis cs:2.995916666666667, -6.5)
  node[treatment label, anchor=east] {DoC$_{\alpha 2 \rho}$};
\draw[treatment line] ([yshift=-2pt] axis cs:6.893375, 0) |- (axis cs:2.995916666666667, -7.5)
  node[treatment label, anchor=east] {LEAP$_{\alpha 2 \rho}$};
\draw[treatment line] ([yshift=-2pt] axis cs:6.97925, 0) |- (axis cs:10.802833333333334, -8.0)
  node[treatment label, anchor=west] {EMQ$^{\text{LasCal}}$};
\draw[treatment line] ([yshift=-2pt] axis cs:7.6, 0) |- (axis cs:10.802833333333334, -7.0)
  node[treatment label, anchor=west] {ATC$_{\alpha 2 \rho}$};
\draw[treatment line] ([yshift=-2pt] axis cs:8.13675, 0) |- (axis cs:10.802833333333334, -6.0)
  node[treatment label, anchor=west] {LasCal$_{\zeta 2 \rho}$};
\draw[treatment line] ([yshift=-2pt] axis cs:8.959625, 0) |- (axis cs:10.802833333333334, -5.0)
  node[treatment label, anchor=west] {CPCS$_{\zeta 2 \rho}$};
\draw[treatment line] ([yshift=-2pt] axis cs:9.04975, 0) |- (axis cs:10.802833333333334, -4.0)
  node[treatment label, anchor=west] {CC};
\draw[treatment line] ([yshift=-2pt] axis cs:9.50825, 0) |- (axis cs:10.802833333333334, -3.0)
  node[treatment label, anchor=west] {Head2Tail$_{\zeta 2 \rho}$};
\draw[treatment line] ([yshift=-2pt] axis cs:9.7195, 0) |- (axis cs:10.802833333333334, -2.0)
  node[treatment label, anchor=west] {TransCal$_{\zeta 2 \rho}$};
\draw[group line] (axis cs:8.959625, -2.6666666666666665) -- (axis cs:9.04975, -2.6666666666666665);

\end{axis}
\end{tikzpicture}
    }
    \caption{CD-diagrams for CS experiments (left panel) and LS experiments (right panel). Reference quantification methods are highlighted in boldtype.}
    \label{fig:quant:cddiagram}
\end{figure}

\clearpage

\subsubsection{Classifier Accuracy Prediction Experiments}
\label{sec:exp:cap}

We now turn to discussing the results we have obtained for classifier accuracy prediction. The evaluation measure we adopt is the Absolute estimation Error (AE) \cite{Garg:2022qv}, given by:
\begin{equation}
    \text{AE}=|\text{Acc}-\hat{\text{Acc}}|
\end{equation}
\noindent between the true accuracy ($\text{Acc}$) and the estimated accuracy ($\hat{\text{Acc}}$). As our measure of classifier accuracy, we consider vanilla accuracy, i.e., the proportion of correctly classified instances. 
We report mean scores of AE averaged across all test samples generated for each dataset.
%

We confront our reference methods ATC, DoC, and LEAP against direct adaptations of the reference methods from quantification and calibration. The IID baseline (Naive) is a method that, regardless of the test set, reports the accuracy estimated in the validation set. In CS experiments, we additionally consider LEAP$^{\text{PCC}}$ as a reference method, and PCC$_{\quant2\acc}$ as an adaptation method. In this case, we also experiment with DMCal$_{\calib2\acc}$, an adaptation of the newly proposed DMcal method which in the experiments of Section~\ref{sec:exp:calibration} showed promising performance for calibration.

\begin{table}[b!]
    \caption{Classifier accuracy prediction under CS in terms of AE}
    \vspace{0.5cm}
    \centering
    \resizebox{\textwidth}{!}{%
    \begin{tabular}{cc|c|cccc|ccccccccc} \toprule
\multicolumn{2}{c}{} & \multicolumn{1}{c|}{Baselines} & \multicolumn{4}{c|}{Reference} & \multicolumn{9}{c}{Adapted Methods} \\
\multicolumn{2}{c}{} & \begin{sideways}Naive\;\end{sideways} & \begin{sideways}ATC\;\end{sideways} & \begin{sideways}DoC\;\end{sideways} & \begin{sideways}LEAP\;\end{sideways} & \begin{sideways}LEAP$^{\text{PCC}}$\;\end{sideways} & \begin{sideways}PCC$_{\rho 2 \alpha}$\;\end{sideways} & \begin{sideways}PACC$_{\rho 2 \alpha}$\;\end{sideways} & \begin{sideways}EMQ$_{\rho 2 \alpha}$\;\end{sideways} & \begin{sideways}EMQ$_{\rho 2 \alpha}^{\text{BCTS}}$\;\end{sideways} & \begin{sideways}KDEy$_{\rho 2 \alpha}$\;\end{sideways} & \begin{sideways}CPCS$_{\zeta 2 \alpha}$\;\end{sideways} & \begin{sideways}TransCal$_{\zeta 2 \alpha}$\;\end{sideways} & \begin{sideways}LasCal$_{\zeta 2 \alpha}$\;\end{sideways} & \begin{sideways}DMCal$_{\zeta 2 \alpha}$\;\end{sideways} \\\midrule
\multirow{6}{*}{\begin{sideways}BERT\;\end{sideways}} & \texttt{imdb}$\rightarrow$\texttt{rt} & $0.044$\cellcolor{green!38} & $\textbf{0.033}$\cellcolor{green!39} & $\textbf{0.032}$\cellcolor{green!40} & $0.037$\cellcolor{green!39} & $0.042$\cellcolor{green!38} & $0.053$\cellcolor{green!37} & $0.129$\cellcolor{green!28} & $0.722$\cellcolor{red!40} & $0.128$\cellcolor{green!28} & $0.157$\cellcolor{green!25} & $0.117$\cellcolor{green!30} & $0.130$\cellcolor{green!28} & $\textbf{0.036}$\cellcolor{green!39} & $0.120$\cellcolor{green!29} \\
 & \texttt{imdb}$\rightarrow$\texttt{yelp} & $0.026$\cellcolor{green!39} & $0.027$\cellcolor{green!39} & $\textbf{0.022}$\cellcolor{green!40} & $0.027$\cellcolor{green!39} & $0.041$\cellcolor{green!37} & $0.068$\cellcolor{green!34} & $0.071$\cellcolor{green!34} & $0.722$\cellcolor{red!40} & $0.073$\cellcolor{green!34} & $0.084$\cellcolor{green!32} & $0.087$\cellcolor{green!32} & $0.150$\cellcolor{green!25} & $0.043$\cellcolor{green!37} & $0.059$\cellcolor{green!35} \\
 & \texttt{rt}$\rightarrow$\texttt{imdb} & $\textbf{0.024}$\cellcolor{green!39} & $0.032$\cellcolor{green!38} & $\textbf{0.024}$\cellcolor{green!40} & $\textbf{0.026}$\cellcolor{green!39} & $0.060$\cellcolor{green!35} & $0.091$\cellcolor{green!32} & $0.082$\cellcolor{green!33} & $0.696$\cellcolor{red!40} & $0.079$\cellcolor{green!33} & $0.124$\cellcolor{green!28} & $0.118$\cellcolor{green!28} & $0.142$\cellcolor{green!25} & $0.038$\cellcolor{green!38} & $0.105$\cellcolor{green!30} \\
 & \texttt{rt}$\rightarrow$\texttt{yelp} & $0.024$\cellcolor{green!39} & $0.031$\cellcolor{green!38} & $\textbf{0.021}$\cellcolor{green!40} & $0.031$\cellcolor{green!38} & $0.039$\cellcolor{green!37} & $0.089$\cellcolor{green!31} & $0.096$\cellcolor{green!31} & $0.689$\cellcolor{red!40} & $0.095$\cellcolor{green!31} & $0.104$\cellcolor{green!30} & $0.106$\cellcolor{green!29} & $0.137$\cellcolor{green!26} & $0.051$\cellcolor{green!36} & $0.149$\cellcolor{green!24} \\
 & \texttt{yelp}$\rightarrow$\texttt{imdb} & $0.076$\cellcolor{green!25} & $0.053$\cellcolor{green!33} & $0.058$\cellcolor{green!31} & $0.068$\cellcolor{green!28} & $0.113$\cellcolor{green!12} & $\textbf{0.039}$\cellcolor{green!38} & $\textbf{0.041}$\cellcolor{green!37} & $0.261$\cellcolor{red!40} & $0.045$\cellcolor{green!36} & $0.048$\cellcolor{green!35} & $0.105$\cellcolor{green!15} & $0.090$\cellcolor{green!20} & $0.060$\cellcolor{green!31} & $\textbf{0.035}$\cellcolor{green!40} \\
 & \texttt{yelp}$\rightarrow$\texttt{rt} & $0.094$\cellcolor{green!25} & $0.058$\cellcolor{green!35} & $0.066$\cellcolor{green!33} & $0.081$\cellcolor{green!29} & $0.128$\cellcolor{green!16} & $\textbf{0.044}$\cellcolor{green!38} & $0.067$\cellcolor{green!32} & $0.339$\cellcolor{red!40} & $0.074$\cellcolor{green!30} & $0.060$\cellcolor{green!34} & $0.109$\cellcolor{green!21} & $0.084$\cellcolor{green!28} & $0.070$\cellcolor{green!31} & $\textbf{0.040}$\cellcolor{green!40} \\\midrule
\multirow{6}{*}{\begin{sideways}DistilBERT\;\end{sideways}} & \texttt{imdb}$\rightarrow$\texttt{rt} & $0.047$\cellcolor{green!17} & $\textbf{0.038}$\cellcolor{green!38} & $0.038$\cellcolor{green!38} & $0.048$\cellcolor{green!13} & $\textbf{0.038}$\cellcolor{green!38} & $\textbf{0.037}$\cellcolor{green!40} & $\textbf{0.044}$\cellcolor{green!23} & $\textbf{0.047}$\cellcolor{green!17} & $\textbf{0.045}$\cellcolor{green!23} & $0.055$\cellcolor{red!2} & $0.066$\cellcolor{red!27} & $0.071$\cellcolor{red!40} & $0.063$\cellcolor{red!19} & $\textbf{0.040}$\cellcolor{green!34} \\
 & \texttt{imdb}$\rightarrow$\texttt{yelp} & $\textbf{0.017}$\cellcolor{green!39} & $0.020$\cellcolor{green!36} & $\textbf{0.016}$\cellcolor{green!40} & $\textbf{0.018}$\cellcolor{green!38} & $0.020$\cellcolor{green!36} & $0.017$\cellcolor{green!39} & $0.041$\cellcolor{green!18} & $0.043$\cellcolor{green!17} & $0.041$\cellcolor{green!19} & $0.054$\cellcolor{green!7} & $0.046$\cellcolor{green!14} & $0.111$\cellcolor{red!40} & $0.023$\cellcolor{green!34} & $0.038$\cellcolor{green!21} \\
 & \texttt{rt}$\rightarrow$\texttt{imdb} & $\textbf{0.026}$\cellcolor{green!40} & $0.038$\cellcolor{green!34} & $\textbf{0.027}$\cellcolor{green!39} & $0.032$\cellcolor{green!37} & $0.036$\cellcolor{green!35} & $\textbf{0.030}$\cellcolor{green!38} & $0.125$\cellcolor{red!5} & $0.200$\cellcolor{red!40} & $0.123$\cellcolor{red!4} & $0.122$\cellcolor{red!4} & $0.054$\cellcolor{green!27} & $0.124$\cellcolor{red!5} & $0.037$\cellcolor{green!35} & $0.092$\cellcolor{green!9} \\
 & \texttt{rt}$\rightarrow$\texttt{yelp} & $0.025$\cellcolor{green!35} & $0.029$\cellcolor{green!31} & $0.029$\cellcolor{green!31} & $0.030$\cellcolor{green!29} & $0.051$\cellcolor{green!8} & $\textbf{0.020}$\cellcolor{green!40} & $0.086$\cellcolor{red!27} & $0.064$\cellcolor{red!5} & $0.081$\cellcolor{red!22} & $0.096$\cellcolor{red!38} & $0.071$\cellcolor{red!11} & $0.098$\cellcolor{red!40} & $0.048$\cellcolor{green!11} & $0.080$\cellcolor{red!21} \\
 & \texttt{yelp}$\rightarrow$\texttt{imdb} & $0.070$\cellcolor{green!9} & $\textbf{0.044}$\cellcolor{green!32} & $\textbf{0.039}$\cellcolor{green!37} & $0.066$\cellcolor{green!12} & $0.089$\cellcolor{red!8} & $\textbf{0.039}$\cellcolor{green!37} & $0.069$\cellcolor{green!9} & $0.123$\cellcolor{red!40} & $0.078$\cellcolor{green!1} & $0.049$\cellcolor{green!28} & $0.060$\cellcolor{green!18} & $0.067$\cellcolor{green!11} & $0.058$\cellcolor{green!19} & $\textbf{0.036}$\cellcolor{green!40} \\
 & \texttt{yelp}$\rightarrow$\texttt{rt} & $0.099$\cellcolor{red!4} & $0.070$\cellcolor{green!17} & $0.059$\cellcolor{green!25} & $0.097$\cellcolor{red!2} & $0.114$\cellcolor{red!15} & $0.060$\cellcolor{green!24} & $0.073$\cellcolor{green!15} & $0.146$\cellcolor{red!40} & $0.094$\cellcolor{red!0} & $0.057$\cellcolor{green!26} & $0.075$\cellcolor{green!13} & $0.052$\cellcolor{green!30} & $0.086$\cellcolor{green!5} & $\textbf{0.040}$\cellcolor{green!40} \\\midrule
\multirow{6}{*}{\begin{sideways}RoBERTa\;\end{sideways}} & \texttt{imdb}$\rightarrow$\texttt{rt} & $0.081$\cellcolor{green!15} & $\textbf{0.039}$\cellcolor{green!40} & $0.056$\cellcolor{green!30} & $0.091$\cellcolor{green!9} & $\textbf{0.043}$\cellcolor{green!37} & $0.045$\cellcolor{green!36} & $0.095$\cellcolor{green!7} & $0.177$\cellcolor{red!40} & $0.112$\cellcolor{red!2} & $0.082$\cellcolor{green!14} & $0.089$\cellcolor{green!10} & $0.066$\cellcolor{green!24} & $0.077$\cellcolor{green!17} & $0.068$\cellcolor{green!23} \\
 & \texttt{imdb}$\rightarrow$\texttt{yelp} & $\textbf{0.020}$\cellcolor{green!38} & $\textbf{0.021}$\cellcolor{green!37} & $\textbf{0.019}$\cellcolor{green!39} & $\textbf{0.021}$\cellcolor{green!37} & $\textbf{0.022}$\cellcolor{green!37} & $0.027$\cellcolor{green!33} & $0.038$\cellcolor{green!25} & $0.093$\cellcolor{red!13} & $0.041$\cellcolor{green!23} & $0.049$\cellcolor{green!17} & $0.059$\cellcolor{green!10} & $0.130$\cellcolor{red!40} & $\textbf{0.018}$\cellcolor{green!40} & $0.034$\cellcolor{green!28} \\
 & \texttt{rt}$\rightarrow$\texttt{imdb} & $\textbf{0.025}$\cellcolor{green!39} & $0.054$\cellcolor{green!22} & $0.041$\cellcolor{green!30} & $0.028$\cellcolor{green!37} & $0.063$\cellcolor{green!17} & $\textbf{0.024}$\cellcolor{green!40} & $0.159$\cellcolor{red!39} & $0.139$\cellcolor{red!27} & $0.157$\cellcolor{red!38} & $0.160$\cellcolor{red!40} & $0.064$\cellcolor{green!16} & $0.093$\cellcolor{red!0} & $0.042$\cellcolor{green!29} & $0.116$\cellcolor{red!14} \\
 & \texttt{rt}$\rightarrow$\texttt{yelp} & $\textbf{0.026}$\cellcolor{green!40} & $0.045$\cellcolor{green!27} & $0.031$\cellcolor{green!36} & $\textbf{0.029}$\cellcolor{green!38} & $0.052$\cellcolor{green!23} & $\textbf{0.026}$\cellcolor{green!39} & $0.137$\cellcolor{red!32} & $0.148$\cellcolor{red!40} & $0.136$\cellcolor{red!32} & $0.140$\cellcolor{red!35} & $0.061$\cellcolor{green!16} & $0.108$\cellcolor{red!13} & $0.037$\cellcolor{green!32} & $0.101$\cellcolor{red!9} \\
 & \texttt{yelp}$\rightarrow$\texttt{imdb} & $0.054$\cellcolor{green!29} & $0.026$\cellcolor{green!37} & $0.034$\cellcolor{green!35} & $0.044$\cellcolor{green!32} & $0.068$\cellcolor{green!25} & $\textbf{0.019}$\cellcolor{green!40} & $0.083$\cellcolor{green!21} & $0.289$\cellcolor{red!40} & $0.096$\cellcolor{green!17} & $0.055$\cellcolor{green!29} & $0.028$\cellcolor{green!37} & $0.093$\cellcolor{green!17} & $0.032$\cellcolor{green!36} & $0.049$\cellcolor{green!31} \\
 & \texttt{yelp}$\rightarrow$\texttt{rt} & $0.113$\cellcolor{green!23} & $\textbf{0.026}$\cellcolor{green!39} & $0.057$\cellcolor{green!33} & $0.099$\cellcolor{green!26} & $0.111$\cellcolor{green!24} & $\textbf{0.024}$\cellcolor{green!40} & $0.254$\cellcolor{red!2} & $0.459$\cellcolor{red!40} & $0.304$\cellcolor{red!11} & $0.253$\cellcolor{red!2} & $0.053$\cellcolor{green!34} & $0.081$\cellcolor{green!29} & $0.049$\cellcolor{green!35} & $0.161$\cellcolor{green!14} \\\bottomrule
\multirow{4}{*}{\begin{sideways}Wins\;\end{sideways}} &  $\Pr(M\succ {1}R)$ &$\dag 87.22\%$ & --- & --- & --- & --- & $75.67\%$ & $42.67\%$ & $17.89\%$ & $42.00\%$ & $42.78\%$ & $54.61\%$ & $30.39\%$ & $74.17\%$ & $47.94\%$ \\
 &  $\Pr(M\succ {2}R)$ &$53.06\%$ & --- & --- & --- & --- & $64.17\%$ & $30.89\%$ & $11.83\%$ & $29.00\%$ & $31.83\%$ & $39.00\%$ & $23.61\%$ & $53.28\%$ & $36.56\%$ \\
 &  $\Pr(M\succ {3}R)$ &$31.56\%$ & --- & --- & --- & --- & $52.67\%$ & $23.06\%$ & $9.50\%$ & $21.67\%$ & $24.72\%$ & $27.83\%$ & $17.50\%$ & $32.06\%$ & $28.83\%$ \\
 &  $\Pr(M\succ {4}R)$ &$14.83\%$ & --- & --- & --- & --- & $31.89\%$ & $16.17\%$ & $5.78\%$ & $14.56\%$ & $17.11\%$ & $13.89\%$ & $12.00\%$ & $15.28\%$ & $22.11\%$ \\\bottomrule
  & Ave Rank &6.18 & \textbf{5.03} & \textbf{4.70} & \textbf{6.13} & 7.35 & \textbf{5.15} & 8.70 & 12.08 & 9.02 & 8.76 & 8.02 & 10.01 & 6.40 & 7.46 \\\bottomrule
\end{tabular}
    }%
    \label{tab:cap:covshift}
\end{table}

Table~\ref{tab:cap:covshift} reports the experimental results we have obtained for CS. The reference methods tend to dominate the experimental comparison, with DoC performing very well once again. Somehow surprisingly, though, LEAP$^{\text{PCC}}$ did not perform better than LEAP notwithstanding the fact that PCC should be more reliable than KDEy as an underlying quantifier under CS. Conversely, the adaptation PCC$_{\quant2\acc}$ is one of the best performing systems of the experimental comparison. EMQ$_{\quant2\acc}$ behaves unstably, but the pre-calibration round of EMQ$_{\quant2\acc}^{\text{BCTS}}$ seems to partially compensate for it. TransCal$_{\calib2\acc}$ is one of the worst performing systems of the lot. This was not to be expected given that, under CS, it showed excellent quantification performance. However, this partially corroborates our hypothesis that the calibrated posteriors of TransCal$_{\calib2\acc}$ are not very informative from a finer-grained perspective, i.e., that when we separately inspect the posteriors assigned to the predicted positives and the predicted negatives (as is the rationale behind $\calib2\acc$), these do not locally reflect the model's uncertainty well. It is also surprising that, among all the calibration adaptations, LasCal$_{\calib 2 \acc}$, a method addressing LS, works best.

Table~\ref{tab:cap:labelshift} reports the AE results under LS. Also in this case, the reference methods achieve superior performance. The best performing system is LEAP, closely followed by DoC. The adaptation methods show erratic behavior, sometimes achieving very good results and sometimes failing loudly.

\begin{table}[t!]
    \caption{Classifier accuracy prediction under label shift in terms of AE}
    \vspace{0.5cm}
    \centering
    \resizebox{.9\textwidth}{!}{%
    \begin{tabular}{cc|c|ccc|ccccccc} \toprule
\multicolumn{2}{c}{} & \multicolumn{1}{c|}{Baselines} & \multicolumn{3}{c|}{Reference} & \multicolumn{7}{c}{Adapted Methods} \\
\multicolumn{2}{c}{} & \begin{sideways}Naive\;\end{sideways} & \begin{sideways}ATC\;\end{sideways} & \begin{sideways}DoC\;\end{sideways} & \begin{sideways}LEAP\;\end{sideways} & \begin{sideways}PACC$_{\rho 2 \alpha}$\;\end{sideways} & \begin{sideways}EMQ$_{\rho 2 \alpha}$\;\end{sideways} & \begin{sideways}KDEy$_{\rho 2 \alpha}$\;\end{sideways} & \begin{sideways}CPCS$_{\zeta 2 \alpha}$\;\end{sideways} & \begin{sideways}TransCal$_{\zeta 2 \alpha}$\;\end{sideways} & \begin{sideways}LasCal$_{\zeta 2 \alpha}$\;\end{sideways} & \begin{sideways}DMCal$_{\zeta 2 \alpha}$\;\end{sideways} \\\midrule
\multirow{10}{*}{\begin{sideways}Logistic Regression\;\end{sideways}} & \texttt{cmc.3} & $0.178$\cellcolor{red!40} & $0.149$\cellcolor{red!17} & $\textbf{0.074}$\cellcolor{green!40} & $\textbf{0.076}$\cellcolor{green!38} & $\textbf{0.104}$\cellcolor{green!17} & $0.147$\cellcolor{red!16} & $0.141$\cellcolor{red!11} & $0.156$\cellcolor{red!23} & $0.168$\cellcolor{red!31} & $0.169$\cellcolor{red!33} & $0.095$\cellcolor{green!23} \\
 & \texttt{yeast} & $0.130$\cellcolor{red!40} & $0.088$\cellcolor{green!0} & $\textbf{0.047}$\cellcolor{green!40} & $\textbf{0.048}$\cellcolor{green!39} & $\textbf{0.051}$\cellcolor{green!36} & $0.057$\cellcolor{green!29} & $0.067$\cellcolor{green!20} & $0.103$\cellcolor{red!14} & $0.095$\cellcolor{red!5} & $0.114$\cellcolor{red!24} & $0.065$\cellcolor{green!22} \\
 & \texttt{semeion} & $0.119$\cellcolor{red!40} & $0.088$\cellcolor{red!14} & $0.035$\cellcolor{green!29} & $0.041$\cellcolor{green!24} & $\textbf{0.023}$\cellcolor{green!40} & $0.028$\cellcolor{green!35} & $0.031$\cellcolor{green!33} & $0.110$\cellcolor{red!32} & $0.074$\cellcolor{red!2} & $0.090$\cellcolor{red!16} & $0.040$\cellcolor{green!25} \\
 & \texttt{wine-q-red} & $\textbf{0.021}$\cellcolor{green!40} & $0.060$\cellcolor{red!0} & $\textbf{0.023}$\cellcolor{green!38} & $\textbf{0.022}$\cellcolor{green!38} & $0.093$\cellcolor{red!33} & $0.066$\cellcolor{red!5} & $0.089$\cellcolor{red!29} & $0.040$\cellcolor{green!21} & $0.087$\cellcolor{red!27} & $0.028$\cellcolor{green!33} & $0.099$\cellcolor{red!40} \\
 & \texttt{ctg.1} & $0.041$\cellcolor{green!11} & $\textbf{0.022}$\cellcolor{green!34} & $0.023$\cellcolor{green!33} & $\textbf{0.018}$\cellcolor{green!39} & $0.026$\cellcolor{green!29} & $0.040$\cellcolor{green!11} & $0.030$\cellcolor{green!25} & $\textbf{0.018}$\cellcolor{green!40} & $0.081$\cellcolor{red!40} & $\textbf{0.020}$\cellcolor{green!37} & $0.027$\cellcolor{green!28} \\
 & \texttt{ctg.2} & $0.097$\cellcolor{red!16} & $0.072$\cellcolor{green!4} & $0.093$\cellcolor{red!13} & $0.125$\cellcolor{red!40} & $\textbf{0.032}$\cellcolor{green!37} & $0.041$\cellcolor{green!30} & $\textbf{0.030}$\cellcolor{green!39} & $0.057$\cellcolor{green!17} & $0.038$\cellcolor{green!32} & $\textbf{0.030}$\cellcolor{green!40} & $0.047$\cellcolor{green!24} \\
 & \texttt{ctg.3} & $0.105$\cellcolor{red!20} & $\textbf{0.021}$\cellcolor{green!40} & $0.041$\cellcolor{green!25} & $0.028$\cellcolor{green!34} & $0.057$\cellcolor{green!14} & $0.132$\cellcolor{red!40} & $0.042$\cellcolor{green!25} & $0.078$\cellcolor{red!0} & $0.045$\cellcolor{green!23} & $0.058$\cellcolor{green!13} & $0.028$\cellcolor{green!34} \\
 & \texttt{spambase} & $\textbf{0.018}$\cellcolor{green!34} & $0.017$\cellcolor{green!35} & $\textbf{0.014}$\cellcolor{green!40} & $0.017$\cellcolor{green!34} & $0.037$\cellcolor{green!2} & $0.063$\cellcolor{red!40} & $0.027$\cellcolor{green!18} & $0.034$\cellcolor{green!6} & $0.057$\cellcolor{red!30} & $0.023$\cellcolor{green!25} & $0.026$\cellcolor{green!19} \\
 & \texttt{wine-q-white} & $0.120$\cellcolor{red!31} & $0.094$\cellcolor{red!8} & $\textbf{0.045}$\cellcolor{green!34} & $\textbf{0.039}$\cellcolor{green!40} & $0.066$\cellcolor{green!16} & $0.110$\cellcolor{red!22} & $0.075$\cellcolor{green!8} & $0.082$\cellcolor{green!2} & $0.080$\cellcolor{green!3} & $0.130$\cellcolor{red!40} & $0.092$\cellcolor{red!6} \\
 & \texttt{pageblocks.5} & $0.326$\cellcolor{red!40} & $0.122$\cellcolor{green!15} & $0.049$\cellcolor{green!34} & $\textbf{0.031}$\cellcolor{green!40} & $0.053$\cellcolor{green!33} & $0.046$\cellcolor{green!35} & $0.055$\cellcolor{green!33} & $0.300$\cellcolor{red!33} & $0.232$\cellcolor{red!14} & $0.248$\cellcolor{red!18} & $0.080$\cellcolor{green!26} \\\midrule
\multirow{10}{*}{\begin{sideways}Naïve Bayes\;\end{sideways}} & \texttt{cmc.3} & $0.092$\cellcolor{green!21} & $0.059$\cellcolor{green!33} & $\textbf{0.042}$\cellcolor{green!40} & $\textbf{0.042}$\cellcolor{green!39} & $0.113$\cellcolor{green!13} & $0.257$\cellcolor{red!40} & $0.097$\cellcolor{green!19} & $0.064$\cellcolor{green!31} & $0.086$\cellcolor{green!23} & $0.091$\cellcolor{green!21} & $0.142$\cellcolor{green!2} \\
 & \texttt{yeast} & $0.273$\cellcolor{green!14} & $\textbf{0.159}$\cellcolor{green!40} & $0.226$\cellcolor{green!25} & $0.325$\cellcolor{green!3} & $0.241$\cellcolor{green!21} & $0.523$\cellcolor{red!40} & $0.513$\cellcolor{red!37} & $0.238$\cellcolor{green!22} & $0.370$\cellcolor{red!6} & $0.242$\cellcolor{green!21} & $0.522$\cellcolor{red!39} \\
 & \texttt{semeion} & $\textbf{0.027}$\cellcolor{green!40} & $0.185$\cellcolor{red!3} & $0.051$\cellcolor{green!33} & $0.040$\cellcolor{green!36} & $0.314$\cellcolor{red!40} & $0.201$\cellcolor{red!8} & $0.306$\cellcolor{red!37} & $0.136$\cellcolor{green!9} & $0.199$\cellcolor{red!8} & $0.096$\cellcolor{green!20} & $0.201$\cellcolor{red!8} \\
 & \texttt{wine-q-red} & $0.035$\cellcolor{green!36} & $\textbf{0.028}$\cellcolor{green!39} & $0.031$\cellcolor{green!37} & $0.049$\cellcolor{green!32} & $0.074$\cellcolor{green!23} & $0.264$\cellcolor{red!40} & $0.075$\cellcolor{green!23} & $\textbf{0.025}$\cellcolor{green!40} & $0.103$\cellcolor{green!14} & $\textbf{0.027}$\cellcolor{green!39} & $0.107$\cellcolor{green!12} \\
 & \texttt{ctg.1} & $0.062$\cellcolor{green!8} & $0.040$\cellcolor{green!23} & $0.027$\cellcolor{green!32} & $\textbf{0.015}$\cellcolor{green!40} & $0.038$\cellcolor{green!25} & $0.137$\cellcolor{red!40} & $0.061$\cellcolor{green!10} & $\textbf{0.020}$\cellcolor{green!37} & $0.038$\cellcolor{green!25} & $0.030$\cellcolor{green!30} & $0.106$\cellcolor{red!19} \\
 & \texttt{ctg.2} & $0.068$\cellcolor{green!19} & $\textbf{0.023}$\cellcolor{green!40} & $0.036$\cellcolor{green!34} & $0.029$\cellcolor{green!37} & $0.115$\cellcolor{red!1} & $0.200$\cellcolor{red!40} & $0.143$\cellcolor{red!14} & $0.130$\cellcolor{red!8} & $0.149$\cellcolor{red!17} & $0.130$\cellcolor{red!8} & $0.180$\cellcolor{red!31} \\
 & \texttt{ctg.3} & $0.119$\cellcolor{green!0} & $\textbf{0.032}$\cellcolor{green!40} & $0.136$\cellcolor{red!7} & $0.050$\cellcolor{green!31} & $0.208$\cellcolor{red!40} & $0.167$\cellcolor{red!21} & $0.069$\cellcolor{green!23} & $0.072$\cellcolor{green!21} & $0.147$\cellcolor{red!11} & $\textbf{0.035}$\cellcolor{green!38} & $0.149$\cellcolor{red!12} \\
 & \texttt{spambase} & $0.055$\cellcolor{green!26} & $0.055$\cellcolor{green!26} & $0.051$\cellcolor{green!28} & $\textbf{0.019}$\cellcolor{green!40} & $0.163$\cellcolor{red!13} & $0.157$\cellcolor{red!11} & $0.235$\cellcolor{red!40} & $0.061$\cellcolor{green!24} & $0.146$\cellcolor{red!6} & $0.148$\cellcolor{red!7} & $0.152$\cellcolor{red!9} \\
 & \texttt{wine-q-white} & $0.081$\cellcolor{green!27} & $0.040$\cellcolor{green!37} & $0.047$\cellcolor{green!35} & $0.038$\cellcolor{green!37} & $0.121$\cellcolor{green!17} & $0.348$\cellcolor{red!40} & $0.124$\cellcolor{green!16} & $\textbf{0.029}$\cellcolor{green!40} & $0.103$\cellcolor{green!21} & $0.070$\cellcolor{green!29} & $0.172$\cellcolor{green!4} \\
 & \texttt{pageblocks.5} & $0.272$\cellcolor{red!40} & $\textbf{0.025}$\cellcolor{green!40} & $0.067$\cellcolor{green!26} & $0.045$\cellcolor{green!33} & $0.076$\cellcolor{green!23} & $0.173$\cellcolor{red!7} & $0.049$\cellcolor{green!32} & $0.199$\cellcolor{red!16} & $0.229$\cellcolor{red!26} & $0.195$\cellcolor{red!14} & $0.229$\cellcolor{red!26} \\\midrule
\multirow{10}{*}{\begin{sideways}k Nearest Neighbor\;\end{sideways}} & \texttt{cmc.3} & $0.175$\cellcolor{red!14} & $0.225$\cellcolor{red!40} & $\textbf{0.075}$\cellcolor{green!36} & $\textbf{0.069}$\cellcolor{green!40} & $0.171$\cellcolor{red!12} & $\textbf{0.084}$\cellcolor{green!32} & $0.166$\cellcolor{red!9} & $0.147$\cellcolor{red!0} & $0.155$\cellcolor{red!3} & $0.175$\cellcolor{red!14} & $0.123$\cellcolor{green!12} \\
 & \texttt{yeast} & $0.130$\cellcolor{red!12} & $0.102$\cellcolor{green!5} & $\textbf{0.047}$\cellcolor{green!40} & $0.099$\cellcolor{green!7} & $0.061$\cellcolor{green!31} & $0.174$\cellcolor{red!40} & $0.063$\cellcolor{green!30} & $0.114$\cellcolor{red!2} & $0.100$\cellcolor{green!6} & $0.119$\cellcolor{red!5} & $\textbf{0.053}$\cellcolor{green!36} \\
 & \texttt{semeion} & $0.141$\cellcolor{green!6} & $0.080$\cellcolor{green!22} & $0.026$\cellcolor{green!37} & $\textbf{0.015}$\cellcolor{green!40} & $0.123$\cellcolor{green!11} & $0.317$\cellcolor{red!40} & $0.097$\cellcolor{green!18} & $0.086$\cellcolor{green!21} & $0.026$\cellcolor{green!37} & $0.081$\cellcolor{green!22} & $0.089$\cellcolor{green!20} \\
 & \texttt{wine-q-red} & $\textbf{0.029}$\cellcolor{green!39} & $0.146$\cellcolor{red!40} & $\textbf{0.031}$\cellcolor{green!38} & $\textbf{0.028}$\cellcolor{green!40} & $0.112$\cellcolor{red!16} & $0.086$\cellcolor{green!0} & $0.142$\cellcolor{red!37} & $0.046$\cellcolor{green!28} & $0.072$\cellcolor{green!10} & $0.042$\cellcolor{green!30} & $0.104$\cellcolor{red!11} \\
 & \texttt{ctg.1} & $0.072$\cellcolor{red!6} & $0.108$\cellcolor{red!40} & $0.026$\cellcolor{green!35} & $\textbf{0.021}$\cellcolor{green!40} & $\textbf{0.023}$\cellcolor{green!37} & $0.060$\cellcolor{green!4} & $0.028$\cellcolor{green!33} & $0.046$\cellcolor{green!16} & $0.040$\cellcolor{green!22} & $0.055$\cellcolor{green!8} & $0.025$\cellcolor{green!36} \\
 & \texttt{ctg.2} & $0.162$\cellcolor{red!31} & $0.177$\cellcolor{red!40} & $0.036$\cellcolor{green!34} & $0.098$\cellcolor{green!2} & $\textbf{0.027}$\cellcolor{green!40} & $0.045$\cellcolor{green!30} & $0.033$\cellcolor{green!36} & $0.060$\cellcolor{green!22} & $0.064$\cellcolor{green!20} & $0.079$\cellcolor{green!12} & $\textbf{0.030}$\cellcolor{green!38} \\
 & \texttt{ctg.3} & $0.194$\cellcolor{red!40} & $0.119$\cellcolor{red!5} & $0.033$\cellcolor{green!33} & $0.036$\cellcolor{green!32} & $\textbf{0.019}$\cellcolor{green!40} & $0.122$\cellcolor{red!7} & $\textbf{0.020}$\cellcolor{green!39} & $0.160$\cellcolor{red!24} & $0.091$\cellcolor{green!6} & $0.138$\cellcolor{red!14} & $\textbf{0.019}$\cellcolor{green!39} \\
 & \texttt{spambase} & $0.040$\cellcolor{green!10} & $0.070$\cellcolor{red!28} & $\textbf{0.018}$\cellcolor{green!40} & $0.024$\cellcolor{green!31} & $0.030$\cellcolor{green!24} & $0.078$\cellcolor{red!40} & $0.030$\cellcolor{green!23} & $0.043$\cellcolor{green!7} & $0.073$\cellcolor{red!32} & $0.029$\cellcolor{green!25} & $0.039$\cellcolor{green!12} \\
 & \texttt{wine-q-white} & $0.076$\cellcolor{green!12} & $\textbf{0.032}$\cellcolor{green!38} & $\textbf{0.030}$\cellcolor{green!40} & $\textbf{0.031}$\cellcolor{green!39} & $0.134$\cellcolor{red!22} & $0.133$\cellcolor{red!21} & $0.164$\cellcolor{red!40} & $\textbf{0.037}$\cellcolor{green!35} & $0.044$\cellcolor{green!31} & $0.065$\cellcolor{green!18} & $0.098$\cellcolor{red!1} \\
 & \texttt{pageblocks.5} & $0.336$\cellcolor{red!40} & $0.232$\cellcolor{red!13} & $\textbf{0.029}$\cellcolor{green!40} & $0.036$\cellcolor{green!38} & $0.059$\cellcolor{green!32} & $0.107$\cellcolor{green!19} & $0.071$\cellcolor{green!29} & $0.248$\cellcolor{red!17} & $0.218$\cellcolor{red!9} & $0.247$\cellcolor{red!16} & $0.131$\cellcolor{green!13} \\\midrule
\multirow{10}{*}{\begin{sideways}Multi-layer Perceptron\;\end{sideways}} & \texttt{cmc.3} & $0.128$\cellcolor{green!13} & $0.083$\cellcolor{green!26} & $\textbf{0.047}$\cellcolor{green!37} & $\textbf{0.039}$\cellcolor{green!40} & $0.164$\cellcolor{green!2} & $0.303$\cellcolor{red!40} & $0.172$\cellcolor{red!0} & $0.104$\cellcolor{green!20} & $0.114$\cellcolor{green!17} & $0.118$\cellcolor{green!16} & $0.117$\cellcolor{green!16} \\
 & \texttt{yeast} & $0.089$\cellcolor{green!18} & $\textbf{0.051}$\cellcolor{green!40} & $\textbf{0.054}$\cellcolor{green!38} & $\textbf{0.055}$\cellcolor{green!37} & $0.087$\cellcolor{green!19} & $0.192$\cellcolor{red!40} & $0.138$\cellcolor{red!9} & $\textbf{0.057}$\cellcolor{green!36} & $0.063$\cellcolor{green!33} & $0.078$\cellcolor{green!24} & $\textbf{0.069}$\cellcolor{green!29} \\
 & \texttt{semeion} & $0.083$\cellcolor{green!3} & $\textbf{0.015}$\cellcolor{green!40} & $\textbf{0.015}$\cellcolor{green!39} & $\textbf{0.018}$\cellcolor{green!38} & $0.042$\cellcolor{green!25} & $0.165$\cellcolor{red!40} & $0.021$\cellcolor{green!36} & $0.079$\cellcolor{green!6} & $0.036$\cellcolor{green!28} & $0.051$\cellcolor{green!20} & $0.020$\cellcolor{green!37} \\
 & \texttt{wine-q-red} & $0.037$\cellcolor{green!31} & $0.029$\cellcolor{green!35} & $0.035$\cellcolor{green!32} & $\textbf{0.021}$\cellcolor{green!40} & $0.112$\cellcolor{red!10} & $0.164$\cellcolor{red!40} & $0.102$\cellcolor{red!5} & $\textbf{0.021}$\cellcolor{green!39} & $0.086$\cellcolor{green!3} & $0.028$\cellcolor{green!35} & $0.087$\cellcolor{green!3} \\
 & \texttt{ctg.1} & $0.060$\cellcolor{red!40} & $0.034$\cellcolor{green!3} & $0.022$\cellcolor{green!24} & $\textbf{0.013}$\cellcolor{green!40} & $0.022$\cellcolor{green!24} & $0.038$\cellcolor{red!2} & $0.033$\cellcolor{green!5} & $0.026$\cellcolor{green!17} & $0.022$\cellcolor{green!23} & $0.043$\cellcolor{red!10} & $0.058$\cellcolor{red!36} \\
 & \texttt{ctg.2} & $0.116$\cellcolor{red!12} & $\textbf{0.027}$\cellcolor{green!40} & $0.107$\cellcolor{red!6} & $0.073$\cellcolor{green!12} & $0.163$\cellcolor{red!40} & $0.084$\cellcolor{green!6} & $0.129$\cellcolor{red!20} & $0.042$\cellcolor{green!30} & $0.033$\cellcolor{green!36} & $0.066$\cellcolor{green!17} & $0.066$\cellcolor{green!16} \\
 & \texttt{ctg.3} & $0.100$\cellcolor{red!40} & $0.056$\cellcolor{green!7} & $\textbf{0.026}$\cellcolor{green!38} & $\textbf{0.025}$\cellcolor{green!40} & $\textbf{0.030}$\cellcolor{green!35} & $\textbf{0.033}$\cellcolor{green!31} & $0.051$\cellcolor{green!11} & $0.085$\cellcolor{red!24} & $0.047$\cellcolor{green!16} & $0.073$\cellcolor{red!11} & $0.050$\cellcolor{green!13} \\
 & \texttt{spambase} & $0.020$\cellcolor{green!30} & $0.019$\cellcolor{green!33} & $\textbf{0.015}$\cellcolor{green!40} & $\textbf{0.016}$\cellcolor{green!38} & $0.056$\cellcolor{red!40} & $0.045$\cellcolor{red!18} & $0.033$\cellcolor{green!5} & $0.024$\cellcolor{green!22} & $0.021$\cellcolor{green!28} & $0.030$\cellcolor{green!11} & $0.030$\cellcolor{green!11} \\
 & \texttt{wine-q-white} & $0.077$\cellcolor{red!33} & $0.054$\cellcolor{red!2} & $\textbf{0.027}$\cellcolor{green!35} & $\textbf{0.024}$\cellcolor{green!40} & $0.066$\cellcolor{red!17} & $0.059$\cellcolor{red!8} & $0.068$\cellcolor{red!20} & $0.049$\cellcolor{green!5} & $0.040$\cellcolor{green!17} & $0.071$\cellcolor{red!25} & $0.082$\cellcolor{red!40} \\
 & \texttt{pageblocks.5} & $0.225$\cellcolor{red!40} & $0.139$\cellcolor{red!2} & $0.069$\cellcolor{green!27} & $0.053$\cellcolor{green!34} & $0.048$\cellcolor{green!37} & $\textbf{0.042}$\cellcolor{green!39} & $\textbf{0.041}$\cellcolor{green!40} & $0.108$\cellcolor{green!10} & $0.105$\cellcolor{green!12} & $0.141$\cellcolor{red!3} & $0.059$\cellcolor{green!32} \\\bottomrule
\multirow{3}{*}{\begin{sideways}Wins\;\end{sideways}} &  $\Pr(M\succ {1}R)$ &$46.33\%$ & --- & --- & --- & $55.77\%$ & $39.88\%$ & $56.55\%$ & $63.42\%$ & $51.75\%$ & $55.15\%$ & $54.43\%$ \\
 &  $\Pr(M\succ {2}R)$ &$26.97\%$ & --- & --- & --- & $35.48\%$ & $21.40\%$ & $34.83\%$ & $34.27\%$ & $25.32\%$ & $30.90\%$ & $30.80\%$ \\
 &  $\Pr(M\succ {3}R)$ &$12.07\%$ & --- & --- & --- & $19.82\%$ & $11.82\%$ & $19.07\%$ & $16.32\%$ & $12.90\%$ & $15.43\%$ & $16.12\%$ \\\bottomrule
  & Ave Rank &7.15 & \textbf{5.65} & \textbf{4.17} & \textbf{4.05} & 5.97 & 7.91 & 6.12 & 5.79 & 6.58 & 6.20 & 6.42 \\\bottomrule
\end{tabular}
    }%
    \label{tab:cap:labelshift}
\end{table}

Figure~\ref{fig:cap:errbyshift} show the AE as a function of the level of shift for a selection of methods. Most methods improve over the naive baseline, with the exception of LEAP$^{\text{PCC}}$ (left pannel) which becomes progressively unreliable as the level of shift increases. None of the methods, however, are able to compensate for higher levels of CS, suggesting that further research is needed. Concerning the experiments of LS, LEAP and PACC$_{\quant2\acc}$ are very stable even at high levels of prior shift. The rest of the methods, including DoC, perform worse at higher levels of prior shift. 

The CD-diagrams (Figure~\ref{fig:cap:cddiagram}) reveal that, under CS, the top-performing methods DoC, ATC, and PCC$_{\quant2\acc}$ are actually not significantly different from each other, and that LEAP and LasCal$_{\calib2\acc}$ are not better than the Naive baseline. In LS, LEAP stands out over DoC and ATC, which nevertheless perform better than the rest of the methods. Only EMQ$_{\quant2\acc}$ proves inferior than the Naive baseline under prior shift. 

\begin{figure}
    \centering
    \includegraphics[width=0.48\linewidth]{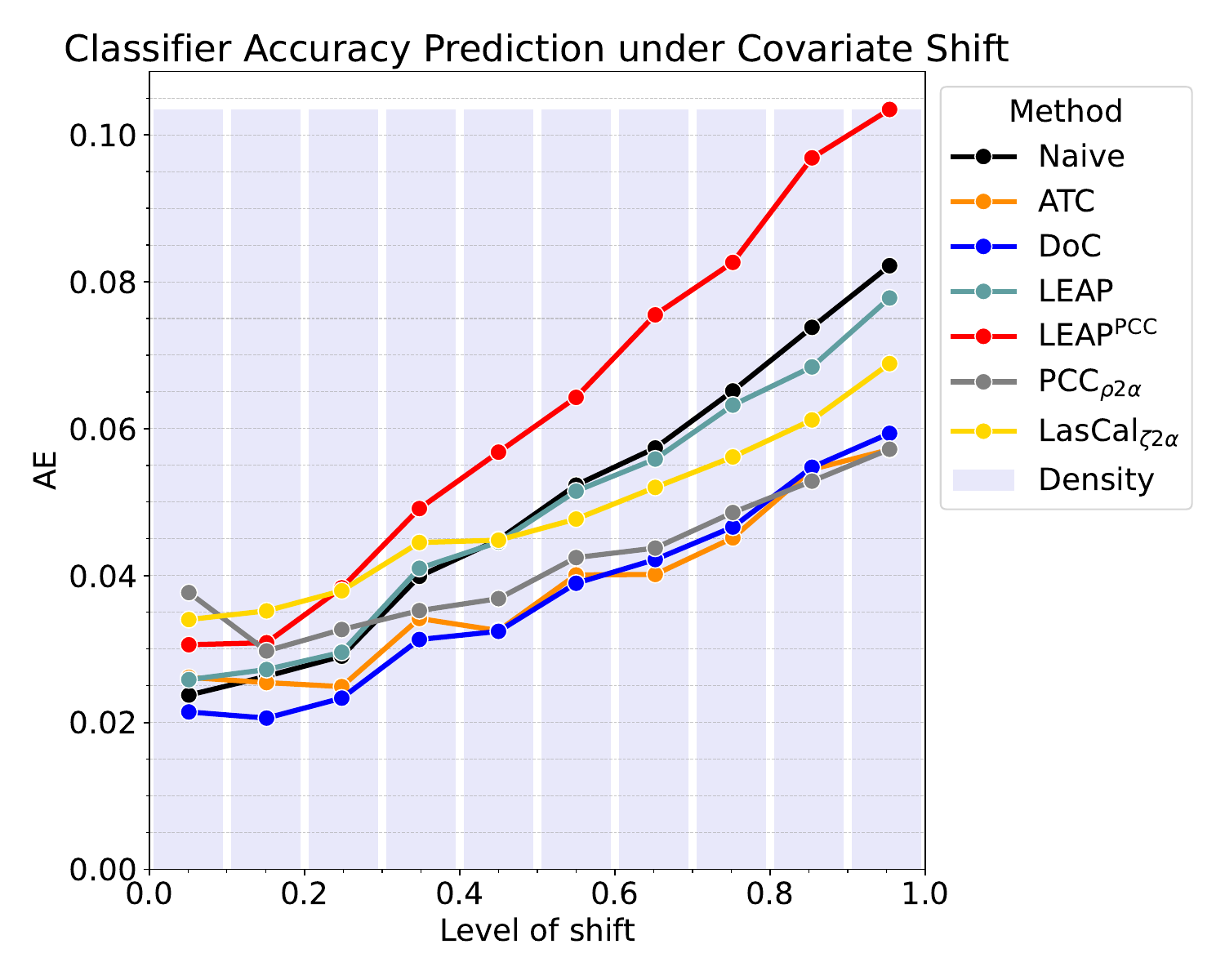}
    \includegraphics[width=0.48\linewidth]{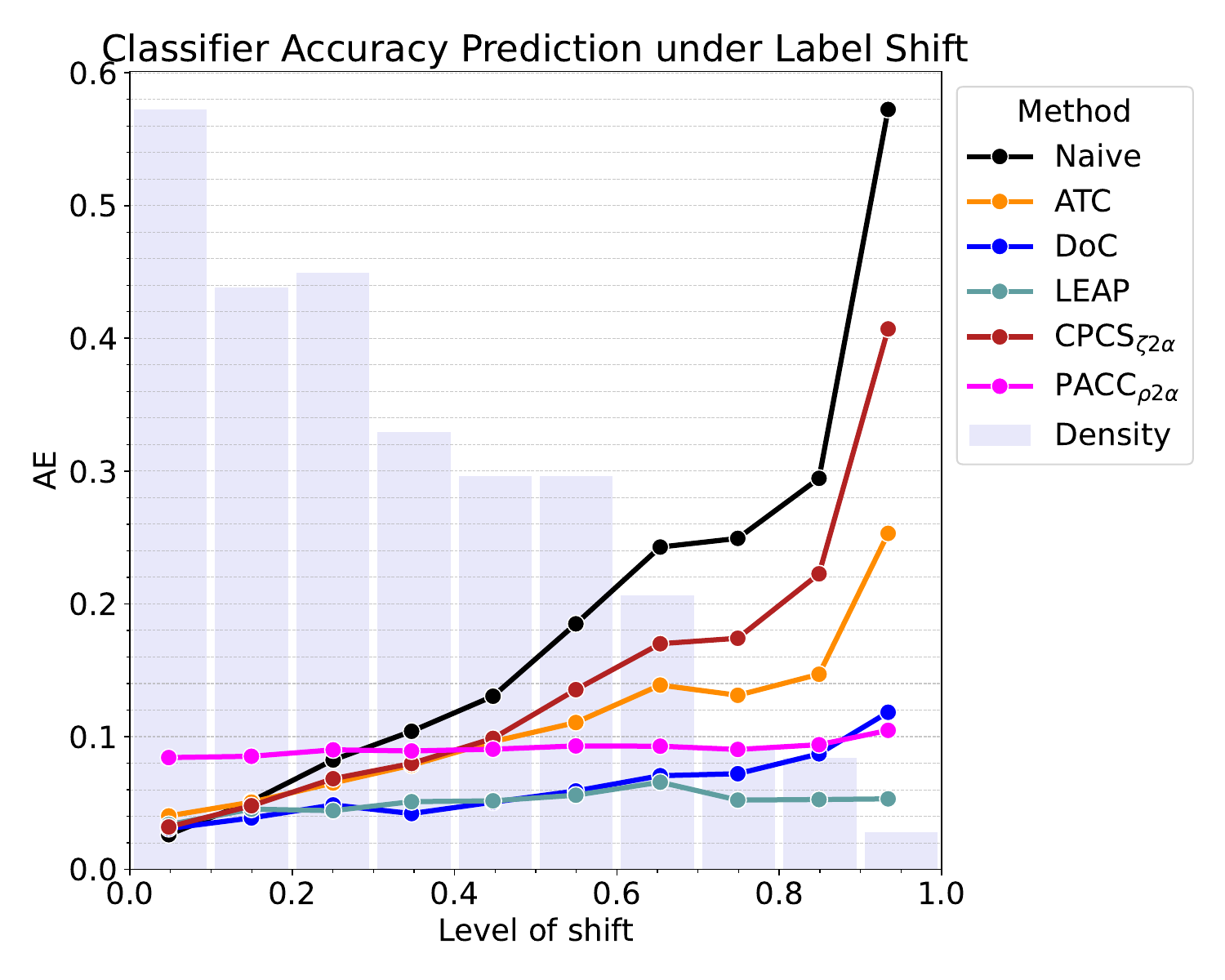}
    \caption{Classifier accuracy prediction error in terms of AE as a function of shift intensity in CS experiments (left panel) and LS (right panel).}
    \label{fig:cap:errbyshift}

    \vspace{1cm}

    \centering
    \resizebox{0.48\textwidth}{!}{%
    \begin{tikzpicture}[
  treatment line/.style={rounded corners=1.5pt, line cap=round, shorten >=1pt},
  treatment label/.style={font=\small},
  group line/.style={ultra thick},
]

\begin{axis}[
  clip={false},
  axis x line={center},
  axis y line={none},
  axis line style={-},
  xmin={1},
  ymax={0},
  scale only axis={true},
  width={\axisdefaultwidth},
  ticklabel style={anchor=south, yshift=1.3*\pgfkeysvalueof{/pgfplots/major tick length}, font=\small},
  every tick/.style={draw=black},
  major tick style={yshift=.5*\pgfkeysvalueof{/pgfplots/major tick length}},
  minor tick style={yshift=.5*\pgfkeysvalueof{/pgfplots/minor tick length}},
  title style={yshift=\baselineskip},
  xmax={14},
  ymin={-8.5},
  height={9\baselineskip},
  title={Classifier Accuracy Prediction under Covariate Shift},
]

\draw[treatment line] ([yshift=-2pt] axis cs:4.7027777777777775, 0) |- (axis cs:3.5361111111111105, -2.0)
  node[treatment label, anchor=east] {\textbf{DoC}};
\draw[treatment line] ([yshift=-2pt] axis cs:5.025, 0) |- (axis cs:3.5361111111111105, -3.0)
  node[treatment label, anchor=east] {\textbf{ATC}};
\draw[treatment line] ([yshift=-2pt] axis cs:5.15, 0) |- (axis cs:3.5361111111111105, -4.0)
  node[treatment label, anchor=east] {PCC$_{\rho 2 \alpha}$};
\draw[treatment line] ([yshift=-2pt] axis cs:6.133333333333334, 0) |- (axis cs:3.5361111111111105, -5.0)
  node[treatment label, anchor=east] {\textbf{LEAP}};
\draw[treatment line] ([yshift=-2pt] axis cs:6.183333333333334, 0) |- (axis cs:3.5361111111111105, -6.0)
  node[treatment label, anchor=east] {Naive};
\draw[treatment line] ([yshift=-2pt] axis cs:6.400555555555556, 0) |- (axis cs:3.5361111111111105, -7.0)
  node[treatment label, anchor=east] {LasCal$_{\zeta 2 \alpha}$};
\draw[treatment line] ([yshift=-2pt] axis cs:7.354444444444445, 0) |- (axis cs:3.5361111111111105, -8.0)
  node[treatment label, anchor=east] {\textbf{LEAP$^{\text{PCC}}$}};
\draw[treatment line] ([yshift=-2pt] axis cs:7.457777777777777, 0) |- (axis cs:13.247777777777777, -8.0)
  node[treatment label, anchor=west] {DMCal$_{\zeta 2 \alpha}$};
\draw[treatment line] ([yshift=-2pt] axis cs:8.022222222222222, 0) |- (axis cs:13.247777777777777, -7.0)
  node[treatment label, anchor=west] {CPCS$_{\zeta 2 \alpha}$};
\draw[treatment line] ([yshift=-2pt] axis cs:8.703611111111112, 0) |- (axis cs:13.247777777777777, -6.0)
  node[treatment label, anchor=west] {PACC$_{\rho 2 \alpha}$};
\draw[treatment line] ([yshift=-2pt] axis cs:8.759166666666667, 0) |- (axis cs:13.247777777777777, -5.0)
  node[treatment label, anchor=west] {KDEy$_{\rho 2 \alpha}$};
\draw[treatment line] ([yshift=-2pt] axis cs:9.015555555555556, 0) |- (axis cs:13.247777777777777, -4.0)
  node[treatment label, anchor=west] {EMQ$_{\rho 2 \alpha}^{\text{BCTS}}$};
\draw[treatment line] ([yshift=-2pt] axis cs:10.011111111111111, 0) |- (axis cs:13.247777777777777, -3.0)
  node[treatment label, anchor=west] {TransCal$_{\zeta 2 \alpha}$};
\draw[treatment line] ([yshift=-2pt] axis cs:12.081111111111111, 0) |- (axis cs:13.247777777777777, -2.0)
  node[treatment label, anchor=west] {EMQ$_{\rho 2 \alpha}$};
\draw[group line] (axis cs:7.457777777777777, -4.666666666666667) -- (axis cs:8.022222222222222, -4.666666666666667);
\draw[group line] (axis cs:6.133333333333334, -3.3333333333333335) -- (axis cs:6.400555555555556, -3.3333333333333335);
\draw[group line] (axis cs:8.703611111111112, -3.3333333333333335) -- (axis cs:8.759166666666667, -3.3333333333333335);
\draw[group line] (axis cs:4.7027777777777775, -1.3333333333333333) -- (axis cs:5.15, -1.3333333333333333);

\end{axis}
\end{tikzpicture}
    }
    \resizebox{0.48\textwidth}{!}{%
    \begin{tikzpicture}[
  treatment line/.style={rounded corners=1.5pt, line cap=round, shorten >=1pt},
  treatment label/.style={font=\small},
  group line/.style={ultra thick},
]

\begin{axis}[
  clip={false},
  axis x line={center},
  axis y line={none},
  axis line style={-},
  xmin={1},
  ymax={0},
  scale only axis={true},
  width={\axisdefaultwidth},
  ticklabel style={anchor=south, yshift=1.3*\pgfkeysvalueof{/pgfplots/major tick length}, font=\small},
  every tick/.style={draw=black},
  major tick style={yshift=.5*\pgfkeysvalueof{/pgfplots/major tick length}},
  minor tick style={yshift=.5*\pgfkeysvalueof{/pgfplots/minor tick length}},
  title style={yshift=\baselineskip},
  xmax={11},
  ymin={-6.5},
  height={7\baselineskip},
  xtick={1,3,5,7,9,11},
  minor x tick num={1},
  title={Classifier Accuracy Prediction under Label Shift},
]

\draw[treatment line] ([yshift=-2pt] axis cs:4.046375, 0) |- (axis cs:3.1297083333333338, -2.5)
  node[treatment label, anchor=east] {\textbf{LEAP}};
\draw[treatment line] ([yshift=-2pt] axis cs:4.16725, 0) |- (axis cs:3.1297083333333338, -3.5)
  node[treatment label, anchor=east] {\textbf{DoC}};
\draw[treatment line] ([yshift=-2pt] axis cs:5.648, 0) |- (axis cs:3.1297083333333338, -4.5)
  node[treatment label, anchor=east] {\textbf{ATC}};
\draw[treatment line] ([yshift=-2pt] axis cs:5.78725, 0) |- (axis cs:3.1297083333333338, -5.5)
  node[treatment label, anchor=east] {CPCS$_{\zeta 2 \alpha}$};
\draw[treatment line] ([yshift=-2pt] axis cs:5.97175, 0) |- (axis cs:3.1297083333333338, -6.5)
  node[treatment label, anchor=east] {PACC$_{\rho 2 \alpha}$};
\draw[treatment line] ([yshift=-2pt] axis cs:6.119625, 0) |- (axis cs:8.823166666666667, -7.0)
  node[treatment label, anchor=west] {KDEy$_{\rho 2 \alpha}$};
\draw[treatment line] ([yshift=-2pt] axis cs:6.202, 0) |- (axis cs:8.823166666666667, -6.0)
  node[treatment label, anchor=west] {LasCal$_{\zeta 2 \alpha}$};
\draw[treatment line] ([yshift=-2pt] axis cs:6.416, 0) |- (axis cs:8.823166666666667, -5.0)
  node[treatment label, anchor=west] {DMCal$_{\zeta 2 \alpha}$};
\draw[treatment line] ([yshift=-2pt] axis cs:6.58425, 0) |- (axis cs:8.823166666666667, -4.0)
  node[treatment label, anchor=west] {TransCal$_{\zeta 2 \alpha}$};
\draw[treatment line] ([yshift=-2pt] axis cs:7.151, 0) |- (axis cs:8.823166666666667, -3.0)
  node[treatment label, anchor=west] {Naive};
\draw[treatment line] ([yshift=-2pt] axis cs:7.9065, 0) |- (axis cs:8.823166666666667, -2.0)
  node[treatment label, anchor=west] {EMQ$_{\rho 2 \alpha}$};
\draw[group line] (axis cs:6.416, -2.6666666666666665) -- (axis cs:6.58425, -2.6666666666666665);
\draw[group line] (axis cs:6.119625, -4.0) -- (axis cs:6.202, -4.0);
\draw[group line] (axis cs:5.78725, -3.6666666666666665) -- (axis cs:6.119625, -3.6666666666666665);

\end{axis}
\end{tikzpicture}
    }
    \caption{CD-diagrams for CS experiments (left panel) and LS experiments (right panel). Reference classifier accuracy prediction methods are highlighted in boldtype.}
    \label{fig:cap:cddiagram}
\end{figure}

\clearpage

\section{Conclusions}
\label{sec:discussions}

In this paper, we have investigated the interconnections between calibration, quantification, and classifier accuracy prediction under dataset shift, showing that these problems are close to each other, to the extent that methods originally proposed for one of the problems can be adapted to address the other problems. We have first provided formal proofs of the reduction of one problem to another based on the availability of perfect oracles. Building on these proofs, we have designed practical adaptation methods which we have tested experimentally. Additionally, we have proposed two new calibration methods, PacCal$^\sigma$ and DMCal, that leverage consolidated principles that render quantification methods robust to LS conditions.



Our experiments have led to several interesting (and sometimes unexpected) findings. Among the three tasks considered, calibration appears to benefit the most from the adaptation of methods originally proposed for quantification and classifier accuracy prediction. This may seem counter-intuitive at first glance, given that calibration is the most long-standing field among the three. However, the specific branch of calibration on which this paper focuses (i.e., calibration under dataset shift) is relatively recent \cite{CalibrationDatasetShiftNeurIPS2019}. 
In this respect, the newly proposed DMCal, which builds on ideas from the quantification literature, has shown strong potential for calibration under both CS and, especially, LS.

Conversely, quantification and classifier accuracy prediction have benefited less from the adaptation of methods from the other tasks and particularly so when confronting LS, where native state-of-the-art approaches from their respective fields still stand out in terms of performance. Notwithstanding this, the adaptation of one calibration method, TransCal, has surprisingly emerged as the top-performing method for quantification under CS conditions in our experiments. Further experiments involving different data sources and alternative protocols for simulating CS may provide further evidence that confirm or refute this finding.

We have observed that EMQ tends to stand out as one of the best-performing methods overall in both quantification and calibration problems. However, it often presents stability issues that hinder its effectiveness. 
Introducing a calibration phase before the application of the EM algorithm can help mitigate this effect, but in certain cases it may even prove counterproductive. More research is needed in order to better understand the conditions under which introducing a calibration phase is beneficial, and to explore whether it is possible to automatically select the most appropriate calibrator to use depending on the type of shift.

Perhaps one of the most surprising outcomes of our experiments is the unexpectedly good performance that DoC, a method originally proposed for classifier accuracy prediction, has consistently shown across all three tasks, and across both types of shift, in most cases. A key component of DoC is the internal sampling generation protocol used for generating validation samples with which a regressor is trained. One may (maybe legitimately) argue that our variants of DoC benefit from a certain advantage over the rest, as they rely on different sampling generation protocols depending on the type of shift. However, given the potential the method has showcased across all tasks and types of shift, we rather believe this is a strong indication that the inner workings of DoC deserve further attention. We foresee the key components for attaining a unified approach to the three problems will revolve around: (i) devising effective methods for detecting the type of shift at play \cite{Rabanser:2019ba}, (ii) devising effective protocols for generating validation samples representative of that particular type of shift \cite{Gonzalez:2024cs}, and (iii) investigating more effective ways to train a task-specific regressor \cite{Guillory:2021so,Perez-Mon:2025jt}.



Something our experiments seem to indicate, though, is that no single method consistently outperforms the others across all problems and regardless of the type of dataset shift. 
In future research, it would be interesting to better characterize the relative merits of each method with respect to specific types (and ideally even with respect to the level of intensity) of dataset shift, enabling more informed decisions about which method to use under which conditions. This line of work could potentially inspire the development of hybrid approaches that leverage insights from related disciplines. While our paper is limited to binary problems, in future work it would also be interesting to investigate extensions of the proposed methods to the multiclass regime. 


\section*{Acknowledgments}

We are grateful to Teodora Popordanoska for her valuable guidance on the LasCal repository. We would also like to thank Andrea Pedrotti for providing code scripts for training language models.

\noindent This work has been funded by the QuaDaSh project \emph{Finanziato dall'Unione europea---Next Generation EU, Missione 4 Componente 2 CUP B53D23026250001}.

\bibliographystyle{apalike}
\bibliography{references}

\clearpage
\newpage

\appendix
\section{Classifiers accuracy}
\label{app:accuracy}

Table~\ref{tab:lm:acc} reports the classifiers accuracy of different LM models on the sentiment datasets, for different combinations of source and target domains, while Table~\ref{tab:sklearn:acc} reports the accuracy of our classifiers of choice for the tabular datasets in the UCI Machine Learning repository.

\begin{table}[ht]
    \caption{LM's accuracy on sentiment classification datasets}
    \vspace{0.5cm}
    \centering
    \begin{tabular}{llrrr}
\toprule
 &  & \multicolumn{3}{c}{target}  \\
 & source & \texttt{imdb} & \texttt{rt} & \texttt{yelp} \\
\midrule
\multirow{3}{*}{BERT} & \texttt{imdb} & 0.751 & 0.689 & 0.734 \\
 & \texttt{rt} & 0.697 & 0.689 & 0.694 \\
 & \texttt{yelp} & 0.730 & 0.702 & 0.889 \\
\cline{1-5}
\multirow{3}{*}{DistilBERT} & \texttt{imdb} & 0.855 & 0.773 & 0.865 \\
 & \texttt{rt} & 0.823 & 0.785 & 0.809 \\
 & \texttt{yelp} & 0.772 & 0.713 & 0.909 \\
\cline{1-5}
\multirow{3}{*}{RoBERTa} & \texttt{imdb} & 0.848 & 0.695 & 0.865 \\
 & \texttt{rt} & 0.732 & 0.737 & 0.749 \\
 & \texttt{yelp} & 0.844 & 0.735 & 0.940 \\
\bottomrule
\end{tabular}
    
    \label{tab:lm:acc}
\end{table}

\begin{table}[ht]
    \caption{Classifiers accuracy on UCI Machine Learning datasets}
    \vspace{0.5cm}
    \centering
    \begin{tabular}{lcccc}
\toprule
 & Logistic Regression & Multi-layer Perceptron & Naïve Bayes & k Nearest Neighbor \\
\midrule
\texttt{cmc.3} & 0.663 & 0.683 & 0.649 & 0.667 \\
\texttt{ctg.1} & 0.918 & 0.933 & 0.897 & 0.911 \\
\texttt{ctg.2} & 0.918 & 0.920 & 0.743 & 0.908 \\
\texttt{ctg.3} & 0.972 & 0.969 & 0.958 & 0.955 \\
\texttt{pageblocks.5} & 0.981 & 0.986 & 0.976 & 0.982 \\
\texttt{semeion} & 0.962 & 0.969 & 0.793 & 0.964 \\
\texttt{spambase} & 0.927 & 0.935 & 0.826 & 0.896 \\
\texttt{wine-q-red} & 0.758 & 0.781 & 0.742 & 0.723 \\
\texttt{wine-q-white} & 0.739 & 0.763 & 0.684 & 0.761 \\
\texttt{yeast} & 0.771 & 0.769 & 0.300 & 0.749 \\
\bottomrule
\end{tabular}

    \label{tab:sklearn:acc}
\end{table}

\section{Brier Scores in calibration experiments}

Tables~\ref{tab:calib:covshift:brier} and~\ref{tab:calib:labelshift:brier} report the Brier Scores obtained in our calibration experiments for CS and LS, respectively. The Brier Score (BS) is defined as:
\begin{equation}
    \text{BS}=\frac{1}{|\teset|}\sum_{(x_i,y_i)\in\teset}(y_i-\hcal(x_i))^2
\end{equation}
\noindent where $y_i\in \{0,1\}$ is the true label of the $i$-th test instance, and $\hcal(x_i)$ is its predicted posterior probability, as estimated by a calibrated classifier $\hcal$.

\begin{table}[h]
    \caption{Calibration performance under CS in terms of Brier Score}
    \vspace{0.5cm}
    \centering
    \resizebox{\textwidth}{!}{%
\begin{tabular}{cc|c|cccc|ccccccccccccc} \toprule
\multicolumn{2}{c}{} & \multicolumn{1}{c|}{Baselines} & \multicolumn{4}{c|}{Reference} & \multicolumn{13}{c}{Adapted Methods} \\
\multicolumn{2}{c}{} & \begin{sideways}Platt\;\end{sideways} & \begin{sideways}Head2Tail\;\end{sideways} & \begin{sideways}CPCS\;\end{sideways} & \begin{sideways}TransCal\;\end{sideways} & \begin{sideways}LasCal\;\end{sideways} & \begin{sideways}$\text{PCC}^{5\text{B}}_{\rho 2\zeta}$\;\end{sideways} & \begin{sideways}PACC$_{\rho 2\zeta}^{5\text{B}}$\;\end{sideways} & \begin{sideways}EMQ$_{\rho 2\zeta}^{5\text{B}}$\;\end{sideways} & \begin{sideways}KDEy$_{\rho 2\zeta}^{5\text{B}}$\;\end{sideways} & \begin{sideways}ATC$_{\alpha 2\zeta}^{6\text{B}}$\;\end{sideways} & \begin{sideways}DoC$_{\alpha 2\zeta}^{6\text{B}}$\;\end{sideways} & \begin{sideways}LEAP$_{\alpha 2\zeta}^{6\text{B}}$\;\end{sideways} & \begin{sideways}LEAP$_{\alpha 2\zeta}^{\text{PCC}-6\text{B}}$\;\end{sideways} & \begin{sideways}EMQ\;\end{sideways} & \begin{sideways}EMQ$^{\text{BCTS}}$\;\end{sideways} & \begin{sideways}EMQ$^{\text{TransCal}}$\;\end{sideways} & \begin{sideways}PacCal$^{\sigma}$\;\end{sideways} & \begin{sideways}DMCal\;\end{sideways} \\\midrule
\multirow{6}{*}{\begin{sideways}BERT\;\end{sideways}} & \texttt{imdb}$\rightarrow$\texttt{rt} & $0.188$\cellcolor{green!38} & $0.206$\cellcolor{green!30} & $0.214$\cellcolor{green!26} & $0.197$\cellcolor{green!34} & $0.187$\cellcolor{green!39} & $0.188$\cellcolor{green!38} & $0.189$\cellcolor{green!38} & $0.202$\cellcolor{green!32} & $0.210$\cellcolor{green!28} & $0.188$\cellcolor{green!38} & $\textbf{0.185}$\cellcolor{green!40} & $0.191$\cellcolor{green!37} & $0.204$\cellcolor{green!31} & $0.230$\cellcolor{green!19} & $0.216$\cellcolor{green!26} & $0.362$\cellcolor{red!40} & $0.188$\cellcolor{green!38} & $0.214$\cellcolor{green!27} \\
\bottomrule
\end{tabular}
    }%
    \label{tab:calib:covshift:brier}
\end{table}

\begin{table}[h]
    \caption{Calibration performance under label shift in terms of Brier Score}
    \vspace{0.5cm}
    \centering
    \resizebox{\textwidth}{!}{%
    \begin{tabular}{cc|c|cccc|ccccccccccc} \toprule
\multicolumn{2}{c}{} & \multicolumn{1}{c|}{Baselines} & \multicolumn{4}{c|}{Reference} & \multicolumn{11}{c}{Adapted Methods} \\
\multicolumn{2}{c}{} & \begin{sideways}Platt\;\end{sideways} & \begin{sideways}Head2Tail\;\end{sideways} & \begin{sideways}CPCS\;\end{sideways} & \begin{sideways}TransCal\;\end{sideways} & \begin{sideways}LasCal\;\end{sideways} & \begin{sideways}PACC$_{\rho 2\zeta}^{5\text{B}}$\;\end{sideways} & \begin{sideways}EMQ$_{\rho 2\zeta}^{5\text{B}}$\;\end{sideways} & \begin{sideways}KDEy$_{\rho 2\zeta}^{5\text{B}}$\;\end{sideways} & \begin{sideways}ATC$_{\alpha 2\zeta}^{6\text{B}}$\;\end{sideways} & \begin{sideways}DoC$_{\alpha 2\zeta}^{6\text{B}}$\;\end{sideways} & \begin{sideways}LEAP$_{\alpha 2\zeta}^{6\text{B}}$\;\end{sideways} & \begin{sideways}EMQ\;\end{sideways} & \begin{sideways}EMQ$^{\text{BCTS}}$\;\end{sideways} & \begin{sideways}EMQ$^{\text{LasCal}}$\;\end{sideways} & \begin{sideways}PacCal$^{\sigma}$\;\end{sideways} & \begin{sideways}DMCal\;\end{sideways} \\\midrule
\multirow{10}{*}{\begin{sideways}Logistic Regression\;\end{sideways}} & \texttt{cmc.3} & $0.248$\cellcolor{red!4} & $0.242$\cellcolor{red!0} & $0.241$\cellcolor{red!0} & $0.243$\cellcolor{red!0} & $0.240$\cellcolor{green!0} & $0.240$\cellcolor{green!0} & $0.244$\cellcolor{red!1} & $0.231$\cellcolor{green!5} & $0.274$\cellcolor{red!19} & $0.252$\cellcolor{red!6} & $0.240$\cellcolor{green!0} & $\textbf{0.177}$\cellcolor{green!37} & $\textbf{0.173}$\cellcolor{green!40} & $0.261$\cellcolor{red!12} & $0.309$\cellcolor{red!40} & $0.177$\cellcolor{green!37} \\
\bottomrule
\end{tabular}
    }%
    \label{tab:calib:labelshift:brier}
\end{table}

\end{document}